\RequirePackage{fix-cm}
\documentclass[twocolumn]{svjour3}          
\smartqed  
\usepackage{graphicx}
\usepackage{amssymb}
\usepackage{amsmath}
\usepackage[font=small,format=plain,labelfont=bf,up,textfont=it,up]{caption}
\usepackage{subcaption}
\usepackage{multirow}
\usepackage{booktabs}
\usepackage{array}

\newlength{\textwidthtwo}
\providecommand{\norm}[1]{\lVert#1\rVert}

\renewcommand\Im{\operatorname{Im}}

\newcommand{\R}{\mathbb{R}}

\newcommand{\ul}{\textbf}

\newcommand{\www}{\mbox{\boldmath$\omega$}}

\newcommand{\desda}{\Leftrightarrow}

\DeclareMathAlphabet\gothic{U}{euf}{m}{n}

\begin{document}
\title{A Multi-Orientation Analysis Approach to Retinal Vessel Tracking}


\author{Erik Bekkers         \and
        Remco Duits          \and
        Tos Berendschot      \and
        Bart ter Haar Romeny
}


\institute{E.J. Bekkers \and R. Duits \and B.M. ter Haar Romeny
              \at
              Eindhoven University of Technology \\
              Department of Biomedical Engineering \\
              P.O. Box 513, NL-5600 MB Eindhoven, The Netherlands\\
              \email{e.j.bekkers@tue.nl}
           \and
           T.T.J.M. Berendschot
              \at
              University Eye Clinic Maastricht,
              Maastricht, the Netherlands
}

\date{Received: date / Accepted: date}

\maketitle

\begin{abstract}
This paper presents a method for retinal vasculature extraction based on biologically inspired multi-orientation analysis. We apply multi-orientation analysis via so-called invertible orientation scores, modeling the cortical columns in the visual system of higher mammals. This allows us to generically deal with many hitherto complex problems inherent to vessel tracking, such as crossings, bifurcations, parallel vessels, vessels of varying widths and vessels with high curvature. Our approach applies tracking in invertible orientation scores via a novel geometrical principle for curve optimization in the Euclidean motion group SE(2). The method runs fully automatically and provides a detailed model of the retinal vasculature, which is crucial as a sound basis for further quantitative analysis of the retina, especially in screening applications.
\end{abstract}

\begin{figure*}[!ht]
        \centering
        \begin{subfigure}[b]{0.23\textwidth}
                \centering
                \includegraphics[width=\textwidth]{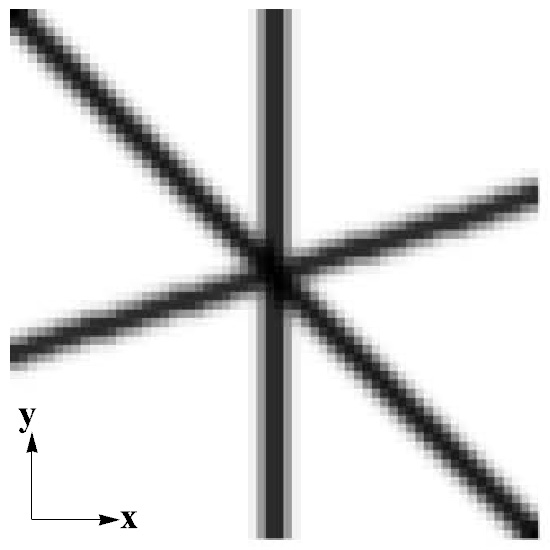}
                \caption{Image}
        \end{subfigure}
        \begin{subfigure}[b]{0.23\textwidth}
                \centering
                \includegraphics[width=\textwidth]{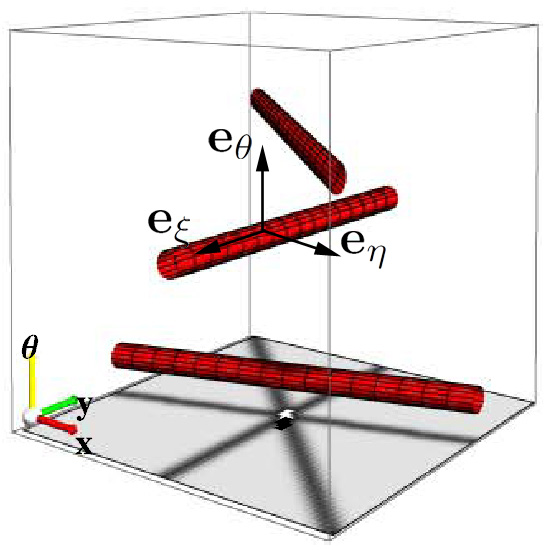}
                \caption{Orientation score}
                \label{fig:OrientationScores:b}
        \end{subfigure}
        \begin{subfigure}[b]{0.23\textwidth}
                \centering
                \includegraphics[width=\textwidth]{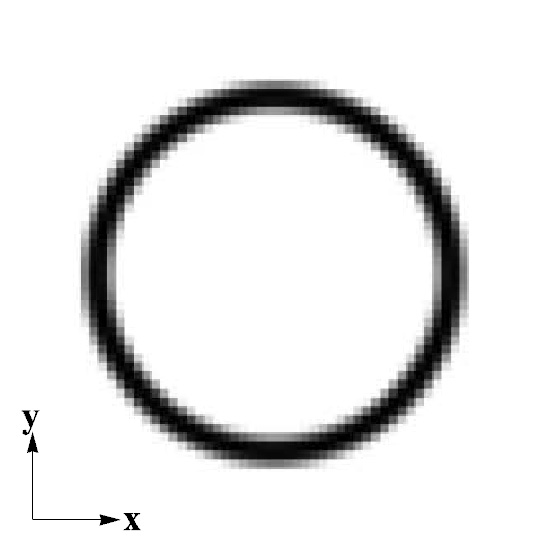}
                \caption{Image}
                \label{fig:OrientationScores:c}
        \end{subfigure}
        \begin{subfigure}[b]{0.23\textwidth}
                \centering
                \includegraphics[width=\textwidth]{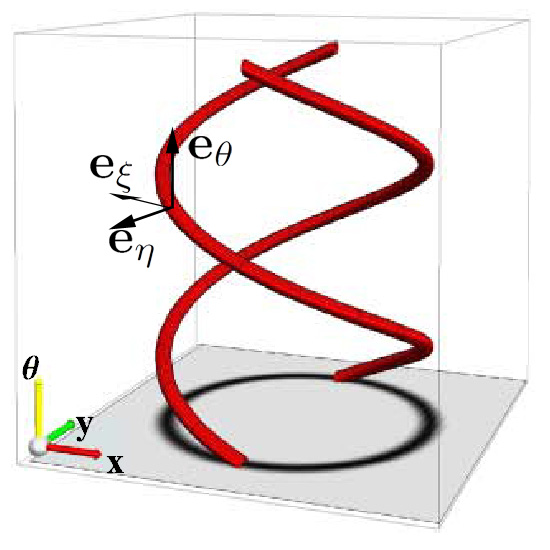}
                \caption{Orientation score}
                \label{fig:OrientationScores:d}
        \end{subfigure}
        \\
        \begin{subfigure}[b]{0.45\textwidth}
                \centering
                \includegraphics[width=\textwidth]{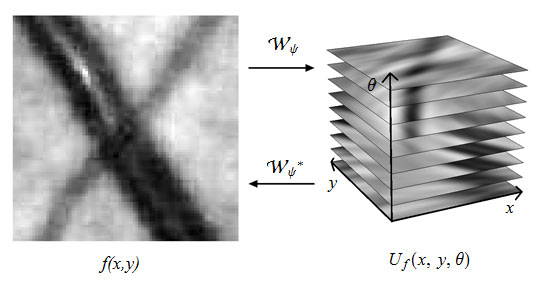}
                \caption{Transforms}
                \label{fig:OrientationScores:e}
        \end{subfigure}
        ~~
        \begin{subfigure}[b]{0.45\textwidth}
                \centering
                \includegraphics[width=\textwidth]{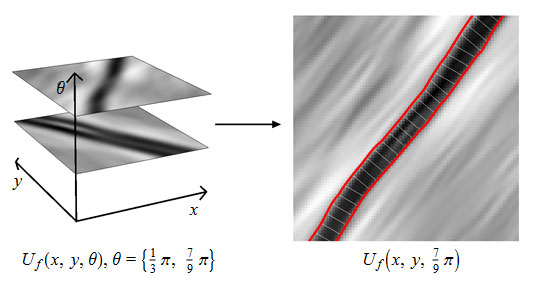}
                \caption{Edge tracking in OS}
                \label{fig:OrientationScores:f}
        \end{subfigure}
        \caption{(a,b): Crossing lines in the image domain (a) are clearly unwrapped in the orientation score domain (b). (c,d): The domain of an orientation score is curved and the orientation dimension is 2$\pi$-periodic. (e) An invertible orientation score transformation $W_\psi$ neatly distributes pixel information from an image $f$ over a set of orientations in an orientation score $U_f$. An orientation score transformation is invertible if there exists a well-posed inverse transformation $W_\psi^{-1} = W_\psi^*$ that allows exact reconstruction from the orientation score domain. (f) Disentanglement of crossing structures in the score domain allows vessel tracking through crossings.}
        \label{fig:OrientationScores}
\end{figure*}

\section{Introduction}
\label{sec:Introduction}
The retinal vasculature is the only part of the body's circulatory system that can be observed non-invasively by optical means. A large variety of diseases affect the vasculature in a way that may cause geometrical and functional changes. Retinal images are therefore not only suitable for investigation of ocular diseases such as glaucoma and age-related macular degeneration (AMD), but also for systemic diseases such as diabetes, hypertension and arteriosclerosis. This makes the retinal vasculature a rewarding and well researched subject and and a growing number of image processing techniques are developed to segment and analyze the retinal vasculature \cite{Abramoff2010,Patton2006}.

\paragraph{Retinal vessel tracking}
Typically there are two types of methods for vessel extraction: pixel classification methods \cite{Krause2013,Philipsen2012,Budai2009,Odstrcilik2009} and vessel tracking methods \cite{Al-Diri2009,Can1999,Yin2012,Grisan2004,Espona2007,Poletti2011,Chutatape1998}. The first type of method classifies pixels as either being part of a vessel or background, resulting in a pixel map in which white pixels represent blood vessels. Of the pixel classification methods, the approach by Krause et al. \cite{Krause2013} is most similar to ours as both methods rely on a transformation to a higher dimensional domain. In their work they applied vessel detection based on the local Radon transform, of which we will show later in this article that this is a special case of an orientation score transform based on cake wavelets. The other type of method, vessel tracking, is based on recursively expanding a model of the vasculature from a set of seed points. One advantage of vessel tracking over pixel classification is that it guarantees connectedness of vessel segments, whereas in pixel classification methods this is not necessarily the case. For further quantitative analysis of the vasculature, tracking algorithms are preferred because they intrinsically provide geometrical and topological information. For example, vessel widths, curvatures, segment lengths, bifurcation density and other features can relatively easily be extracted from the generated vessel models.

Several different approaches to vessel tracking can be found in literature. There are methods based on active contours \cite{Al-Diri2009,Espona2007}, matched filters \cite{Can1999,Grisan2004,Chutatape1998}, and probabilistic models \cite{Yin2012} among others \cite{Yin2012,Poletti2011}. The majority of papers on vessel tracking report limitations regarding tracking blood vessels through crossings, bifurcations and/or more complex situations. In this paper we aim at solving these problems by means of orientation analysis via so called \emph{orientation scores}, which are objects in which image information is neatly organized based on position and orientation \cite{Duits2005}. We propose two new tracking algorithms that act directly on the domain of an orientation score, and we show that these methods are highly capable of dealing with the aforementioned problems. Afterwards, we will extend one of the orientation score based algorithms to a vasculature tracking algorithm, which is capable of constructing models of the complete retinal vasculature.

\paragraph{Orientation scores}
Inspired by the cortical orientation columns in the primary visual cortext \cite{Hubel2009}, Duits et al. developed a mathematical framework for image processing and analysis based on \emph{2D orientation scores} \cite{Duits2007}. Similar to the perceptual organization of orientation in the visual cortex, a 2D orientation score is an object that maps 2D positions and orientation angles $(x,y,\theta)$ to complex scalars. Instead of assigning an orientation to each position and thereby extending the codomain, we extend the domain of the image where our modeling can deal with multiple orientations per position. When constructing an orientation score it is crucial one does not tamper the data evidence before tracking takes place. Therefore we consider \emph{invertible} orientation scores that provide a full comprehensive overview of how the image is decomposed out of local (multiple) orientations. In invertible orientation scores, all orientated structures are disentangled, see Fig.~\ref{fig:OrientationScores}.

\paragraph{Paper structure}
The article is structured as follows: First, in Section~\ref{sec:OrientationScores}, theory about orientation scores is provided. In Section~\ref{sec:VesselTrackingUsingOrientationScores}, two vessel tracking approaches based on orientation scores are described:
\begin{itemize}
    \item the ETOS-algorithm: an all-scale approach based on a new class of wavelets, the so-called \emph{cake wavelets}
    \item the CTOS-algorithm: a multi-scale approach based on the classical Gabor wavelets
\end{itemize}
Both tracking methods rely on a novel generic geometrical principle for curve optimization within the Euclidean motion group, which is explained and mathematically underpinned in Appendix \ref{app:optimalpaths}. We will show that ETOS generally works with different types of orientation scores, however with best performance on invertible orientation scores based on cake wavelets \cite{Duits2005,Duits2007a} (in comparison to non-invertible orientation scores based on Gabor wavelets). The second approach requires a multi-scale and orientation decomposition. The two approaches are described in Section~\ref{sec:Methods}, and evaluated in Section~\ref{sec:validation}. It will turn out that ETOS based on cake wavelets has several advantages over CTOS based on Gabor wavelets. We have validated ETOS more extensively by comparing it to the state of the art in retinal vessel tracking \cite{Al-Diri2009,Bankhead2012,Xu2011} using the publicly availabe REVIEW database \cite{Al-Diri2009}. In Section~\ref{sec:vasculatureTrackingMethods} we describe our vasculature tracking algorithm, composed of proper initialization, junction detection and junction resolver algorithms. In Section~\ref{sec:vasculatureResults} the correctness of topology of the models is evaluated using images of the HRFI-database \cite{Budai2011}. General conclusions can be found in Section~\ref{conclusion}.

\section{Orientation scores}
\label{sec:OrientationScores}

Orientation detection and encoding is a common subject in image processing. E.g., Frangi et al. \cite{Frangi1998} detect retinal blood vessels by calculating a vesselness value, obtained by eigenvalue analysis of the Hessian matrix. In \cite{Soares2006} this is done by taking the maximum modulus over a set of oriented Gabor wavelet responses. Besides the presence of local orientations, the orientation value(s) may be relevant as well. For instance in vessel tracking and certain other image enhancement techniques like coherence enhanced diffusion \cite{Weickert1999} and orientation channel smoothing \cite{Felsberg2006}

The most commonly used methods to detect orientations are capable of detecting only one orientation per position. However, by using oriented wavelets and steerable filters \cite{Freeman1991,Simoncelli1994,Perona1991,Granlund1995} orientation confidence measures can be extracted for any given orientation, thus allowing for the detection of multiple orientations per position. Oriented wavelets allow for a transformation from an image to an orientation score, where each locally present combination of position and orientation is mapped to a single value \cite{Kalitzin1999,Duits2007,Duits2007a}, see Fig.~\ref{fig:OrientationScores}.

In his pioneering paper, Kalitzin \cite{Kalitzin1999} proposed a specific wavelet, given by
\begin{equation}
\label{eq:kalitzin}
\psi(x,y)=\cfrac{1}{\sqrt{2\pi}} \sum\limits_{n=1}^N \cfrac{\overline{z}^n}{\sqrt{n!}}e^{-\cfrac{|z|^2}{2}}, z = x + i y,
\end{equation}
that, by approximation, guaranteed invertibility from the orientation scores back to the original image, without loss of information. These wavelets belong to a specific class of so called proper wavelets (wavelets that allow well-posed reconstruction \cite{Duits2007,Duits2007a}), and are found by expansion in eigenfunctions of the harmonic oscillator. The advantage of this expansion is that this steerable basis is Fourier invariant, allowing to control the wavelet shape in both the spatial and Fourier domain. The disadvantages of such kernels $\psi$ are however that 1) their series do not converge in $\mathbb{L}_2$ and truncation of the pointwise converging series heavily affects the shape and induces undesirable oscillations \cite[p.140-142]{Duits2005}; 2) the wavelets $\psi$ explode in the radial direction along its orientation, e.g. for the case in (\ref{eq:kalitzin}) it explodes with $\sim (8\pi)^{1/4} \sqrt{r} (1-O(r^{-2})) $ \cite[App.7.3]{Duits2005}; 3) does not allow approximate reconstruction by integration over $S^1$ only.

In Kalitzin's paper the well-posedness of the reconstruction was not quantified. Analysis of the well-posedness was done by Duits using the function $M_\psi$, which will be explained in more detail in Section 2.1. The function $M_\psi$ indicates how well spatial frequencies in an image are preserved after a transformation, and it can be seen as a measure for stability of the inverse transformation \cite{Duits2007a,Fuehr2005,Duits2005,Duits2007}. Within these papers a new class of proper wavelets called \emph{cake wavelets} are presented. These are oriented wavelets, able to capture all image scales without any bias to a specific scale. This property is crucial in our vessel edge tracking approach since it reduces the need for a multi-scale approach, as will become more clear in the next sections.

In contrast, the Gabor wavelet is another oriented wavelet, which is widely used in the field of image processing because of its capability to detect oriented features at a certain scale. An example of segmenting crossing lines using Gabor wavelet based orientation scores is given in \cite{Chen2000}. The property to tune the Gabor wavelet to capture features at a specific scale can be very useful, but the scale selection also implies exclusion of other scales. The single-scale Gabor wavelet transformation causes information from the original image to be lost and the transformation is therefore non-invertible. In order to introduce invertibility, one has to use a multi-scale approach \cite{Fischer2007}, which is also computationally more expensive.

Since no information should be lost during the orientation score transformation $W_\psi : f \rightarrow U_f$ (see Fig.~\ref{fig:OrientationScores:e}), the notion of invertibility of an orientation score transformation is essential. The invertibility allows us to relate operators on images to operators on orientation scores, and vice versa. Using invertible orientation scores, one can employ the automatic disentanglement of local orientations involved in a crossing (cf. Fig.~\ref{fig:OrientationScores:f}). For line/contour \emph{enhancement}, this has lead to a generic crossing preserving diffusion method \cite{Franken2009,Duits2010a,Franken2008} which outperformed related diffusions acting directly on the image domain. For \emph{tracking} of crossing blood vessels a similar advantage can be employed, as we will show in this article.

\subsection{Construction of orientation scores}
\label{sec:ConstructionOfOrientationScores}
Consider a 2D image $f$ as a function $f : \mathbb{R}^2\rightarrow \mathbb{R}$, with compact support on the image domain $\Omega = [0,X] \times [0,Y]$, with image dimensions $X,Y \in \mathbb{R^+}$, and which we assume to be square integrable, i.e. $f \in \mathbb{L}_2 (\mathbb{R}^2)$. An orientation score, constructed from image $f$, is defined as a function $U_f : \mathbb{R}^2 \times S^1 \rightarrow \mathbb{C}$ and depends on two variables ($\mathbf{x},\theta$),
where $\mathbf{x}=(x_1,x_2) \in \mathbb{R}^2$ denotes position and $\theta \in [0,2\pi]$ denotes the orientation variable.

An orientation score $U_f:=W_\psi f$ of a function $f$ can be constructed by means of convolution with some anisotropic wavelet $\psi$ via
\begin{equation}
\label{eq:ostransform}
U_f(\mathbf{x},\theta) = (\check{\overline{\psi}}_\theta *  f)(\mathbf{x}) =
\int_{\mathbb{R}^2}\overline{\psi(\mathbf{R}_\theta^{-1}(\mathbf{y}-\mathbf{x}))}f(\mathbf{y})d\mathbf{y},
\end{equation}
where $\psi \in \mathbb{L}_1(\mathbb{R}^2)\bigcap\mathbb{L}_2(\mathbb{R}^2)$ is the convolution kernel with orientation $\theta = 0$, i.e. aligned with the vertical axis in our convention, and where $W_\psi$ denotes the transformation between image $f$ and orientation score $U_f$. The overline denotes complex conjugate, $\check{\psi}_\theta(\mathbf{x}) =  \psi_\theta(\mathbf{-x})$ and the rotation matrix $\mathbf{R}_\theta$ is given by
\begin{equation}
\mathbf{R}_\theta =
\left(
\begin{array}{cc}
\cos\theta & -\sin\theta \\
\sin\theta & \cos\theta
\end{array}
\right),
\end{equation}
see Fig.~\ref{fig:OrientationScores:e}.
Exact reconstruction\footnote{The reconstruction formula can easily be verified using the convolution theorem, $\mathcal{F}[f*g] = \cfrac{1}{2\pi}\mathcal{F}[f]\mathcal{F}[g]$, and the fact that $\mathcal{F}\left[\check{\overline{\psi}}_\theta\right] = \overline{\mathcal{F}[\psi_\theta]}$} from the orientation scores constructed by (\ref{eq:ostransform}) is given by
\begin{equation}
\label{eq:reconstruction}
\begin{array}{ll}
f &= W_\psi^* W_\psi f\\
  &= \mathcal{F}^{-1}
    \left[
        M_\psi^{-1}\mathcal{F}
        \left[
            \mathbf{x} \mapsto \cfrac{1}{2\pi} \int_{0}^{2\pi} (\psi_\theta * U_f(\cdot,\theta)) (\mathbf{x}) d\theta
        \right]
    \right],
\end{array}
\end{equation}
where $\mathcal{F}$ is the unitary Fourier transform on $\mathbb{R}^2$, where $W_\psi^*$ denotes the adjoint wavelet transformation (see \cite{Duits2005} for details), and $M_\psi:\mathbb{R}^2\rightarrow
\mathbb{R}^+$ is calculated by
\begin{equation}
M_\psi = 2\pi \int_0^{2\pi}\overline{\mathcal{F}[\psi_\theta]}\mathcal{F}[\psi_\theta]d\theta
= \int_0^{2\pi}\vert\mathcal{F}[\psi_\theta]\vert^2d\theta.
\end{equation}
\begin{figure}[t]
\centering
\includegraphics[width=\textwidthtwo]{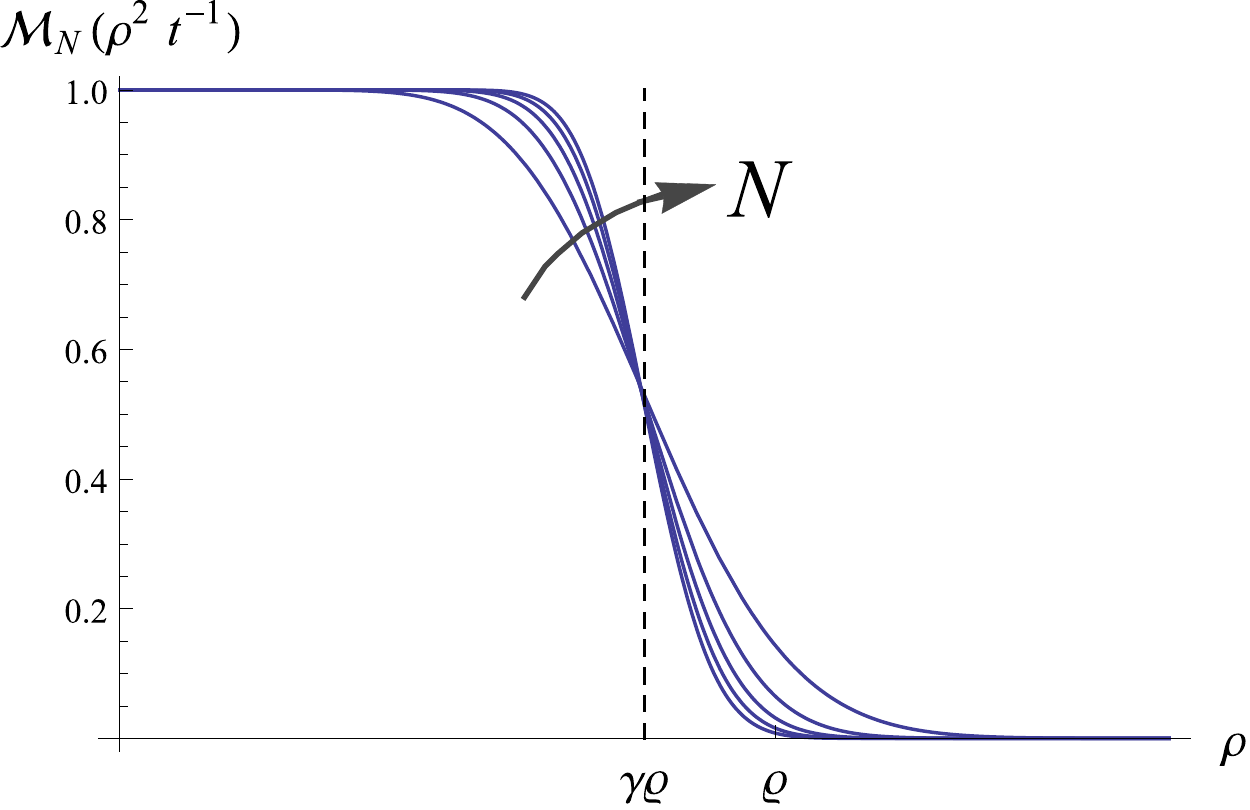}
\caption{Plots of $\mathcal{M}_N\left(\rho^2/ t\right)$, with $t = \cfrac{2 (\gamma \varrho)^2}{1+2N}$ for $N =$ 5, 10, 15, 20, 25}
\label{fig:Mn}
\end{figure}
The function $M_\psi$ provides a stability measure of the inverse transformation. Theoretically, reconstruction is well-posed, as long as
\begin{equation}
\label{eq:conditionmpsi}
0 < \delta < M_\psi(\www) < M < \infty,
\end{equation}
where $\delta$ is arbitrarily small, since then the condition number of $W_{\psi}$ is bounded by $M \: \delta^{-1}$, see \cite[Thm.~20]{Duits2005}. If we do not restrict ourselves to band-limited/disk-limited functions, this requirement (Eq.~(\ref{eq:conditionmpsi})) bites the assumption $\psi \in \mathbb{L}_1(\mathbb{R}^2)\bigcap\mathbb{L}_2(\mathbb{R}^2)$ since it implies $\mathcal{F}\psi$ is a continuous function vanishing at infinity (see e.g. \cite{Rudin1973}) and so is $M_\psi$. In that case we have to resort to distributional wavelet transforms\footnote{This is comparable to the construction of the unitary Fourier transform $F:\mathbb{L}_2(\mathbb{R}^2)\rightarrow\mathbb{L}_2(\mathbb{R}^2)$ whose kernel $k(\www,\mathbf{x}) = e^{-i \www \cdot \mathbf{x}}$ is also not square integrable.} whose closure is again a unitary map from $\mathbb{L}^2(\mathbb{R}^2)$ into a reproducing kernel subspace of $\mathbb{L}_2(\mathbb{R}^{2}\times S^{1})$, for details see Appendix \ref{app:distributionalostrafo}.

In practice, to prevent numerical problems, it is best to aim at $M_\psi(\www) \approx 1$ for $\Vert\www\Vert < \varrho$, where $\varrho$ is the Nyquist frequency of the discretely sampled image, meaning that all relevant frequency components within a ball of radius $\varrho$ are preserved in the same way. Because of the discontinuity at $\Vert\www\Vert = \varrho$, which causes practical problems with the discrete inverse Fourier transform, we will use wavelets $\psi$, with $M_\psi(\www) = \mathcal{M}_{N}\left(\rho^2 t^{-1}\right)$, $N \in \mathbb{N}$, $t > 0$ and $\rho = \Vert\www\Vert$ where
\begin{equation}
\mathcal{M}_N\left(\rho^2 t^{-1}\right) = e^{-\cfrac{\rho^2}{t}}\sum_{k=0}^{N}\cfrac{\left(\rho^2 t^{-1}\right)^{k}}{k!} \leq 1,
\label{eq:Mn}
\end{equation}
where $t$ denotes a scale parameter. To fix the inflection point close to the Nyquist frequency, say at $\rho = \gamma \varrho$ with $0\ll\gamma<1$, we set $t = \cfrac{2 (\gamma \varrho)^2}{1+2N}$ (to fix the bending point: $\cfrac{d^2}{d\rho^2}\mathcal{M}_N(\rho^2 t^{-1})|_{\rho = \gamma \varrho}=0$, see Fig.~\ref{fig:Mn}).
The function $\mathcal{M}_N$ basically is a Gaussian function at scale $t$, multiplied with the Taylor series of its inverse up to a finite order 2N to ensure a slower decay. The function $\mathcal{M}_N$ smoothly approximates 1 on the domain $\rho \in [0,\varrho]$, see Fig.~\ref{fig:Mn}. A wavelet $\psi: \mathbb{R}^2 \rightarrow \mathbb{C}$ with such a $M_\psi$ will be called a proper wavelet.

The methods presented in this paper are best described by a moving frame of reference that lives in the tangent bundle ($T(\mathbb{R}^2 \times S^1$)) above the domain $\mathbb{R}^2 \times S^1$ of an orientation score. This moving frame of reference is given by the map
\begin{equation}
\begin{array}{c}
\mathbb{R}^{2}\times S^{1} \ni g=(x,y,\theta) \\ \mapsto \\(\cos(\theta)\mathbf{e}_x+\sin(\theta)\mathbf{e}_y,-\sin(\theta)\mathbf{e}_x+\cos(\theta)\mathbf{e}_y,\mathbf{e}_\theta) \\\in T_g(\mathbb{R}^{2}\times S^{1})\times T_g(\mathbb{R}^{2}\times S^{1}) \times T_g(\mathbb{R}^{2}\times S^{1}),
\end{array}
\end{equation}
with $\mathbf{e}_x = (1,0,0)$ and $\mathbf{e}_y = (0,1,0)$, $\mathbf{e}_{\theta}=(0,0,1)$ tangent vectors (formally attached at $g$).
To simplify the notation we introduce coordinates $\xi = x\cos(\theta) + y\sin(\theta)$, $\eta = -x\sin(\theta) + y\cos(\theta)$
 and tangent vectors
\begin{equation}
\label{eq:frameOfReference}
\begin{array}{lll}
\mathbf{e}_\xi &= \cos(\theta)\mathbf{e}_x +\sin(\theta)\mathbf{e}_y &= (\cos(\theta),\sin(\theta),0)\\
\mathbf{e}_\eta &= -\sin(\theta)\mathbf{e}_x + \cos(\theta)\mathbf{e}_y &= (-\sin(\theta),\cos(\theta),0)\\
\mathbf{e}_\theta &= (0,0,1).&
\end{array}
\end{equation}
As a result we have that at a given point in the orientation score $(x,y,\theta)$, the tangent vector $\mathbf{e}_\xi$ points in the spatial direction aligned with the orientation score kernel used at layer $\theta$, see Fig.~\ref{fig:OrientationScores:b} and \ref{fig:OrientationScores:d}. We will often rely on the notation in Eq.~(\ref{eq:frameOfReference}) in the remainder of this article.

\subsection{The domain of orientation scores: SE(2)}
The domain of an orientation score is the set $\mathbb{R}^2 \times S^1$. However, from Fig.~\ref{fig:OrientationScores:d} one can recognize a curved geometry on the domain of orientation scores. This is reflected in the fact that $\mathbf{e}_\xi$ and $\mathbf{e}_\eta$ vary with $\theta$. This is modeled by imposing a group structure on the set $\mathbb{R}^2 \times S^1$. This group structure comes from rigid body motions $g = (\mathbf{x},\mathbf{R}_\theta)$ acting on $\mathbb{R}^2 \times S^1$ via
\begin{equation}
\label{action}
g\cdot (\mathbf{x}',\theta') = (\mathbf{R}_\theta \mathbf{x}' + \mathbf{x}, \theta + \theta').
\end{equation}
Note that $(\mathbf{x},\theta) = (x, \mathbf{R}_\theta)(\mathbf{0},0)$ which allows us to uniquely identify
$$
\mathbb{R}^2 \times S^1 \ni (\mathbf{x},\theta) \leftrightarrow (\mathbf{x},\mathbf{R}_\theta) = g \in \mathbb{R}^2 \rtimes SO(2),
$$
i.e. to identify the space of positions and orientations with the rigid body motion group $SE(2) = \mathbb{R}^2 \rtimes SO(2)$. As the combination of two rigid body motions is again a rigid body motion, $SE(2)$ is equipped with the group product:
\begin{equation}\label{product}
g \cdot g' = (\mathbf{x},\mathbf{R}_\theta) (\mathbf{x}',\mathbf{R}_{\theta'}) = (\mathbf{R}_\theta \mathbf{x}' + \mathbf{x}, \mathbf{R}_{\theta+\theta'}),
\end{equation}
which is consistent with Eq.~(\ref{action}). The moving frame of reference (\ref{eq:frameOfReference}) corresponds to the so-called left-invariant vector fileds in SE(2). For details see \cite[Fig.~2.5a]{Franken2008}.

\subsection{Cake wavelets}
\label{sec:CakeWavelets}
Cake wavelets are constructed from the Fourier domain. By using polar coordinates, the Fourier domain can be uniformly divided into $N_o$ equally wide samples ("pieces of cake") in the angular direction, see Fig.~\ref{fig:BSplines}. The spatial wavelet is given by
\begin{equation}
\psi^{cake}(\mathbf{x})=\mathcal{F}^{-1}[\tilde{\psi}^{cake}](\mathbf{x})G_{\sigma_s}(\mathbf{x}),
\label{eq:windowing}
\end{equation}
where $G_{\sigma_s}$ is a Gaussian window, with $0 < 1 \ll \sigma_s$, that is used to avoid long tails in the spatial domain. Note that multiplication with a large window in the spatial domain corresponds to a convolution with a small window in the Fourier domain, such that $M_\psi$ is hardly affected with $\sigma_s$ sufficiently large. Function $\tilde{\psi}^{cake}$ is given by
\begin{equation}
\label{eq:fcake}
\tilde{\psi}^{cake}(\www) =
    B_k\left(\cfrac{(\varphi\mod2\pi)-\pi/2}{s_\theta}\right)\mathcal{M}_N(\rho),
\end{equation}
with $\www = (\rho\cos\varphi,\rho\sin\varphi)$ and where $s_\theta = 2\pi N_o^{-1}$ is the angular resolution in radians. The function $\mathcal{M}_n$
specifies the radial function in the Fourier domain given by
(\ref{eq:Mn}). $B_k$ denotes the $k$th order B-spline given by

\begin{equation}
\begin{array}{l}
B_k(x) = (B_{k-1} * B_0)(x),\\
B_0(x) =
\left\{ \begin{array}{ll}
    1 & \mbox{if $-1/2 < x < +1/2$}\\
    0 & \mbox{otherwise}\end{array} \right.
    .
\end{array}
\end{equation}
Orientation scores constructed from an image $f$ using cake wavelets are denoted by $U_f^{cake}$. Fig.~\ref{fig:orientationScoreFiltering}c demonstrates a typical cake wavelet based orientation score for a certain orientation.

\begin{figure}[t]
\begin{center}
\includegraphics[width=\textwidthtwo]{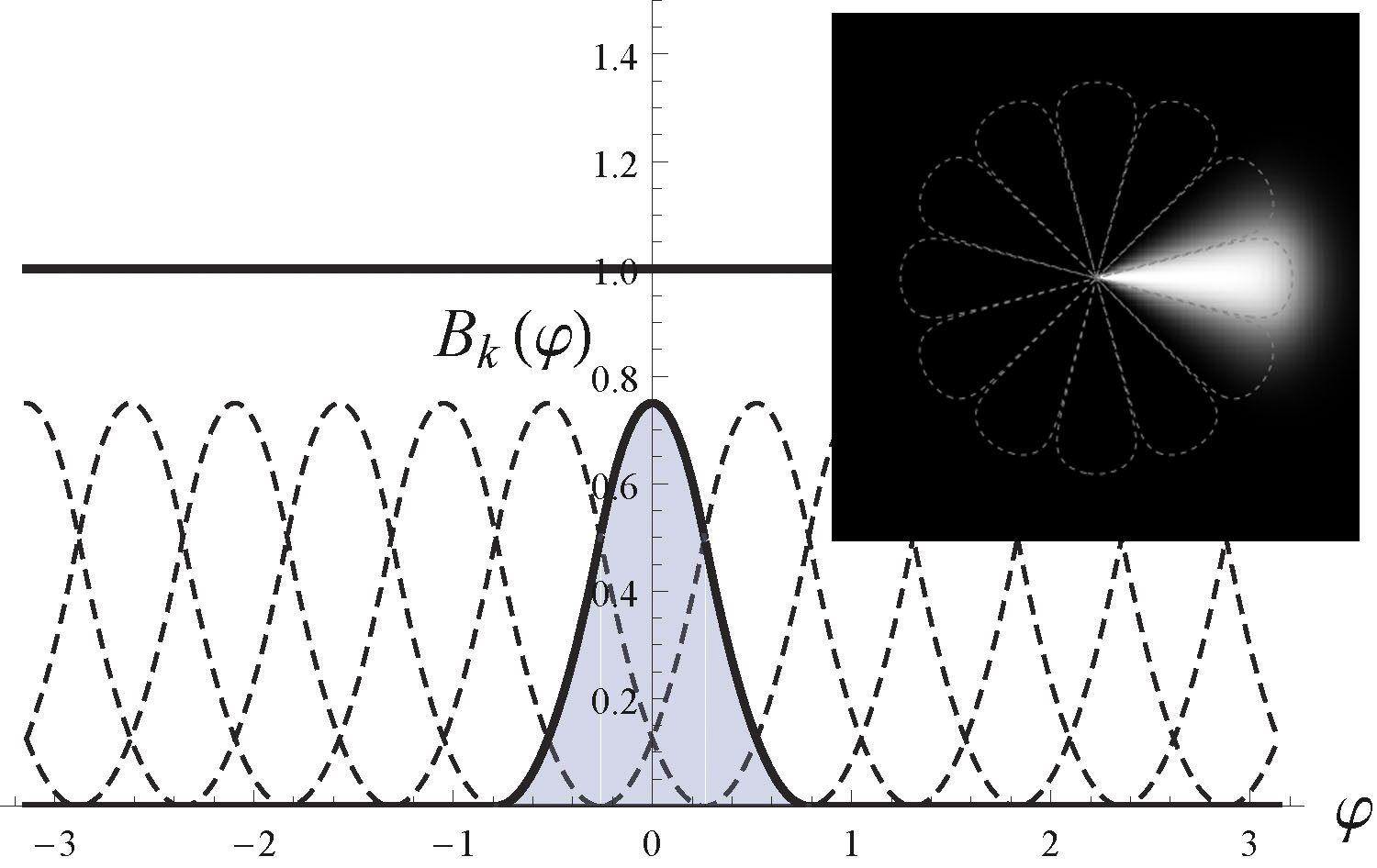}
\end{center}
\caption{The use of B-splines in the construction of cake wavelets. Plot showing quadratic B-splines (k=2), the sum of all shifted B-splines add up to 1. The image in the upper right corner illustrates a Fourier cake wavelet $\tilde{\psi}^{cake}(\www)$ constructed using quadratic B-splines and $\mathcal{M}_N$ with $N=60$, according to Eq.~(\ref{eq:fcake}).}
\label{fig:BSplines}
\end{figure}

\begin{figure*}[!ht]
        \centering
        \begin{subfigure}[t]{0.19\textwidth}
                \centering
                \includegraphics[width=\textwidth]{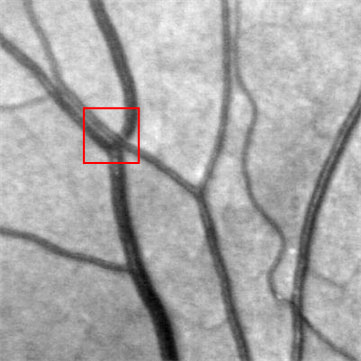}
                \caption{Image selection}
        \end{subfigure}
        \begin{subfigure}[t]{0.19\textwidth}
                \centering
                \includegraphics[width=\textwidth]{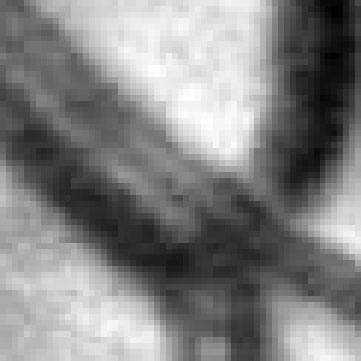}
                \caption{Zoomed image $f(\cdot)$}
        \end{subfigure}
        \begin{subfigure}[t]{0.19\textwidth}
                \centering
                \includegraphics[width=\textwidth]{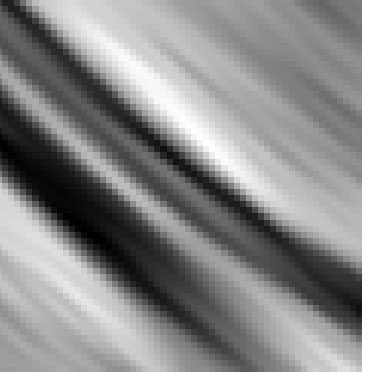}
                \caption{$U_f^{cake}(\cdot,\theta_v)$}
                \label{fig:orientationScoreFiltering:c}
        \end{subfigure}
        \begin{subfigure}[t]{0.19\textwidth}
                \centering
                \includegraphics[width=\textwidth]{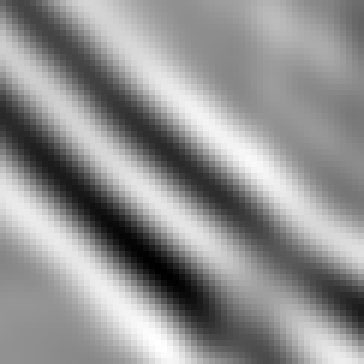}
                \caption{$U_{f,a_1}^{Gabor}(\cdot,\theta_v)$}
        \end{subfigure}
        \begin{subfigure}[t]{0.19\textwidth}
                \centering
                \includegraphics[width=\textwidth]{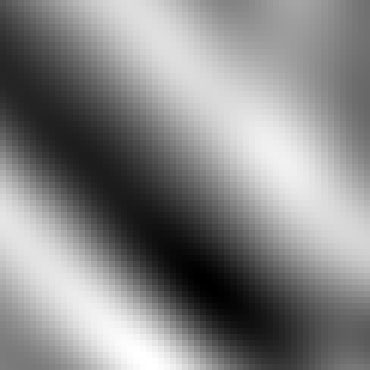}
                \caption{$U_{f,a_3}^{Gabor}(\cdot,\theta_v)$}
        \end{subfigure}
        \caption{Parallel blood vessels and orientation scores. (a) A selection of a fundus image and (b) a close-up view. (c-d) Slices of orientation scores constructed from (b) using cake wavelets and Gabor wavelets at scale $a_1=3*10/(2\pi)$ and $a_3=3*30/(2\pi)$ respectively. The slices correspond to the orientation of the two parallel blood vessels, $\theta_v$.}
        \label{fig:orientationScoreFiltering}
\end{figure*}

\begin{figure*}[!t]
        \centering
        \begin{subfigure}[b]{0.24\textwidth}
                \centering
                \includegraphics[width=\textwidth]{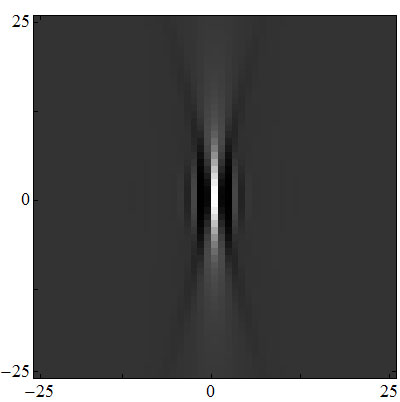}
        \end{subfigure}
        \begin{subfigure}[b]{0.24\textwidth}
                \centering
                \includegraphics[width=\textwidth]{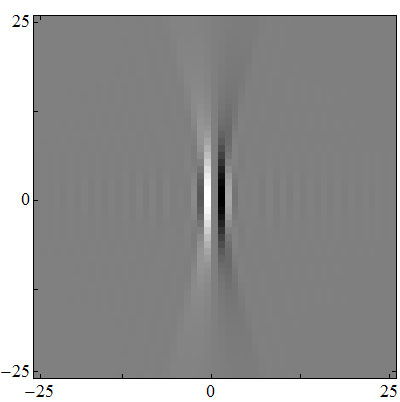}
        \end{subfigure}
        \begin{subfigure}[b]{0.25\textwidth}
                \centering
                \includegraphics[width=\textwidth]{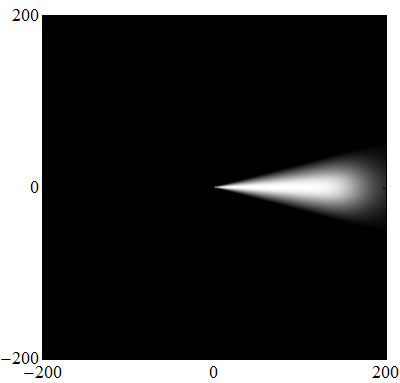}
        \end{subfigure}
        \begin{subfigure}[b]{0.25\textwidth}
                \centering
                \includegraphics[width=\textwidth]{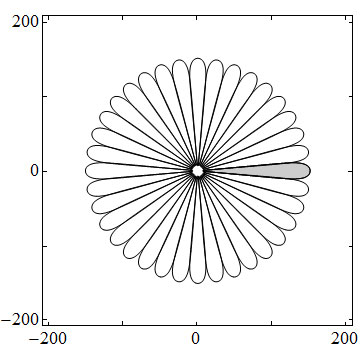}
        \end{subfigure}
        \\
        \begin{subfigure}[b]{0.24\textwidth}
                \centering
                \includegraphics[width=\textwidth]{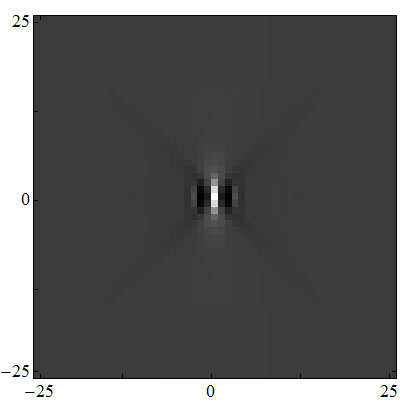}
        \end{subfigure}
        \begin{subfigure}[b]{0.24\textwidth}
                \centering
                \includegraphics[width=\textwidth]{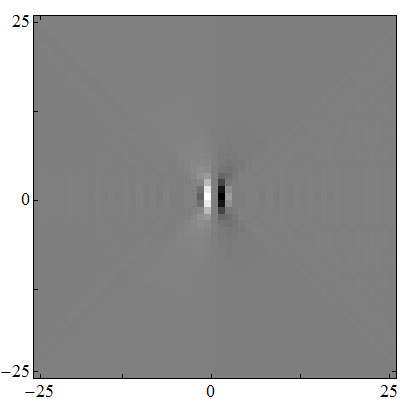}
        \end{subfigure}
        \begin{subfigure}[b]{0.25\textwidth}
                \centering
                \includegraphics[width=\textwidth]{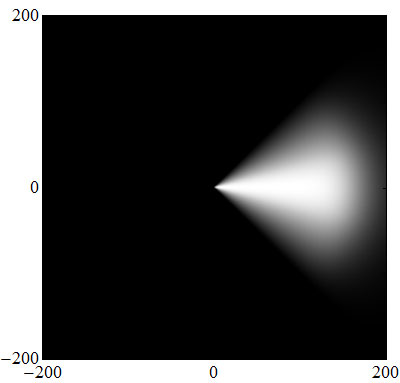}
        \end{subfigure}
        \begin{subfigure}[b]{0.25\textwidth}
                \centering
                \includegraphics[width=\textwidth]{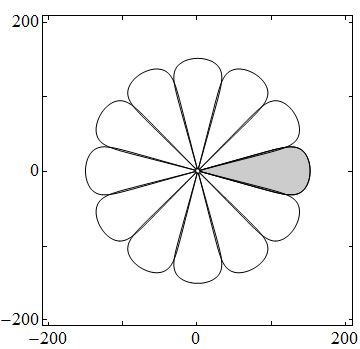}
        \end{subfigure}
        \\
        \begin{subfigure}[b]{0.24\textwidth}
                \centering
                \includegraphics[width=\textwidth]{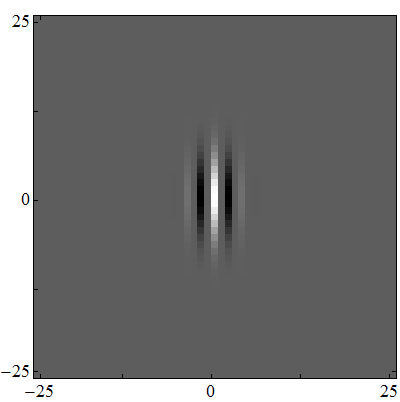}
        \end{subfigure}
        \begin{subfigure}[b]{0.24\textwidth}
                \centering
                \includegraphics[width=\textwidth]{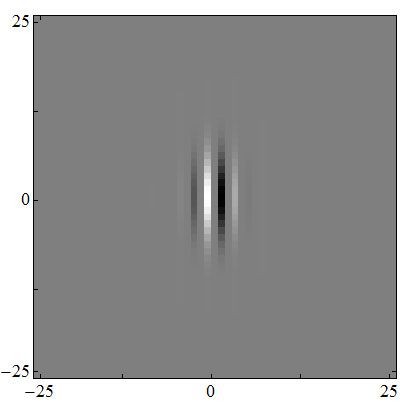}
        \end{subfigure}
        \begin{subfigure}[b]{0.25\textwidth}
                \centering
                \includegraphics[width=\textwidth]{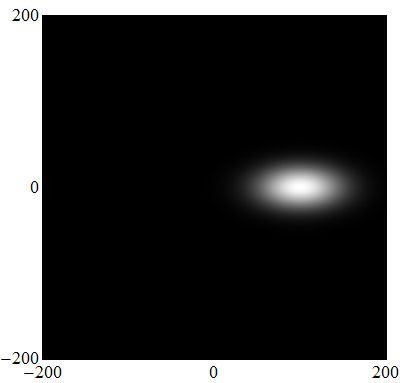}
        \end{subfigure}
        \begin{subfigure}[b]{0.25\textwidth}
                \centering
                \includegraphics[width=\textwidth]{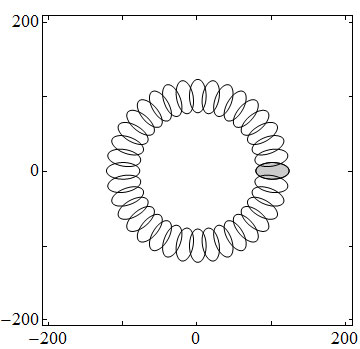}
        \end{subfigure}
\caption{Overview of the wavelets used in this paper. From left to right: the real and imaginary parts of the wavelet in the spatial domain (zoomed by a factor of 8 for the sake of visualization), the wavelet in the Fourier domain and an illustration of the Fourier domain coverage by the filters where contours are drawn at 80\% of the filter maximum.
Note that this last figure also gives an impression of $M_{\psi}$. The top row shows the cake wavelet constructed using $N_o = 36$, middle row with $N_o = 12$ and the bottom
row shows the Gabor wavelet at scale $a = 6/\pi$ and $N_o = 36$.}
\label{fig:filters}
\end{figure*}

The approach of constructing wavelets directly from the Fourier domain allows indirect control over the spatial shape of the filter, and it can easily be adapted. For example, the number of orientations $N_o$ specifies the angular resolution $s_{\theta}$: If $N_o$ is large, the resolution in the orientation dimension is large and the filters become very narrow. This is illustrated in Fig.~\ref{fig:filters}. The cut-off frequency (at the inflection point) of the function $\mathcal{M}_n$, which is usually set as the Nyquist-frequency, could be lowered to filter out high-frequency noise components. Moreover, because B-splines and the function $\mathcal{M}_n$ are used to sample the Fourier domain, the sum of all cake wavelets is approximately 1 over the entire Fourier domain (within a ball of radius $\gamma \varrho$), see Fig.~\ref{fig:BSplines} and \ref{fig:filters}. Thus the cake wavelets indeed are proper wavelets, allowing stable reconstruction via Eq.~(\ref{eq:reconstruction}). In Eq.~(\ref{eq:reconstruction}) one can omit division by $M_\psi^{-1} \approx 1$ in which case stable reconstruction is obtained both by integration over $SE(2):=\mathbb{R}^2 \rtimes S^1$ and/or its partially discrete subgroup $\mathbb{R}^2 \rtimes S_N^1$, with $S_N^1 = \{e^{n (2 \pi i / N_o)}|n=0,1,...,N_o-1\}$.

\paragraph{Remark}
If $N_o \to \infty$ then $\tilde{\psi}^{cake} \to \delta_0$ in the distributional sense. In this case the wavelet transform converges to the localized Radon transform, which has been proposed for effective retinal vessel detection in \cite{Krause2013}. The advantage however of taking $N_0 \ll \infty$ is that we obtain well-posed non singular kernels in the spatial domain while allowing a stable reconstruction for all a-priori set $N_o \in \mathbb{N}$.\\

Cake wavelets are quadrature filters, meaning that the real part contains information about the locally even (symmetric) structures, e.g. ridges, and the imaginary part contains information about the locally odd (antisymmetric) structures, e.g. edges. That is, the real and imaginary part of the filter $\psi_{\theta}$ are related to each other by the Hilbert transform in the direction perpendicular to the wavelets orientation, which is defined by
\begin{equation}
\mathcal{H}^\eta(\psi_\theta)(\mathbf{x})=\mathcal{F}^{-1}[\omega \mapsto i \; \operatorname{sign}(\www \cdot \mathbf{e}_\eta)\mathcal{F}[\psi_\theta](\www)](\mathbf{\mathbf{x}}),
\end{equation}
where $\mathbf{e}_\eta$ specifies the direction in which the Hilbert transform is performed, recall Eq.~(\ref{eq:frameOfReference}).

The quadrature property is useful in our vessel tracking approach, since it allows us to directly detect vessel edges from the imaginary part of the orientation scores, without having to calculate first-order derivatives perpendicular to the vessel orientation. In our applications we did remove the DC-component for the real part to avoid responses on locally constant images. Fig.~\ref{fig:filters} shows the real and imaginary part of the cake wavelet in the spatial domain, as well as the coverage of the wavelet in the Fourier domain.

A final remark on cake wavelets: Since the cake wavelets uniformly cover the Fourier domain ($\mathcal{M}_N \approx 1$), they allow us to use a fast approximate reconstruction scheme, which is given by integration of the orientation scores over the angles only:
\begin{equation}
f^{approx}(\mathbf{x}) \approx \cfrac{1}{2\pi} \int_0^{2\pi}U_f^{cake}(\mathbf{x},\theta)d\theta,
\end{equation}
for details see \cite{Duits2005}.

\begin{figure*}[!ht]
        \centering
        \begin{subfigure}[t]{0.45\textwidth}
                \centering
                \includegraphics[width=\textwidth]{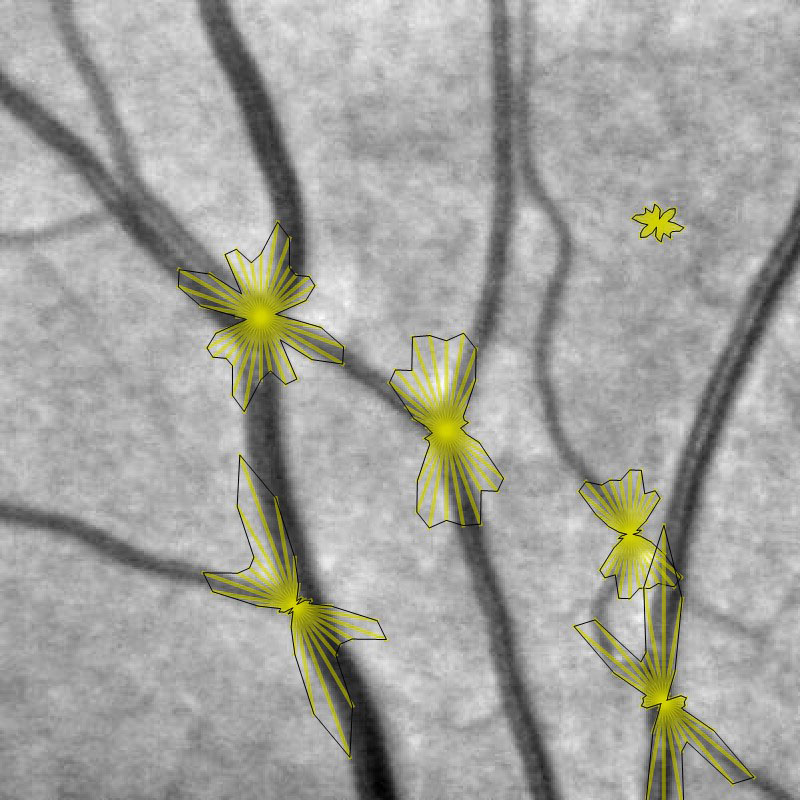}
                \caption{Double-sided}
        \end{subfigure}
        \begin{subfigure}[t]{0.45\textwidth}
                \centering
                \includegraphics[width=\textwidth]{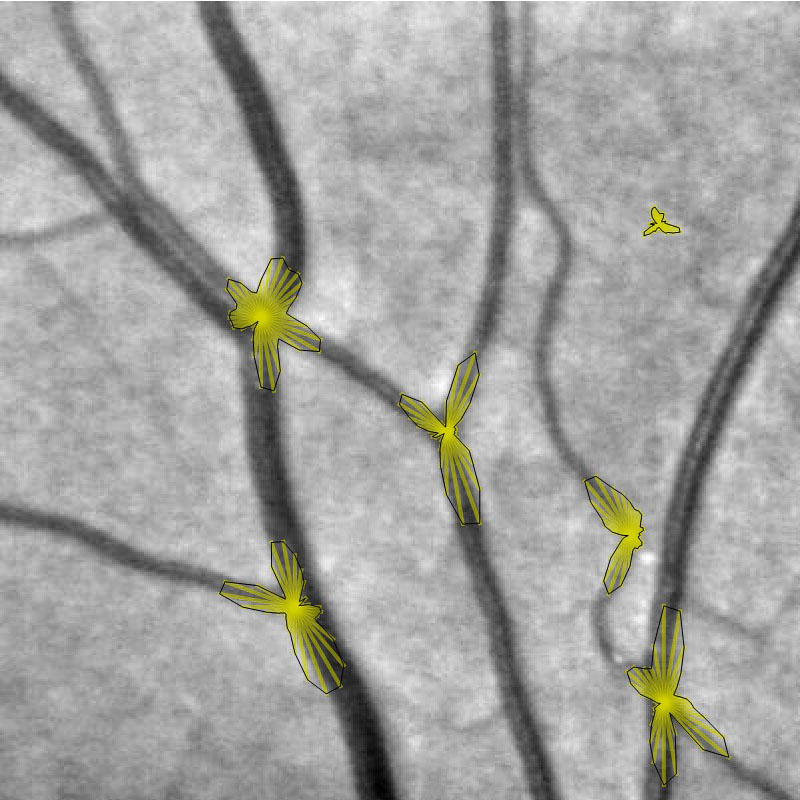}
                \caption{Single-sided}
        \end{subfigure}
        \caption{Comparison between double- and single-sided cake wavelets by visualisation of the orientation column $U_f(\mathbf{x},\cdot)$ at several points in a fundus image. The orientation column at a certain point $\mathbf{x}$ is visualized by drawing 36 lines, each at a certain angle $0 \leq \theta < 2\pi$, and of which the length in direction $\theta$ is given by the absolute value of the score $|U_f(\mathbf{x},\theta)|$.}
        \label{fig:singleVSdouble}
\end{figure*}

\subsection{Gabor wavelets}
\label{sec:GaborWavelets}

Gabor wavelets are directional wavelets, and can be tuned to specific spatial frequencies (and inherently scales). In the field of retinal image processing they are used for vessel detection in various studies \cite{Soares2006,Li2006}. We can exploit the tuning of the wavelet to specific spatial frequencies to match differently
sized blood vessels. The 2D Gabor wavelet is a Gaussian, modulated by a complex exponential, and is defined as:

\begin{equation}
\begin{array}{rl}
\psi^{Gabor}(\mathbf{x}) &= \cfrac{1}{C_{\psi}} \; e^{i\mathbf{k}_0\mathbf{x}} \; e^{-{\cfrac{1}{2}}{\lvert\mathbf{A}\mathbf{x}\rvert}^2},\\
C_{\psi} &= 2 \pi \sqrt{\epsilon} \; e^{-{\cfrac{1}{2}}{\lvert\mathbf{A}^{-1}\mathbf{k}_0\rvert}^2},
\end{array}
\end{equation}
where $\mathbf{A} = diag[\epsilon^{-1/2},1]$ with $\epsilon \geq 1$ is a $2 \times 2$ diagonal matrix that defines the anisotropy of the wavelet. The vector $\mathbf{k}_0$ defines the spatial frequency of the complex exponential and $C_{\psi}$ normalizes the wavelet to unity. In our method we use $\epsilon = 4$, which makes the filter elongated in the
x-direction and we choose $\mathbf{k}_0 = (0,3)$, which causes oscillations perpendicular to the orientation of the wavelet. We can dilate the filter by a scaling parameter $a>0$:
\begin{equation}
\psi_{a}^{Gabor}(\mathbf{x}) = a^{-1}\psi^{Gabor}(a^{-1}\mathbf{x}).
\end{equation}
Orientation scores constructed from an image $f$ using Gabor wavelets at scale $a$ are denoted by $U_{f, a}^{Gabor}$. Fig.~\ref{fig:orientationScoreFiltering}d-e demonstrate typical Gabor wavelet based orientation scores at a small and large scale respectively, for a certain orientation.

The real and imaginary part of the wavelet in the spatial domain, as well as the coverage of the wavelet in the Fourier domain are shown in Fig.~\ref{fig:filters}. Similar to the cake wavelets, Gabor wavelets also have the quadrature property orthogonal to their orientation.

In the Fourier domain, the Gabor filters are represented as Gaussian functions shifted from the origin by $\mathbf{k}_0$. The set of all rotated Gabor functions at a certain scale $a$ covers therefore a certain annulus in the Fourier domain, and a single Gabor wavelet can thus be regarded as an oriented band-pass filter. This is also clearly depicted by the outlines of the frequency responses as shown in Fig.~\ref{fig:filters}.

\subsection{Double-sided vs single-sided wavelets, orientation vs direction}
The cake and Gabor wavelets are double-sided wavelets which do not distinguish between a forward or backward direction (they are symmetric with respect to the y-axis). In order to distinguish between $\pi$ symmetries and $2\pi$ symmetries (see Fig.~\ref{fig:singleVSdouble}), and to be able to handle bifurcations, we decompose the orientation scores into a forward and backward direction\footnote{I.e. we extend the domain $SE(2)=\mathbb{R}^2 \rtimes SO(2)$ of our orientation scores to the group $E(2) = \mathbb{R}^2 \rtimes O(2)$, where $O(2) = \{M \in \mathbb{R}^{2\times2}|M^T = M^{-1}\}$ also includes, besides rotations (with $\det M = +1$), reflections (with $\det M = -1$).}, denoted by a $+$ and $-$ symbol respectively:
\begin{equation}
U_f(x,y,\theta)=U_f^+(x,y,\theta)+U_f^-(x,y,\theta),
\end{equation}
where
\begin{equation}
\begin{array}{l}
U_f^+(\mathbf{x},\theta) = \int_{\mathbb{R}^2} \overline{\psi^+(\mathbf{R}_\theta^{-1}(\mathbf{y}-\mathbf{x}))} f(\mathbf{y}) d\mathbf{y},\\
U_f^-(\mathbf{x},\theta) = \int_{\mathbb{R}^2} \overline{\psi^-(\mathbf{R}_\theta^{-1}(\mathbf{y}-\mathbf{x}))}f(\mathbf{y})d\mathbf{y},
\end{array}
\end{equation}
and where
\begin{equation}
\begin{array}{l}
\psi^+(x,y)=w(x)\psi(x,y),\\
\psi^-(x,y)=w(-x)\psi(x,y) = (1-w(x))\psi(x,y)
\end{array}
\end{equation}
with
\begin{equation}
w(x) = \cfrac{1}{2}+\cfrac{1}{2}\operatorname{erf}(x) = \cfrac{1}{2} + \cfrac{1}{2\pi}\int_0^x e^{-y^2}dy.
\end{equation}
Note that by using the error function, we have $\psi = \psi^- + \psi^+$ and $U_f^-(x,y,\theta) = \overline{U_f^+(x,y,\theta + \pi)}$, so that $U_f(x,y,\theta) = U_f^+(x,y,\theta)+\overline{U_f^+(x,y,\theta+\pi)}$. It is thus possible to choose one of the single-sided wavelets to construct a \emph{directional} orientation score, while still being able to access the original (double-sided) orientation score. Fig.~\ref{fig:singleVSdouble} demonstrates the advantage of using single-sided wavelets over double-sided wavelets in the case of direction estimation based on the orientation column of a score $U_f(x,y,\cdot)$.

\begin{figure*}[!ht]
        \begin{subfigure}[b]{0.41\textwidth}
                \centering
                \includegraphics[height=0.24\textheight]{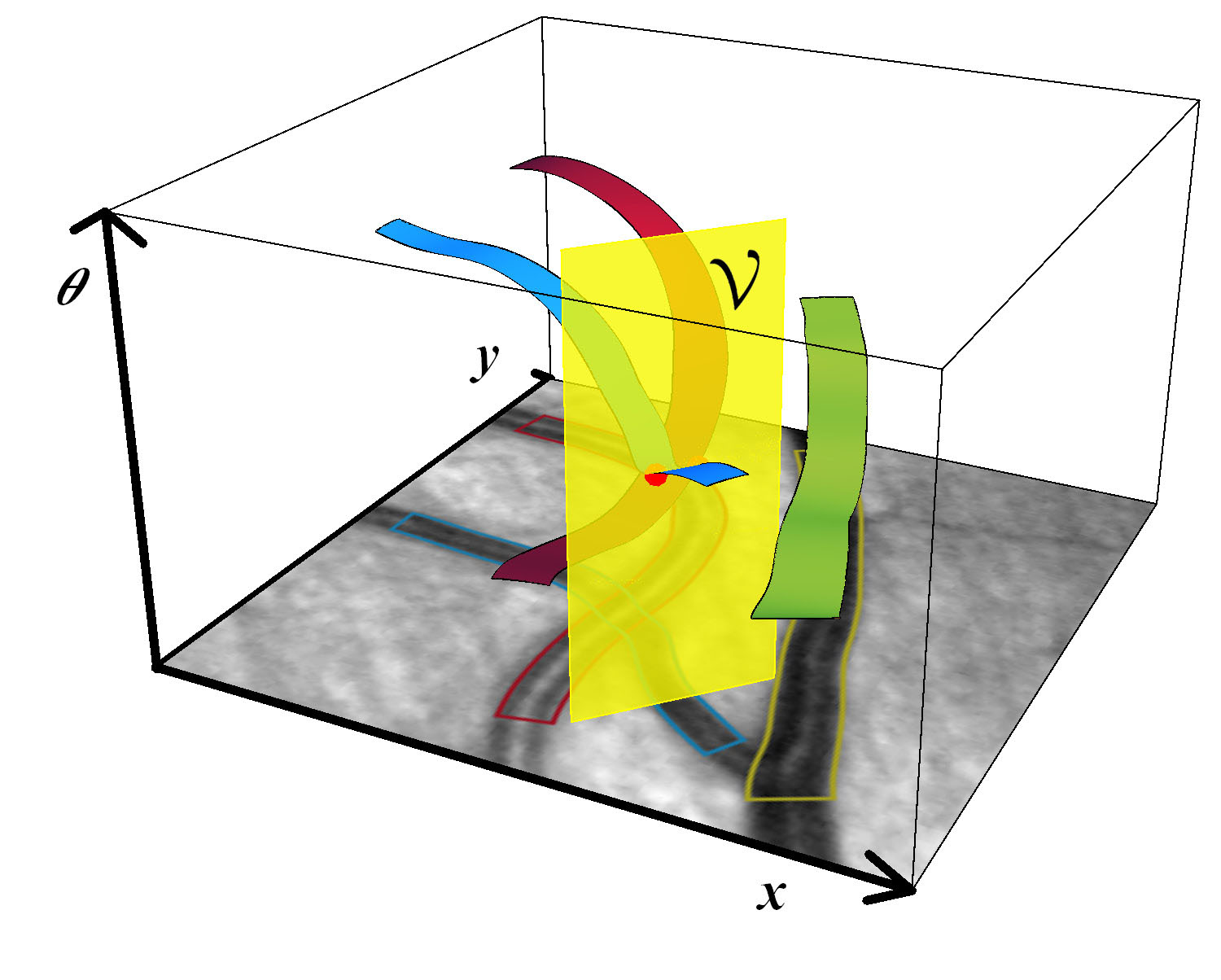}
                \caption{Vessels in OS and detection plane}
                \label{fig:ETOS3D:a}
        \end{subfigure}
        \begin{subfigure}[b]{0.15\textwidth}
                \centering
                \includegraphics[height=0.24\textheight]{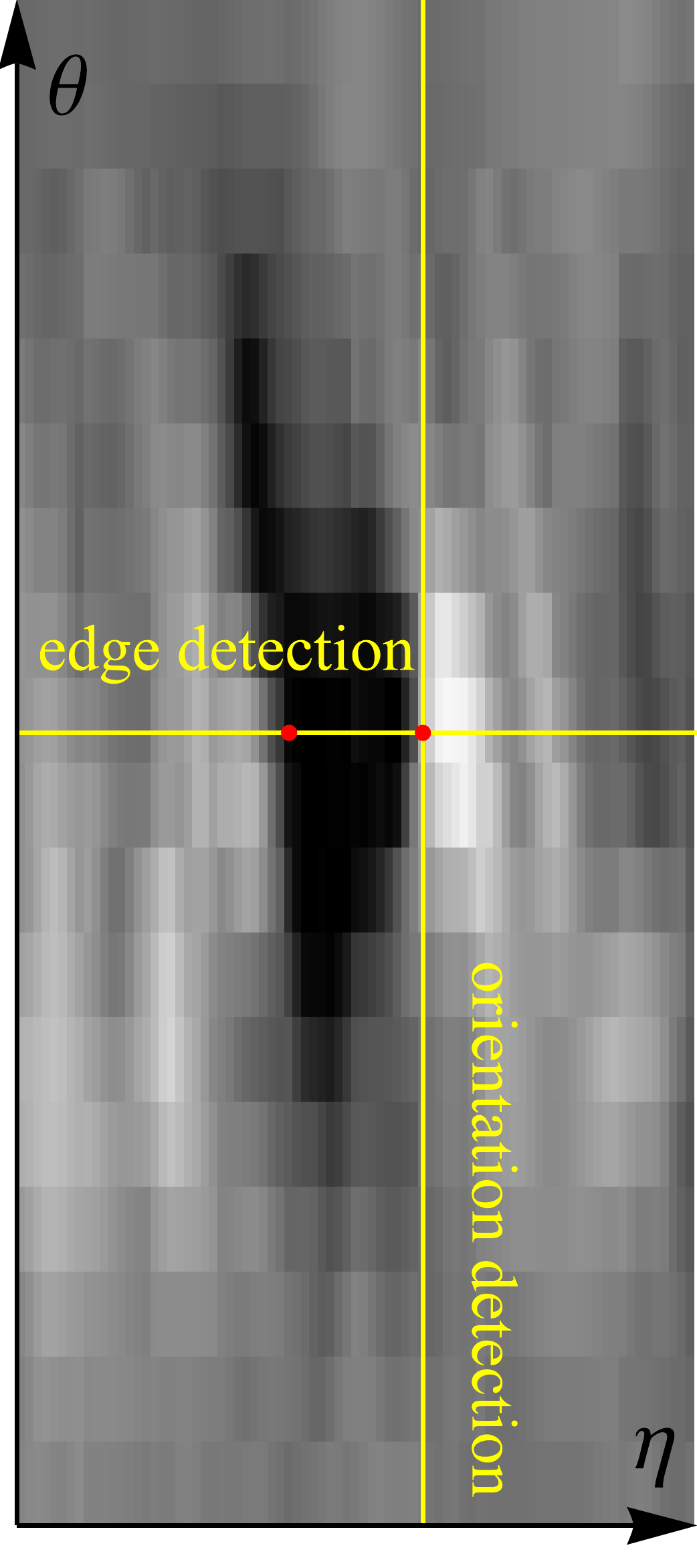}
                \caption{Real}
                \label{fig:ETOS3D:b}
        \end{subfigure}
        \begin{subfigure}[b]{0.15\textwidth}
                \centering
                \includegraphics[height=0.24\textheight]{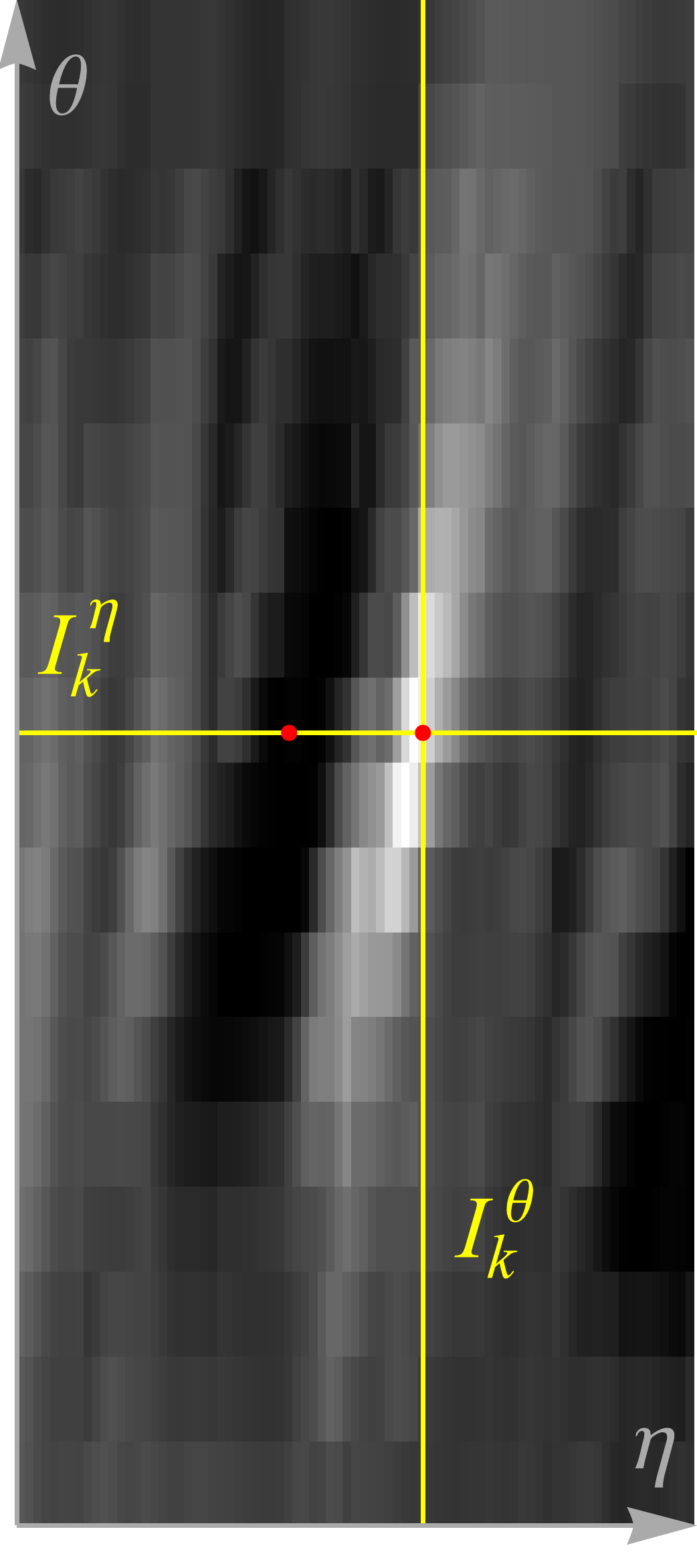}
                \caption{Imaginary}
                \label{fig:ETOS3D:c}
        \end{subfigure}
        \begin{subfigure}[b]{0.25\textwidth}
                \centering
                \includegraphics[height=0.12\textheight]{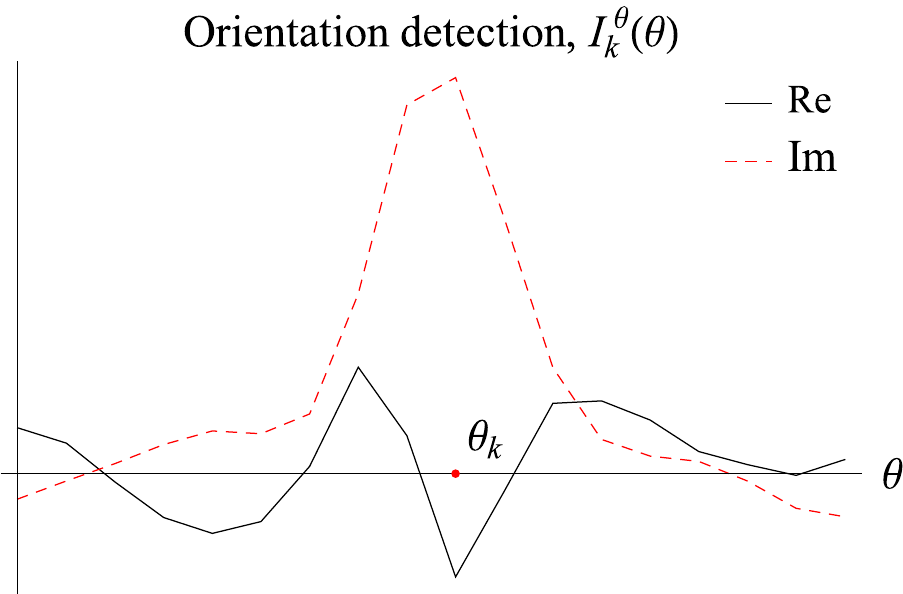}
                \includegraphics[height=0.12\textheight]{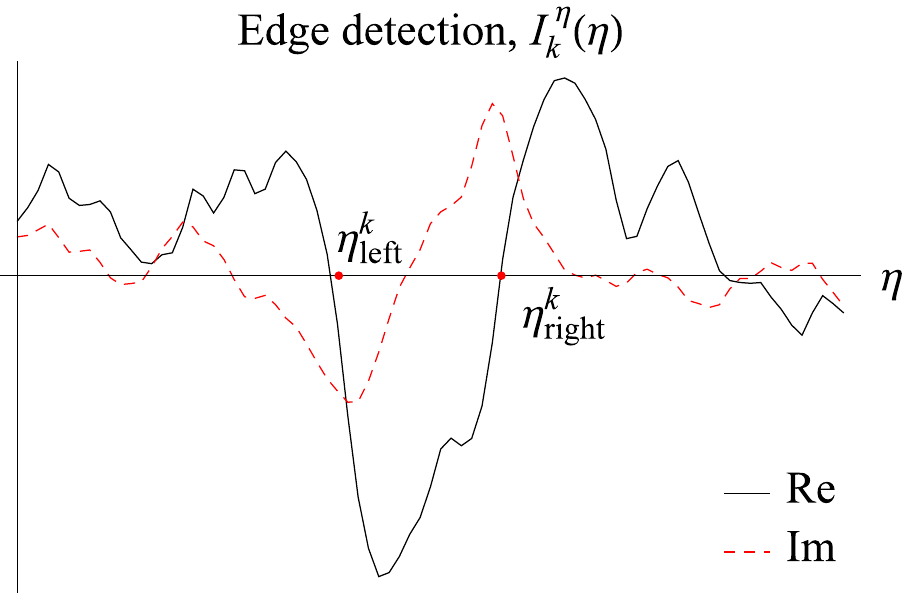}\hfil
                \caption{Detection profiles}
                \label{fig:ETOS3D:d}
        \end{subfigure}
        \caption{Edge tracking in a $\pi$-periodic orientation score constructed from double-sided wavelets. (a) Graphical representation of blood vessels in the orientation score. The real and imaginary part of the orientation score on the yellow plane $\mathcal{V}$ (perpendicular to the blood vessel) are represented in (b) and (c) respectively. In (c) the left and right edge of the blood vessel are expressed as black and white blobs respectively. The edge and orientation detection profiles are demonstrated in (d).}
        \label{fig:ETOS3D}
\end{figure*}

\section{Vessel tracking in orientation scores via optimization in transversal tangent planes $\mathcal{V}$}
\label{sec:VesselTrackingUsingOrientationScores}
In orientation scores, information from the original image is both maintained and neatly organized in different orientations, leading to the disentanglement of crossing structures. Moreover, because of the quadrature property of the wavelets used in the construction of orientation scores, important edge information is well represented in the imaginary part of the score. We therefore propose tracking algorithms that directly act on the orientation score:
\begin{enumerate}
  \item The ETOS algorithm: Edge Tracking based on Orientation Scores (Section~\ref{sec:EdgeTrackingUsingInvertibleOrientationScores}).
  \item The CTOS algorithm: Centerline Tracking based on multi-scale Orientation Scores (Section~\ref{sec:MultiScaleVesselCenterLineTracking}).
\end{enumerate}
In both tracking algorithms we rely on a fundamental geometric principle to extract the most probable paths in orientation scores. Consider to this end Fig.~\ref{fig:ETOS3D}a, where a track $t \mapsto g(t) = (\mathbf{x}(t),\theta(t))$ is considered locally optimal if it is locally optimized in each\footnote{That is locally optimal in $\mathcal{V}|_g$ for each $g=g(t)$, $t\in \textrm{Dom(g)}$.} transversal 2D-tangent plane\footnote{tangent vectors can be considered as differential operators acting on smooth locally defined functions \cite{Aubin}. In our case this boils down to replacing $\mathbf{e}_{x}$ by $\partial_{x}$, $\mathbf{e}_{y}$ by $\partial_{y}$ and
$\mathbf{e}_{\theta}$ by $\partial_{\theta}$.}
\begin{equation}
\mathcal{V}|_g=\textrm{span}\{\left.\partial_{\theta}\right|_g, \left.\partial_{\eta}\right|_g\},
\end{equation}
spanned by $\partial_{\theta}$ and $\partial_{\eta}=-\sin \theta \partial_{x} + \cos \theta \partial_{y}$ in tangent-bundle
\[
T(SE(2))=\bigcup \limits_{g \in SE(2)} T_{g}(SE(2)),
\]
where
$T_{g}(SE(2))$ denotes the tangent space at $g=(x,y,\theta) \in SE(2)$. For more details on this principle we refer to Appendix \ref{app:optimalpaths}.

\subsection{Methods}
\label{sec:Methods}
\subsubsection{ETOS: Edge tracking in orientation scores}
\label{sec:EdgeTrackingUsingInvertibleOrientationScores}
The ETOS algorithm tracks both vessel edges simultaneously through an orientation score. The method iteratively expands a blood vessel model by detecting, at each forward step $k$, the optimal edge locations $(\mathbf{u}_k,\theta_k)$, $(\mathbf{v}_k,\theta_k) \in SE(2)$ from the orientation score. Here $\mathbf{u}_k$ and $\mathbf{v}_k$ denote the 2D left and right edge position respectively, $\theta_k$ denotes the orientation of the blood vessel. At each iteration the vessel center point $\mathbf{c}_k$ and the vessel width $w_k$ are defined as follows:
\begin{gather}
\mathbf{c}_k =\cfrac{\mathbf{u}_k+\mathbf{v}_k}{2},\label{eq:centerPoint}\\
w_k = \| \mathbf{u}_k-\mathbf{v}_k \|\label{eq:width}.
\end{gather}
To describe our method we will rely on a moving frame of reference with basis vectors $\mathbf{e}_{\eta_k}$, $\mathbf{e}_{\xi_k}$ and $\mathbf{e}_{\theta_k}$, which are described by the orientation parameter $\theta_k$ and Eq.~(\ref{eq:frameOfReference}).

In our method the edge positions $g_{u_k}$ and $g_{v_k}$ are detected in the orientation score from the tangent plane $\mathcal{V}$ (yellow plane in Fig.~\ref{fig:ETOS3D:a}). An edge can be detected as a local optimum from the imaginary part of this plane; a local minimum and maximum for the left and right edge respectively (as indicated by the two red dots in Fig.~\ref{fig:ETOS3D:c}). A schematic overview of the tracking process, including the symbols used in this section, is presented in Fig.~\ref{fig:schematicTracking}.

\begin{figure}[!ht!]
\centering
\includegraphics[width=\textwidthtwo]{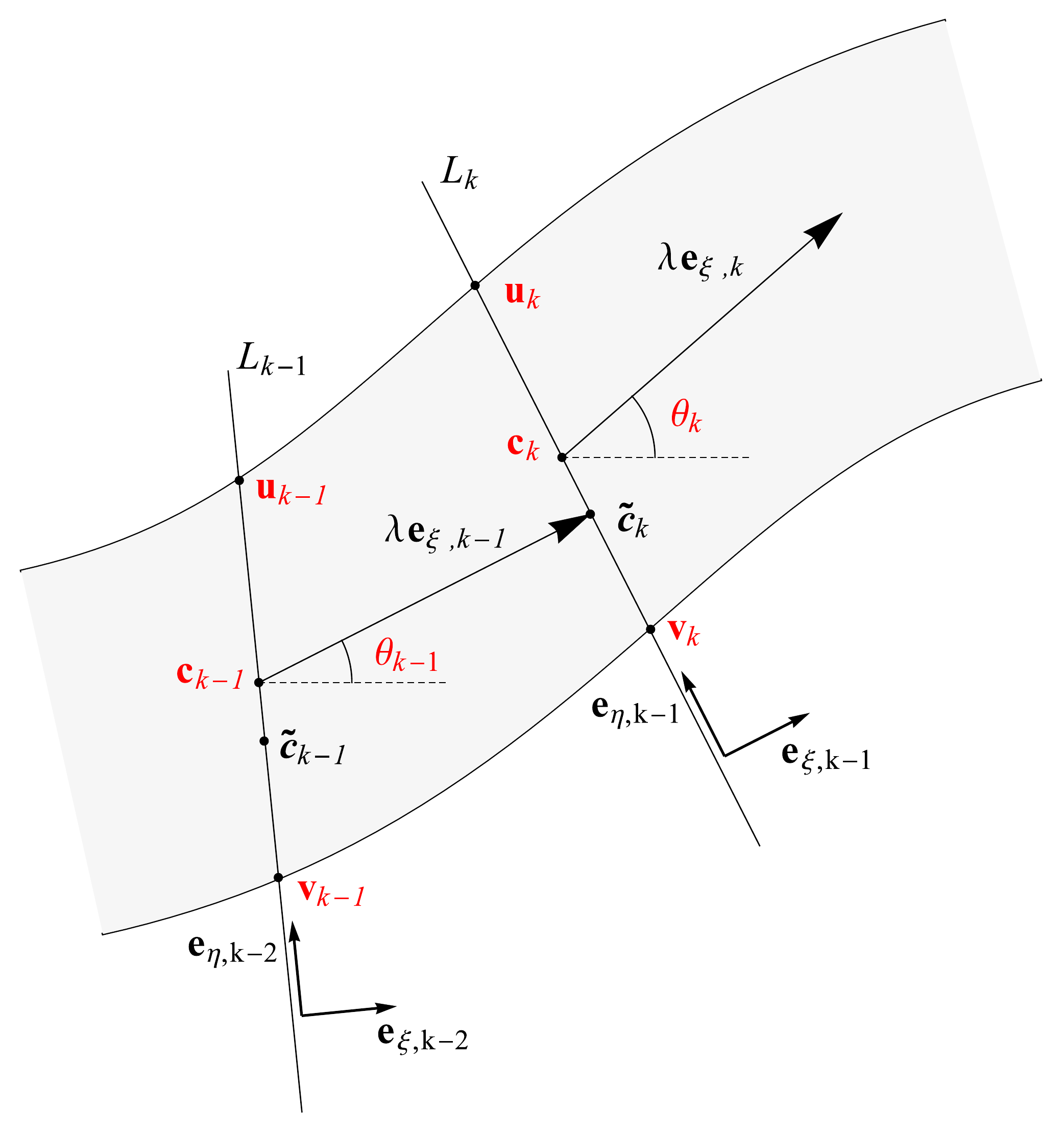}
\caption{Schematic overview image of ETOS. Using the detected vessel center point $\mathbf{c}_{k-1}$ and the orientation $\theta_{k-1}$ detected at the vessel edges $\mathbf{u}_{k-1}$ and $\mathbf{v}_{k-1}$ at iteration $k-1$, a rough estimation of the next center point $\tilde{\mathbf{c}}_{k}$ found by stepping forward in the vessel direction $\mathbf{e}_{\xi,k-1}$ with step size $\lambda$. Through the estimated center point a line $L_{k}$ is defined on which the new vessel edges $\mathbf{u}_{k}$ and $\mathbf{v}_{k}$ are detected. At these edges the vessel orientation $\theta_{k}$ is detected and the final precise vessel center point $\mathbf{c}_{k}$ is calculated as the mean of the two edges.}
\label{fig:schematicTracking}
\end{figure}

For the sake of speed and simplicity, we follow a 2-step approach where the process of detecting the optimal edge positions is separated into two 1D optimization tasks which simply involve the detection of local minima (left edges) and maxima (right edge); in step 1 the edge locations $\mathbf{u}_k$ and $\mathbf{v}_k$ are optimized in the $\mathbf{e}_\eta$ direction ($\eta$-optimization), in step 2 the corresponding orientation is optimized in the $\mathbf{e}_\theta$ direction ($\theta$-optimization), see Fig.~\ref{fig:ETOS3D}b-d. By considering the continuous properties of the blood vessels (e.g. continuous vessel widths), we use a paired edge tracking approach where the left and right edges are detected simultaneously. This approach has the advantage that even if one of the edges is less apparent in the image (e.g. at crossing points, parallel vessels and bifurcations), both edges and their orientation can still be tracked. A possible disadvantage of this approach is however that abrupt changes in vessel width (e.g. at stenoses and aneurysms) may become unnoticed or detected with less detail.

\textbf{Step 1}: Based on a-priori knowledge about the previous vessel orientation, edges and center point, stored in
\[\{(\theta_{l}, \mathbf{u}_{l}, \mathbf{v}_{l} , \mathbf{c}_{l})|k-M \le l \le k-1\},
\]
where we set $M=10$, a new vessel center point $\tilde{\mathbf{c}}_k$ is calculated as \begin{equation}
\label{eq:cest}
\tilde{\mathbf{c}}_k = \mathbf{c}_{k-1} + \lambda \; \mathbf{e}_{\xi_{k-1}},
\end{equation}
where $\lambda$ (typically in the order of 2 pixels) is the tracking step-size.
New edge points are selected from a set of points $\mathbf{p}_k(\eta)$ on a line $L_k$ going through the estimated center point $\tilde{\mathbf{c}}_k$ and perpendicular to the vessel orientation $\theta_{k-1}$:

\begin{equation}
L_k = \{\mathbf{p}_k(\eta)\; | \; \eta \in [-\eta_{max},\eta_{max}]\},
\end{equation}
with
\begin{equation}
\mathbf{p}_k(\eta) = \tilde{\mathbf{c}}_k + \eta \; \mathbf{e}_{\eta_{k-1}},
\label{eq:pk}
\end{equation}
where $\eta$ is a parameter describing the distance to the estimated vessel center point and $\eta_{max} > \norm{\mathbf{u}_k-\mathbf{v}_k}$ denotes the maximum distance to the estimated the vessel center point. Note that orientation $\theta_{k-1}$ of the previous iteration is used as the new orientation is yet to be detected.

\begin{figure*}[!ht]
        \centering
        \begin{subfigure}[b]{0.4\textwidth}
                \centering
                \includegraphics[width=\textwidth]{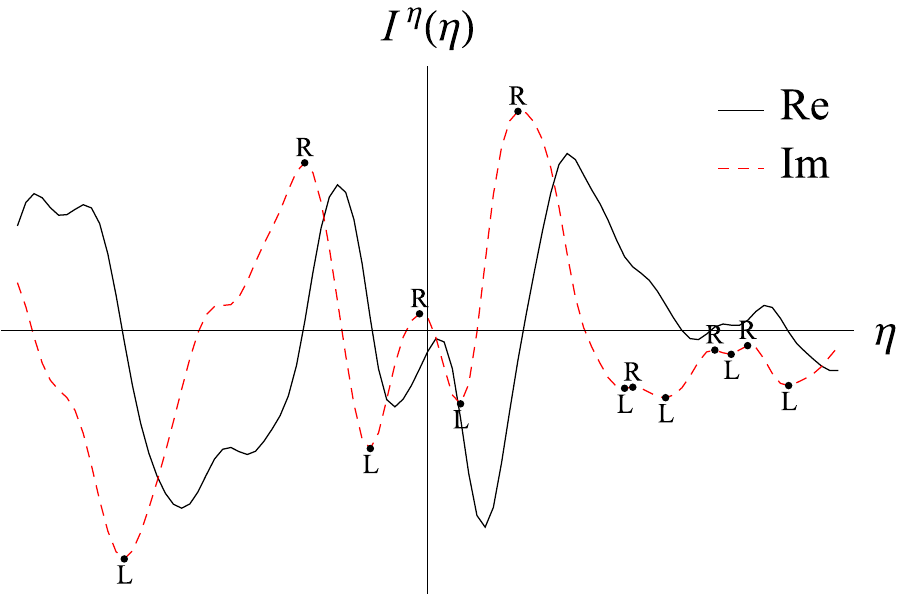}
                \caption{}
        \end{subfigure}
        \begin{subfigure}[b]{0.4\textwidth}
                \centering
                \includegraphics[width=\textwidth]{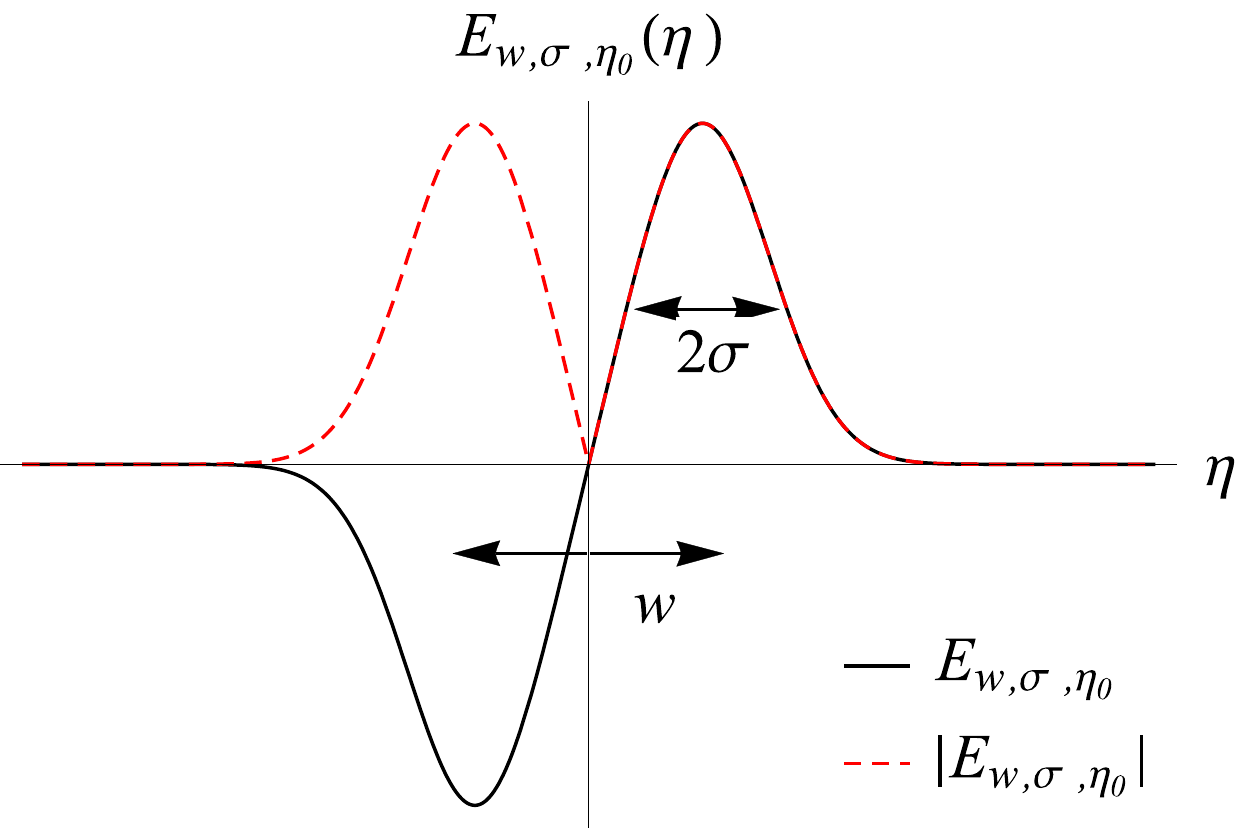}
                \caption{}
        \end{subfigure}\\
        \begin{subfigure}[b]{0.4\textwidth}
                \centering
                \includegraphics[width=\textwidth]{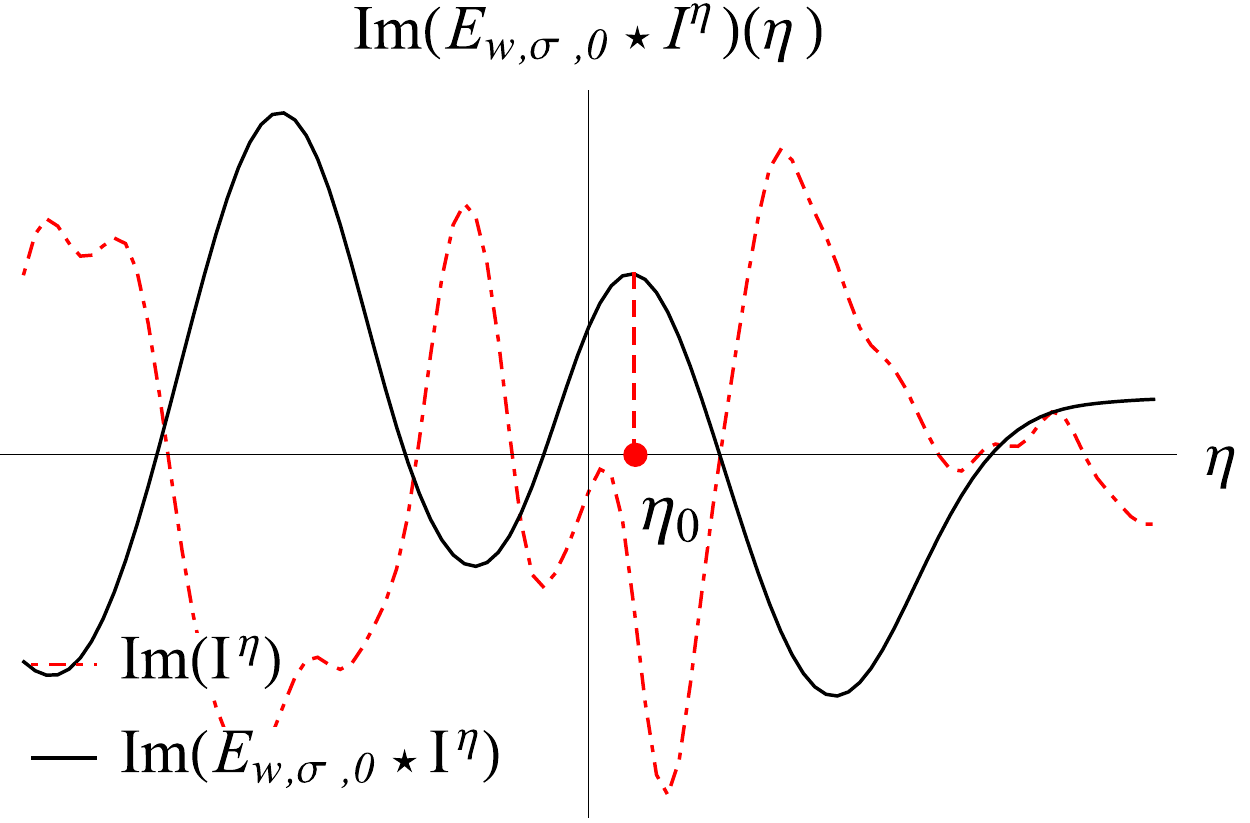}
                \caption{}
        \end{subfigure}
        \begin{subfigure}[b]{0.4\textwidth}
                \centering
                \includegraphics[width=\textwidth]{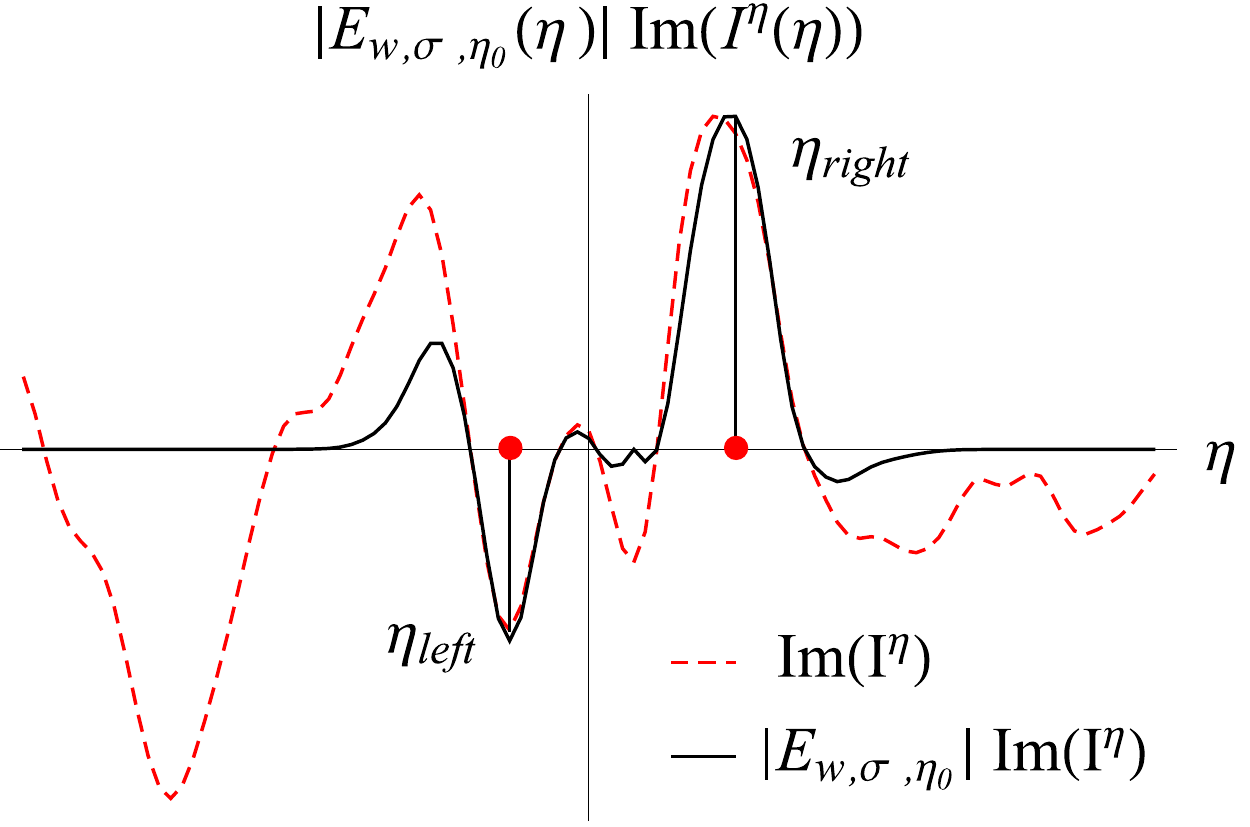}
                \caption{}
        \end{subfigure}
        \caption{Edge detection using the edge probability envelope. (a) Cross-sectional intensity profile taken from Fig.~\ref{fig:orientationScoreFiltering:c}, showing many potential candidate left (L) and right (R) edge positions. (b) The edge probability profile. (c) Centering of the edge probability profile on the vessel of interest by means of correlation. (d) Enveloping the intensity profile results in clearly detectable left and right edge points.}
        \label{fig:EPE}
\end{figure*}

An intensity profile $I_k^\eta(\eta)$ can then be obtained from the orientation scores according to
\begin{equation}
\label{eq:ieta}
I_{k}^\eta(\eta) = U_f(\mathbf{p}_k(\eta),\theta_{k-1}),
\end{equation}
see Fig.~\ref{fig:ETOS3D:d}. New edge points can now easily be found by detecting the local optima on the imaginary part of this profile. However, the detection of vessel edges is made more robust by taking into account that the vessel wall is a continuous structure, and that the width of a blood vessel gradually changes, rather than abruptly. Therefore, we introduce the edge probability envelope. The edge probability envelope is used to indicate the most likely position of the vessel edges and it consists of two Gaussian distributions, one around the expected left vessel edge position and one around the right vessel edge position. The envelope function is given by:
\begin{equation}
\label{eq:edgeEnvelope}
\begin{array}{l}
E_{\overline{w}_k,\sigma,\eta_0}(\eta) = -G_{\sigma}(\eta+\cfrac{\overline{w}_k}{2}-\eta_0)+G_{\sigma}(\eta-\cfrac{\overline{w}_k}{2}-\eta_0),\\
G_{\sigma}(x) = \cfrac{1}{\sigma\sqrt{2\pi}} \; e^{-\cfrac{x^2}{2 \sigma^2}},
\end{array}
\end{equation}
where $\eta_0$ is the estimated location of the vessel center on $L_k$, $\overline{w}_k$ is the mean vessel width calculated over the last $M$ iterations:
\begin{equation}
\overline{w}_k = \cfrac{1}{M}\sum_{m=1}^{M} w_{k-m} = \cfrac{1}{M}\sum_{m=1}^{M} \| \mathbf{u}_{k-m} - \mathbf{v}_{k-m} \|,
\end{equation}
the standard deviation of the Gaussian distributions is denoted with $\sigma$, and $\eta_0$ is used to align the envelope with the actual vessel profile. A robust value for $\eta_0$ is found by optimizing the cross-correlation of the envelope with the imaginary part of the actual profile:
\begin{equation}
\eta_0 = \underset{\eta^*\in[-0.5 \overline{w}_k,0.5 \overline{w}_k]}{\operatorname{argmax}} \int^{-\eta_{max}}_{-\eta_{max}} \Im(I_k^\eta(\eta))E_{\overline{w}_k,\sigma,\eta^*}(\eta) d\eta,
\end{equation}
\begin{figure*}[!ht]
        \centering
        \begin{subfigure}[t]{0.21\textwidth}
                \centering
                \includegraphics[width=\textwidth]{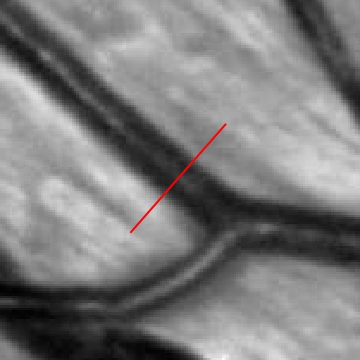}
                \caption{Image, $f(\cdot)$}
        \end{subfigure}
        \begin{subfigure}[t]{0.27\textwidth}
                \centering
                \includegraphics[width=\textwidth]{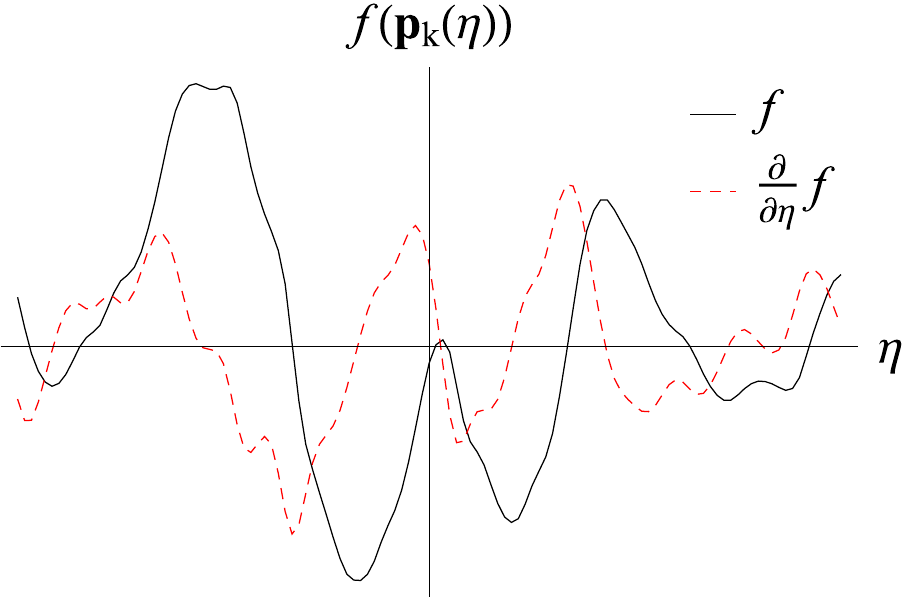}
                \caption{Intensity profile, $f(\mathbf{p}_k(\eta))$}
        \end{subfigure}
        \begin{subfigure}[t]{0.21\textwidth}
                \centering
                \includegraphics[width=\textwidth]{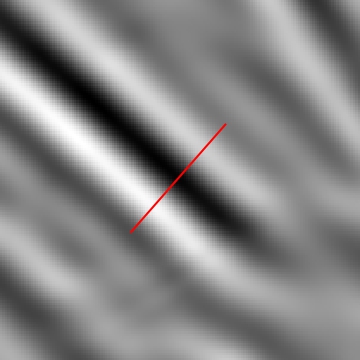}
                \caption{OS, $U_{f,a}^{\eta,Gabor}(\cdot,\theta_v)$}
                \label{fig:orientationScoreFiltering:c}
        \end{subfigure}
        \begin{subfigure}[t]{0.27\textwidth}
                \centering
                \includegraphics[width=\textwidth]{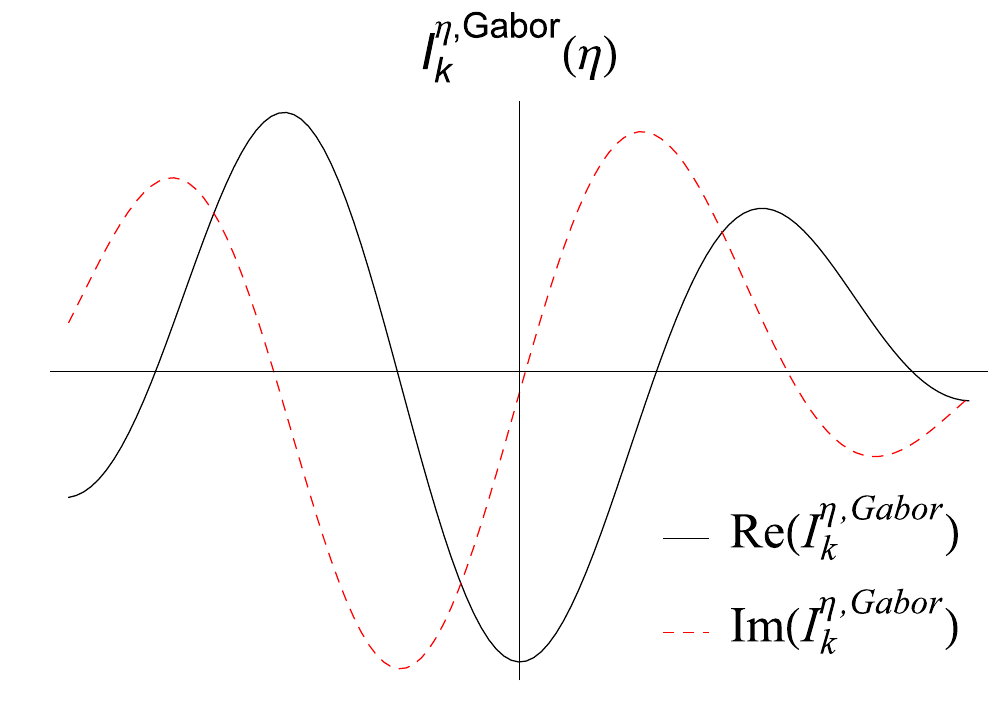}
                \caption{Intensity profile, $I_k^{\eta,Gabor}(\eta)$}
        \end{subfigure}
        \caption{Scale selective orientation scores can filter out a vessels central light reflex. (a) A small sub-image showing a vessel with central light reflex and (b) the corresponding intensity profile. (c) A slice, corresponding to the vessel orientation $\theta_v$, of the orientation score constructed from (a) and (d) the corresponding intensity profile taken hereof. Note that from (d) the vessels center point can be roughly detected as a local minimum.}
        \label{fig:CTOS}
\end{figure*}
see Fig.~\ref{fig:EPE}c. The left and right edges are finally detected as the arguments corresponding to the minimum and maximum points, $\eta_{left}^k$ and $\eta_{right}^k$ respectively, of the product of the envelope and the intensity profile:
\begin{equation}
\begin{array}{ll}
\eta_{left}^k &= \underset{\eta\in[-\eta_{max},\eta_0]}{\operatorname{argmin}} \left\{ \Im \; I_k^\eta(\eta) \; |E_{\overline{w}_k,\sigma,\eta_0}(\eta)| \; \right\},\\
\eta_{right}^k &= \underset{\eta\in[\eta_0,\eta_{max}]}{\operatorname{argmax}} \left\{ \Im \; I_k^\eta(\eta) \; |E_{\overline{w}_k,\sigma,\eta_0}(\eta)| \; \right\},
\end{array}
\end{equation}
see Fig.~\ref{fig:EPE}d for an example. The new edge points can then be assigned by
\begin{equation}
\begin{array}{l}
\mathbf{u}_k = \mathbf{p}_{k}(\eta_{left}^k),\\
\mathbf{v}_k = \mathbf{p}_{k}(\eta_{right}^k).
\end{array}
\end{equation}

\textbf{Step 2}: Orientation $\theta_k$ can be estimated by selecting the orientation that provides the highest orientation score response at both vessel edges. The orientation score response is a combination of the orientation columns at the left and right edges ($U_f(\mathbf{u}_k,\cdot)$ and $U_f(\mathbf{v}_k,\cdot)$ resp.), and the optimal orientation is calculated as:
\begin{equation}
\theta_k = \underset{\theta\in[0,2\pi]}{\operatorname{argmax}} \Im( \; -U_f(\mathbf{u}_k,\theta) + U_f(\mathbf{v}_k,\theta) \;).
\end{equation}
Finally the new center point $\mathbf{c}_k$, which may not be equal to $\tilde{\mathbf{c}}_k$, is calculated as the point between the two edges, according to Eq.~(\ref{eq:centerPoint}).

\subsubsection{CTOS: Multi-scale vessel center-line tracking in orientation scores}
\label{sec:MultiScaleVesselCenterLineTracking}
In this section the scale-selective property of the Gabor wavelets is exploited in the design of a fast orientation score based method called CTOS. The potential presence of a central light reflex in a blood vessel makes the design of a simple and fast centerline tracking algorithm based on local minima tracking in the image nearly impossible. However, by using Gabor wavelets and appropriate scale selection, central light reflexes can be filtered out such that vessel center points can easily be detected as local minima on detection profiles (see Fig.~\ref{fig:CTOS}).

The CTOS algorithm follows the same geometric principle as the ETOS algorithm, for a mathematical underpinning see Appendix~\ref{app:optimalpaths}, however CTOS is done on the real part of orientation scores in order to find the vessel center line, rather than the vessel edges. In this algorithm each iteration $k$ consists of 3-steps: In step 1 the center point $\mathbf{c}_k$ is detected ($\eta$-optimization), in step 2 the orientation $\theta_k$ is detected ($\theta$-optimization) and in step 3 the vessel scale $a_k$ is detected ($a$-optimization).

\textbf{Step 1}: Using the vessel center point $\mathbf{c}_{k-1}$, orientation $\theta_{k-1}$ and scale $a_{k-1}$, which were detected during iteration $k-1$, phase 1 at iteration $k$ is started by estimating the new center point $\tilde{\mathbf{c}}_k$ as given by Eq.~(\ref{eq:cest}). The new center point $\mathbf{c}_k$ is selected from a set of candidate points $\mathbf{p}_k(\eta)$ as given by Eq.~(\ref{eq:pk}). From the candidate points $\mathbf{p}_k(\eta)$ a center point detection profile $I_k^{\eta,Gabor}(\eta)$ is obtained by:
\begin{equation}
I_k^{\eta,Gabor}(\eta) = U_{f,a_{k-1}}^{Gabor}(\mathbf{p}_k(\eta),\theta_{k-1})
\end{equation}
with
$U_{f,a_{k-1}}^{Gabor}$ the orientation score generated by the Gabor wavelets at scale $a_{k-1}$. The new center point is detected as the coordinate belonging to the local minimum on the intensity profile nearest to $\tilde{\mathbf{c}}_{k}$:

\begin{equation}
\label{eq:edgeEnvelope}
\begin{array}{l}
\mathbf{c}_k = \mathbf{p}_{k}(\eta_{0}),\\
\eta_{0} = \underset{\eta\in[-\eta_{max},+\eta_{max}]}{\operatorname{argmin}} \left\{ \; \operatorname{Re} I_k^{\eta,Gabor}(\eta) \; \right\}.
\end{array}
\end{equation}

\textbf{Step 2}: The new vessel orientation is detected as the local maximum of the negative orientation response, nearest to the previous vessel orientation $\theta_{k-1}$:
\begin{equation}
\label{eq:orientationDetectionGabor}
\theta_k = \underset{\theta\in[0,2\pi]}{\operatorname{argmax}} \operatorname{Re}( \; -U_{f,a_{k-1}}^{Gabor}(\mathbf{c}_k,\theta) \;).
\end{equation}

\textbf{Step 3}: At the new center point $\mathbf{c}_k$ and orientation $\theta_k$, scale $a_k$ is detected as the scale that gives the largest negative response:
\begin{equation}
\label{eq:scaleDetection}
a_k = \underset{a>0}{\operatorname{argmax}} \operatorname{Re}( \; -U_{f,a}^{Gabor}(\mathbf{c}_k,\theta_k) \;).
\end{equation}

Compared to the ETOS algorithm, this algorithm is very fast since it only requires three basic (deterministic) detection steps. However the combination of scale and orientation detection makes the algorithm slightly less stable: orientation detection depends on the correct detection of scale and vice versa.

\begin{figure*}[!ht]
\begin{center}$
\begin{array}{cccc}
Seed \; points & ETOS \; (cake) & ETOS \; (Gabor) & CTOS \\
\includegraphics[width=0.24\textwidth]{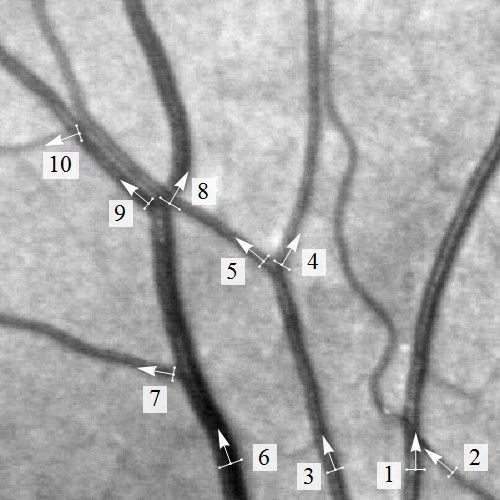}&
\includegraphics[width=0.24\textwidth]{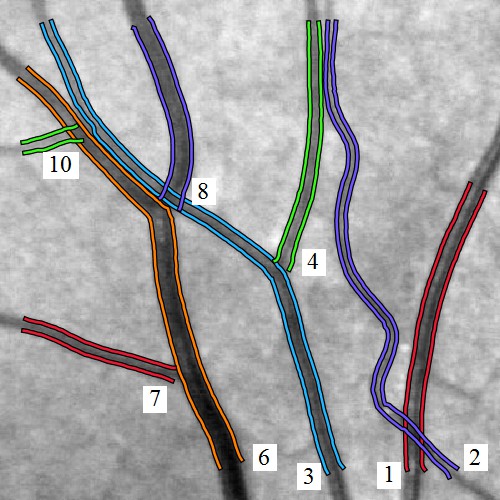}&
\includegraphics[width=0.24\textwidth]{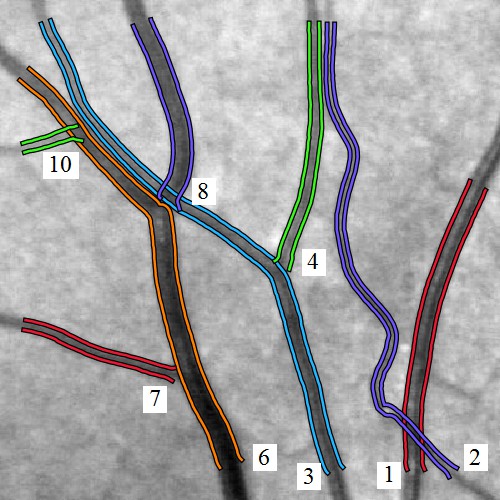}&
\includegraphics[width=0.24\textwidth]{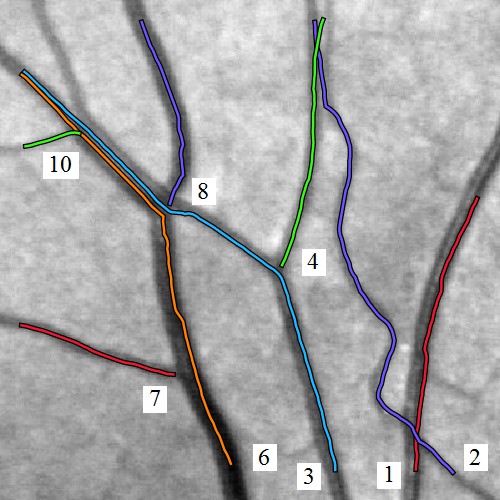}\\
\includegraphics[width=0.24\textwidth]{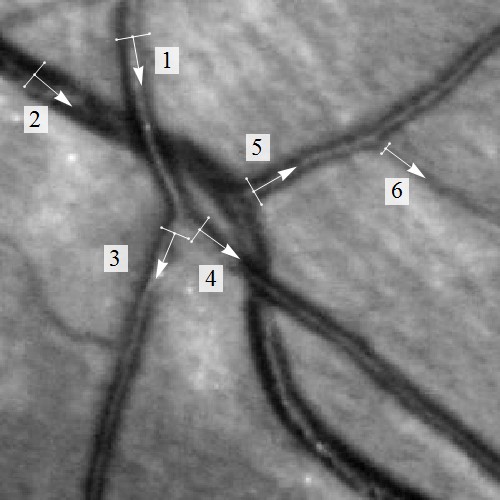}&
\includegraphics[width=0.24\textwidth]{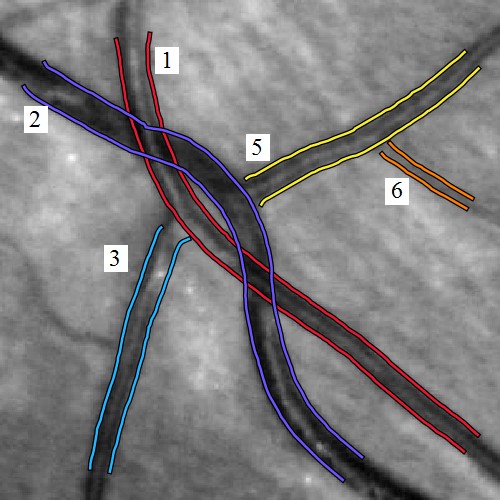}&
\includegraphics[width=0.24\textwidth]{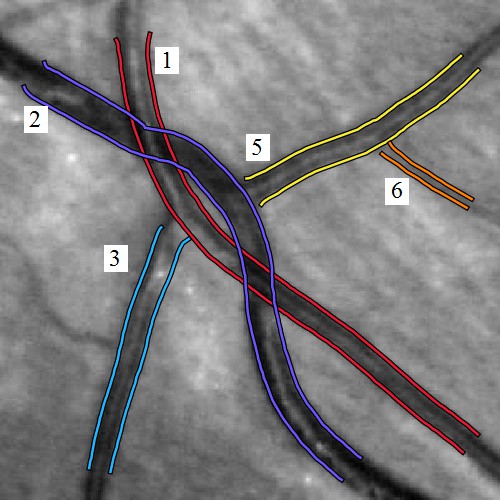}&
\includegraphics[width=0.24\textwidth]{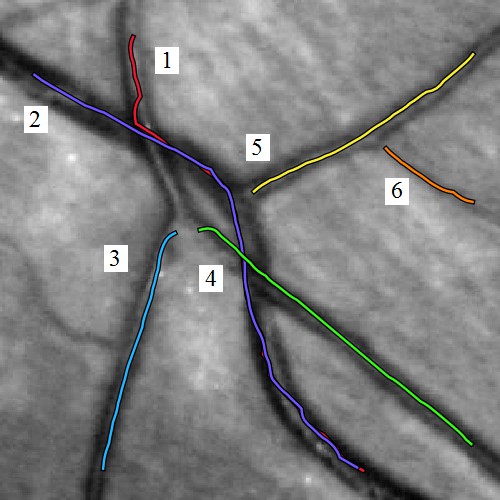}\\
\includegraphics[width=0.24\textwidth]{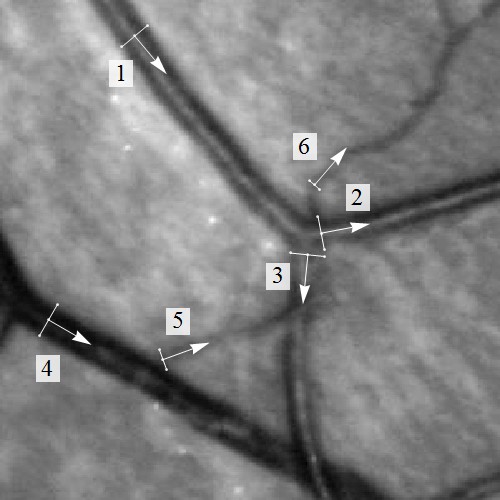}&
\includegraphics[width=0.24\textwidth]{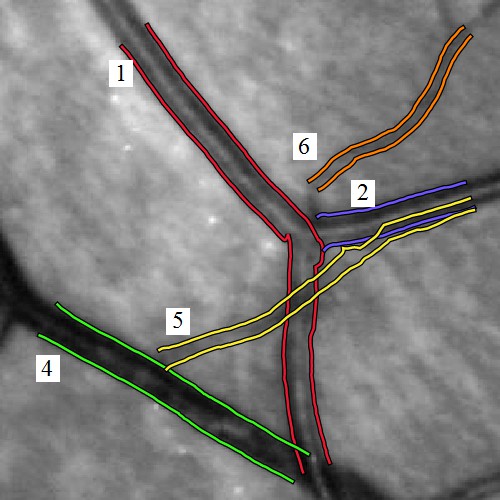}&
\includegraphics[width=0.24\textwidth]{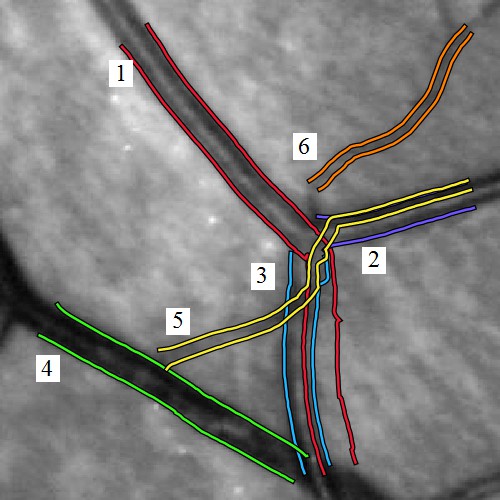}&
\includegraphics[width=0.24\textwidth]{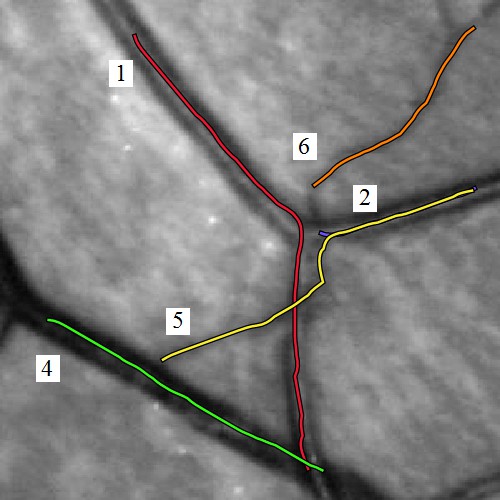}\\
\includegraphics[width=0.24\textwidth]{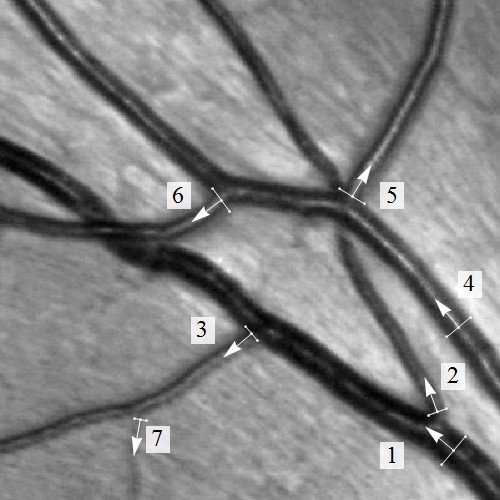}&
\includegraphics[width=0.24\textwidth]{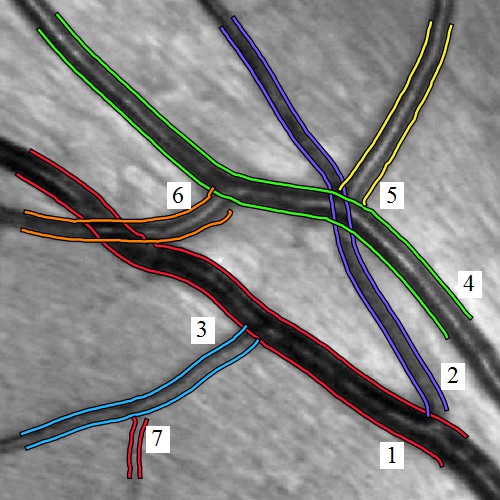}&
\includegraphics[width=0.24\textwidth]{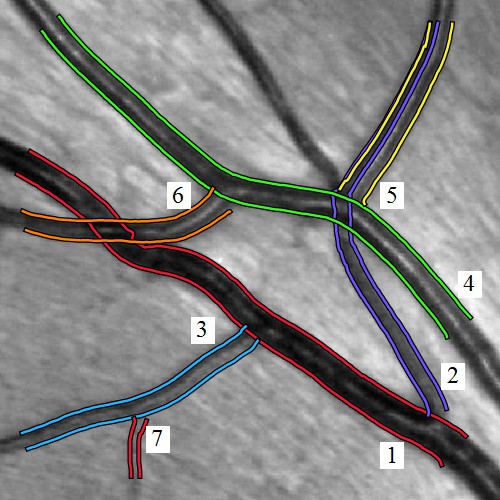}&
\includegraphics[width=0.24\textwidth]{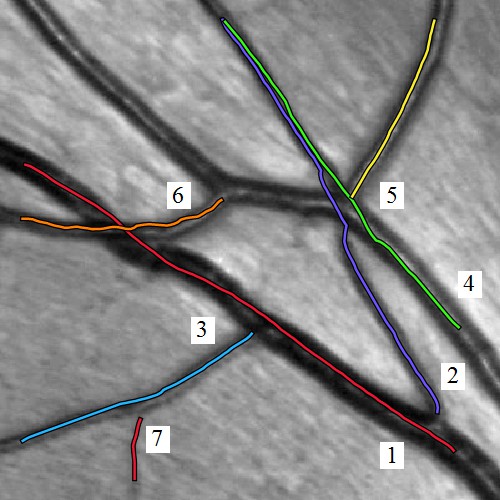}
\end{array}$
\end{center}
\caption{Results of vessel tracking on the test image set. From left to right: seed points, tracking results using the ETOS algorithm using invertible orientation scores (cake
wavelets), tracking results of the ETOS algorithm using non-invertible orientation scores (Gabor wavelets at scale $\tau = 10$) and tracking results of the CTOS algorithm using
orientation scores constructed at scales $\tau = $5, 10, 15, 20, 25 and 30. Note that the results of the CTOS algorithm are only represented as centerlines since vessel width
is not measured. From top to bottom: results on test image 1, 2, 3 and 4.}
\label{fig:experimentalResults}
\end{figure*}

\begin{table*}[!ht]
  \centering
  \caption{REVIEW database comparison of successful measurement percentages (\%), mean vessel widths (Mean) and standard deviations of the measurement errors ($\sigma_\chi$).}
    \begin{tabular}{rrrrrrrrrrrrr}
    \toprule
          & KPIS  &       &       & CLRIS &       &       & VDIS  &       &       & HRIS  &       &\\
          & \%    & Mean  & $\sigma_{\chi}$   & \%    & Mean  & $\sigma_{\chi}$   & \%    & Mean  & $\sigma_{\chi}$   & \%    & Mean  & $\sigma_{\chi}$\\
    \midrule
    Standard & 100.0 & 7.51  & 0.00  & 100.0 & 13.79 & 0.00  & 100.0 & 8.83  & 0.00  & 100.0 & 4.35 & 0.00\\
    O1    & 100.0 & 7.00  & 0.23  & 100.0 & 13.19 & 0.57  & 100.0 & 8.50  & 0.54  & 100.0 & 4.12 & 0.27\\
    O2    & 100.0 & 7.60  & 0.21  & 100.0 & 13.69 & 0.70  & 100.0 & 8.91  & 0.62  & 100.0 & 4.35 & 0.28\\
    O3    & 100.0 & 7.97  & 0.23  & 100.0 & 14.52 & 0.57  & 100.0 & 9.15  & 0.67  & 100.0 & 4.58 & 0.30\\
    \midrule
    Gregson & 100.0 & 7.29  & 0.60  & 100.0 & 12.80 & 2.84  & 100.0 & 10.07 & 1.49  & 100.0 & 7.64  & 1.48\\
    HHFW  & 96.3  & 6.47  & 0.39  & 0.0   &       &       & 78.4  & 7.94  & 0.88  & 88.3  & 4.97  & 0.93\\
    1DG   & 100.0 & 4.95  & 0.40  & 98.6  & 6.30  & 4.14  & 99.9  & 5.78  & 2.11  & 99.6  & 3.81  & 0.90\\
    2DG   & 100.0 & 5.87  & 0.34  & 26.7  & 7.00  & 6.02  & 77.2  & 6.59  & 1.33  & 98.9  & 4.18  & 0.70\\
    \midrule
    ESP   & 100.0 & 6.56  & 0.33  & 93.0  & 15.70 & 1.47  & 99.6  & 8.80  & \textbf{0.77}  & 99.7  & 4.63  & 0.42\\
    Graph & 99.4  & 6.38  & 0.67  & 94.1  & 14.05 & 1.78  & 96.0  & 8.35  & 1.43  & 100.0 & 4.56  & 0.57\\
    ARIA & 100.0 & 6.30  & \textbf{0.29}  & 100.0 & 14.27 & 0.95  & 99.0  & 8.07  & 0.95  & 99.5  & 4.66  & \textbf{0.32}\\
    ETOS & 100.0 & 6.14 & 0.36 & 100.0 & 14.03 & \textbf{0.53} & 99.87 & 8.36 & 0.80 &
    99.83 & 4.95 & 0.45\\
    \bottomrule
    \end{tabular}%
  \caption*{
\textbf{Measurement method abbreviations}: (Standard) - Ground truth measurements based on three human observer measurements, (O1-O3) - Human Observers 1-3, (Gregson) - Gregson rectangle fitting \cite{Gregson1995}, (HHFW) - Half Height Full Width \cite{Brinchmann-Hansen1986}, (1DG) - 1D Gaussian model fitting \cite{Zhou1994}, (2DG) - 2D Gaussian model fitting \cite{Lowell2004}, (ESP) - Extraction of Segment Profiles \cite{Al-Diri2009}, (Graph) - Graph based method \cite{Xu2011}, (ARIA) - Autmated Retinal Image Analyzer \cite{Bankhead2012} and (ETOS) - Edge Tracking on Orientation Scores. \textbf{Dataset abbreviations}: (KPIS) - the Kick Point Image Set, (CLRIS) - Central Light Reflex Image Set, (VDIS) - Vascular Disease Image Set and (HRIS) - the downsampled High Resolution Image Set (HRIS). See Section~\ref{sec:ValidationOfWidthMeasurements} for more details.}
  \label{tab:widthError}%
\end{table*}
\subsection{Validation}
\label{sec:validation}

The algorithms were tested on the green channel of color fundus images. For each image the luminosity is normalized by disposing low frequency luminosity drifts. The low frequency drifts are detected by large scale Gaussian blurring of the image (typically $\sigma = 32$), and are subtracted from the original image.

\subsubsection{Algorithm behavior at complex vessel junctions}
\label{sec:AlgorithmBehaviorAtComplexVesselJunctionPoints}
A qualitative validation is done using a challenging set of 4 sub-images (see Fig.~\ref{fig:experimentalResults}), which were taken from the high-resolution fundus images of the HRFI database \cite{Budai2011}. This set of sub-images contains crossings, overlapping bifurcations with crossings, small vessels crossing large vessels, small vessels, curved vessels, parallel vessels, etc. In each sub-image we manually placed seed points at the start of each blood vessel and at each bifurcation. In total 27 seed points were marked. Each seed point contains initial vessel center position, left edge position, right edge position and orientation, denoted by $\mathbf{c}_0$, $\mathbf{u}_0$, $\mathbf{v}_0$ and $\theta_0$. The initial scale for the CTOS algorithm is detected as the scale that provides the largest scale response at $\mathbf{c}_0$ and $\theta_0$ (see Eq.~(\ref{eq:scaleDetection})).

The tracking experiments are conducted using the following set of tracking parameters: The step size is set to $\lambda = 2$ pixels; The width of the scan line is set to $2 \eta_{max} = 40$ pixels; The number of orientations used to construct the orientation scores is set to $N_o = 36$ and the standard deviation of the Gaussian distributions used in the edge probability envelope is set to $\sigma = 3$.

The ETOS algorithm was tested on both invertible orientation scores, which were constructed by cake wavelets, and non-invertible orientation scores, which were constructed by Gabor wavelets. The scale of the Gabor wavelets was chosen in such a way that the relevant vessel features were presented as well as possible in the orientation scores (e.g. a scale too large would only represent the very large blood vessels correctly, and a scale too low only the small vessels). We found that $a = \cfrac{3}{2\pi} 10$ gave best results. For the CTOS algorithm we used a set of orientation scores constructed by Gabor wavelets at scales $a = \cfrac{3\tau}{2\pi}$ with $\tau = $5,10,15,20,25 and 30.

Results of the tracking experiments are shown in Fig.~\ref{fig:experimentalResults}. From this figure we see that, at complex situations, the ETOS method (column 2 and 3) outperforms the CTOS method (column 4). Best results are obtained when ETOS is used with invertible orientation scores (column 2). The ETOS algorithm acting on the invertible orientation scores generated by the cake kernels only fails to correctly track blood vessel nr 5 from image 3. The algorithm gives excellent results for all other vessels and manages to track the blood vessels through all complex situations, even when the contrast of the vessel edges is very low.

The performance of the ETOS algorithm is slightly decreased when applied to non-invertible orientation scores based on Gabor filters. It now fails to track 3 vessels correctly. The scale selective property of the Gabor wavelets, resulting in non-invertible orientation scores, causes the ETOS algorithm to perform less accurately compared to the application to invertible orientation scores.

The CTOS algorithm, which relies on a multi-scale orientation score approach, has the lowest performance. It fails to correctly track 5 blood vessels. In some cases, incorrect scale selection causes small parallel blood vessel to be detected as one large blood vessel (vessel 2 with 4, and 3 with 6 in the first image). Other tracks failed as a result of incorrect orientation detection.

In conclusion we can state that ETOS outperforms CTOS and that it gives best results when applied on \emph{invertible} orientation scores.

\subsubsection{Validation of width measurements}
\label{sec:ValidationOfWidthMeasurements}
In the previous section we showed that the ETOS algorithm is very well capable of tracking blood vessels through complex situations. In this section we quantitatively validate the reliability of the measured vessel widths that are provided by the ETOS algorithm. This is done by comparing the measured widths to ground truth width measurements provided by the REVIEW database \cite{Al-Diri2009}. The REVIEW database consists of 16 color fundus images, which can be divided into 4 subsets: 1) Kick point image set (KPIS), 2) Central light reflex image set (CLRIS), 3) Vascular disease image set (VDIS) and 4) the high resolution image set (HRIS). Each image set represents images of different quality and resolution, and all are provided with manual width measurements that are performed by three individual observers. A ground truth of the vessel widths is constructed by averaging the measurements of the observers. The HRIS set contains high resolution images (3584x2438) and were down-sampled by a factor of four before submission to the ETOS algorithm. For more information on the dataset we would like to refer to \cite{Al-Diri2009}.

When testing our algorithm, we tracked each segment by initializing the algorithm using the first pair of manually marked edge points. The same parameters that are described in Section~\ref{sec:AlgorithmBehaviorAtComplexVesselJunctionPoints} were used. Fig.~\ref{fig:widthResults} shows a selection of the tracking results in comparison to ground truth vessel edge labeling.

In total 5066 vessel width measurements are available. The error between automated measurements and the ground truth measurements is defined as
\begin{equation}
\chi_i = w_i - w_i^{GT} = \norm{\mathbf{u}_i-\mathbf{v}_i} - w_i^{GT}
\label{eq:error}
\end{equation}
where $w_i$ is the estimated width as measured by the ETOS algorithm (recall Eq.~(\ref{eq:width})), and $w_i^{GT}$ is the ground truth width of the $i$th profile. To be able to compare our method with others we follow the same validation procedure as described in \cite{Al-Diri2009,Bankhead2012,Xu2011}, where the main focus is on the standard deviation of the errors. This is motivated by the idea that different implicit definitions of vessel widths may lead to consistent errors. If this bias however is consistent enough, the error could easily be accounted for by subtraction of a bias constant. A low standard deviation of the errors indicates that the error is consistent.

\begin{figure}[!h]
\begin{center}
\includegraphics[width=\textwidthtwo]{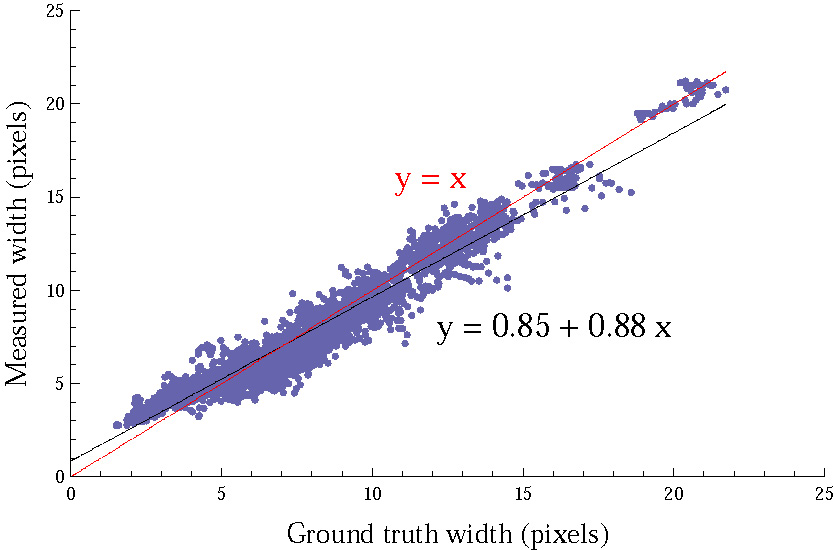}
\end{center}
\caption{A scatter plot, plotting 5059 ground truth widths against the widths measured by our ETOS algorithm. The linear regression model $y = 0.85 + 0.88 x$ indicates an offset of less then a pixel, suggesting that ETOS slightly over-estimates the vessel widths. The slope of 0.88 indicates a strong positive relation between the ground truth and measured widths.}
\label{fig:scatter}
\end{figure}

\begin{figure}[!h]
\begin{center}$
\begin{array}{ccc}
\includegraphics[width=0.33\textwidthtwo]{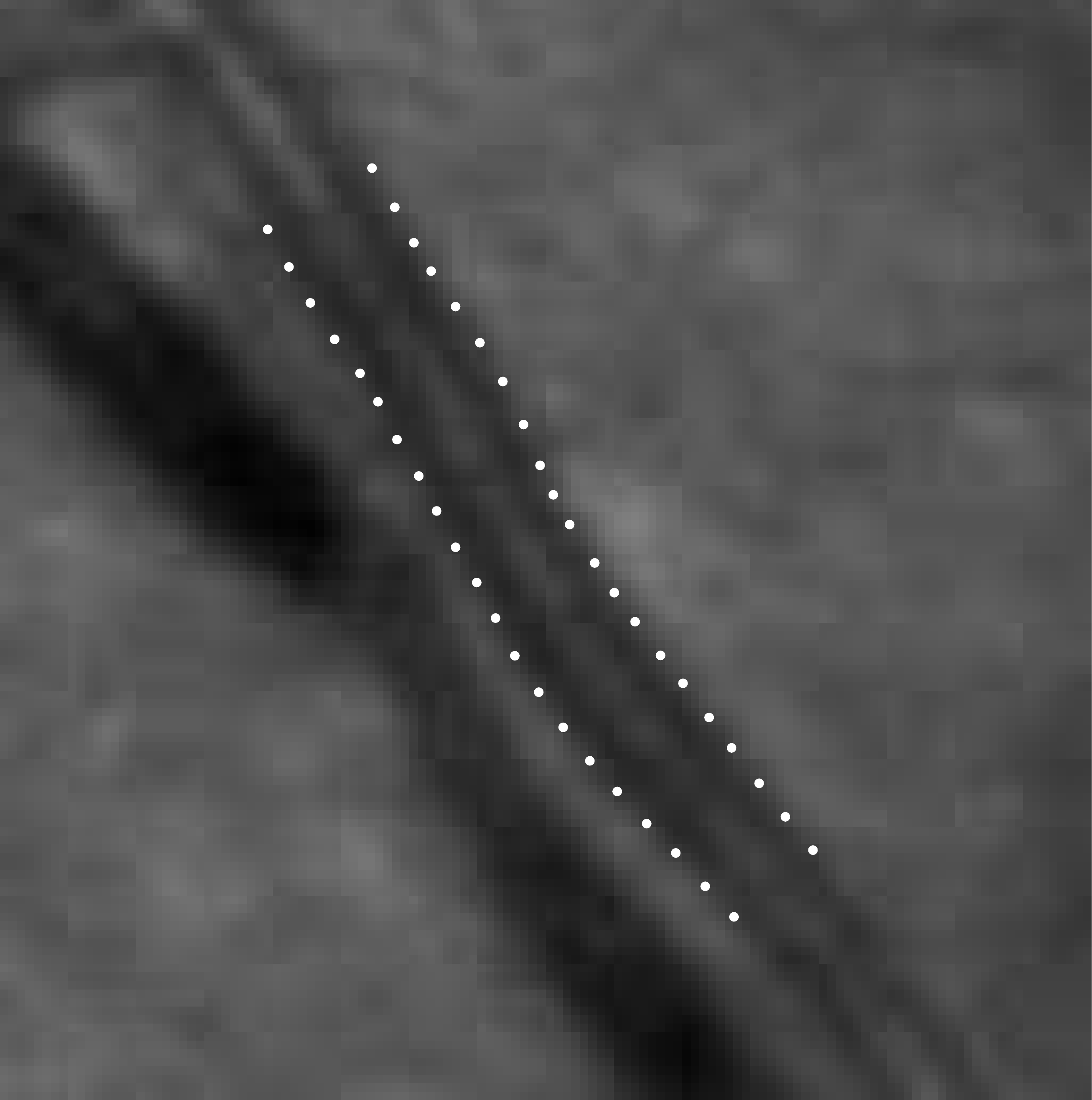}&
\includegraphics[width=0.33\textwidthtwo]{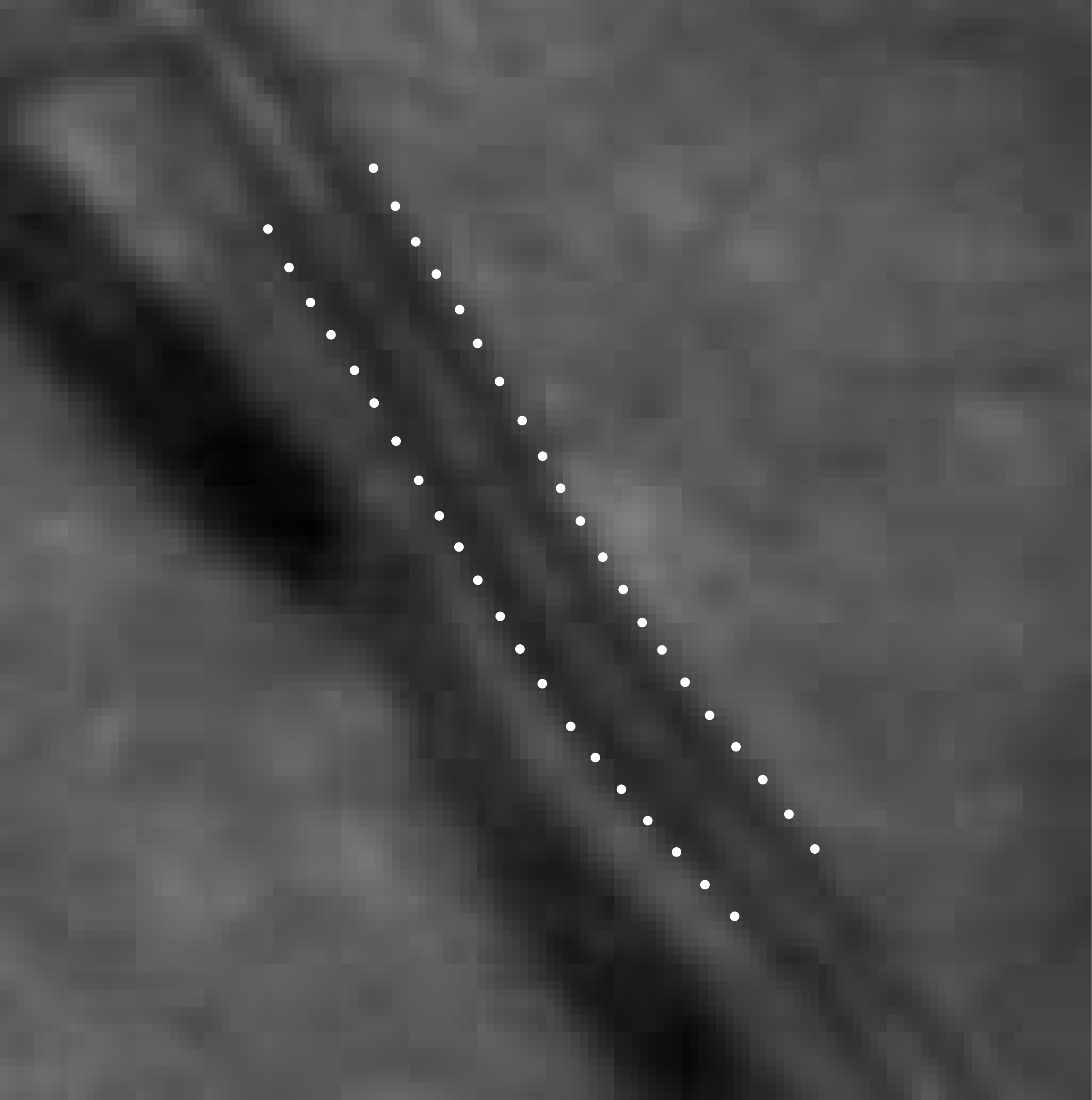}&
\includegraphics[width=0.33\textwidthtwo]{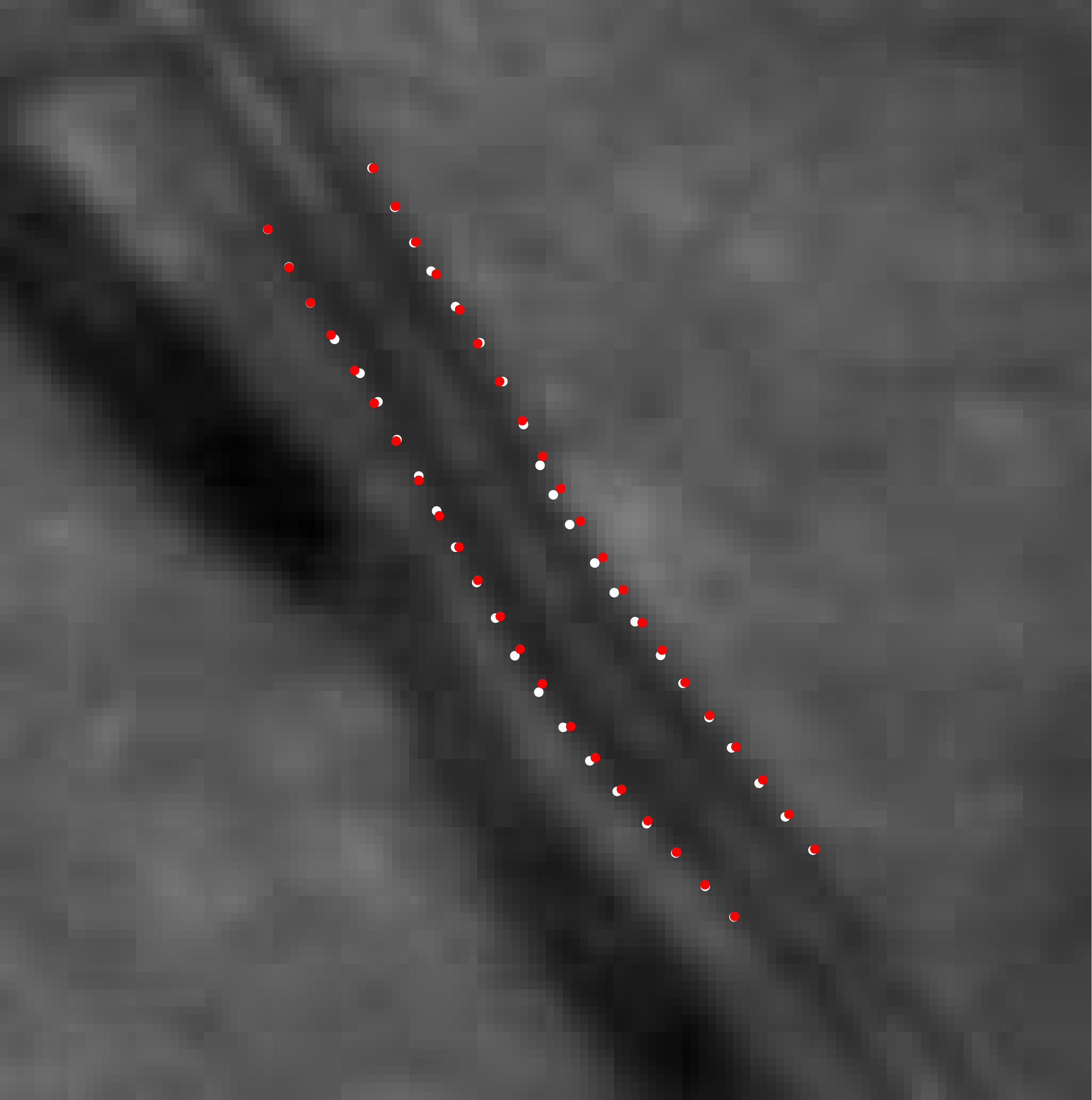}\\
\includegraphics[width=0.33\textwidthtwo]{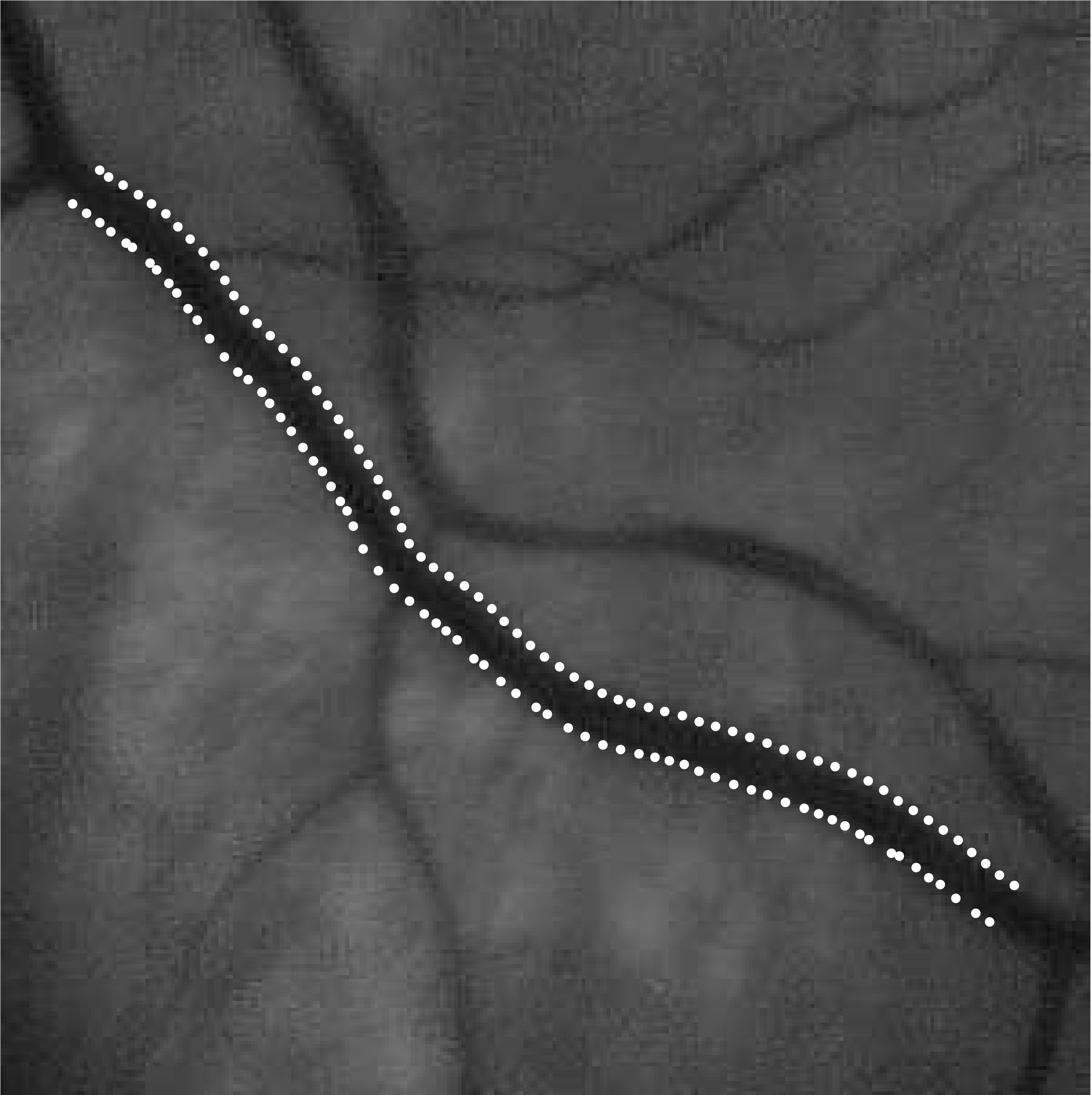}&
\includegraphics[width=0.33\textwidthtwo]{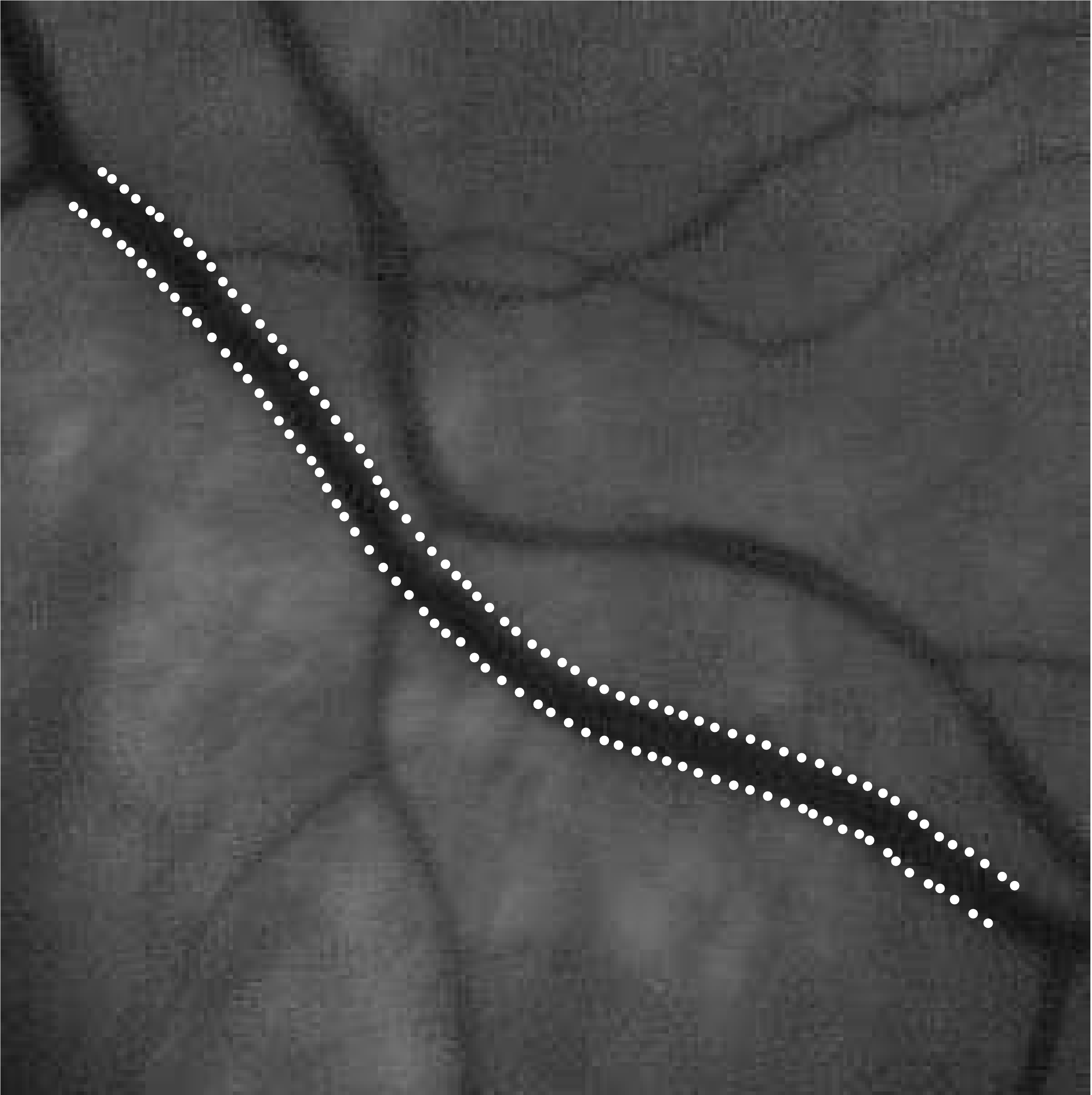}&
\includegraphics[width=0.33\textwidthtwo]{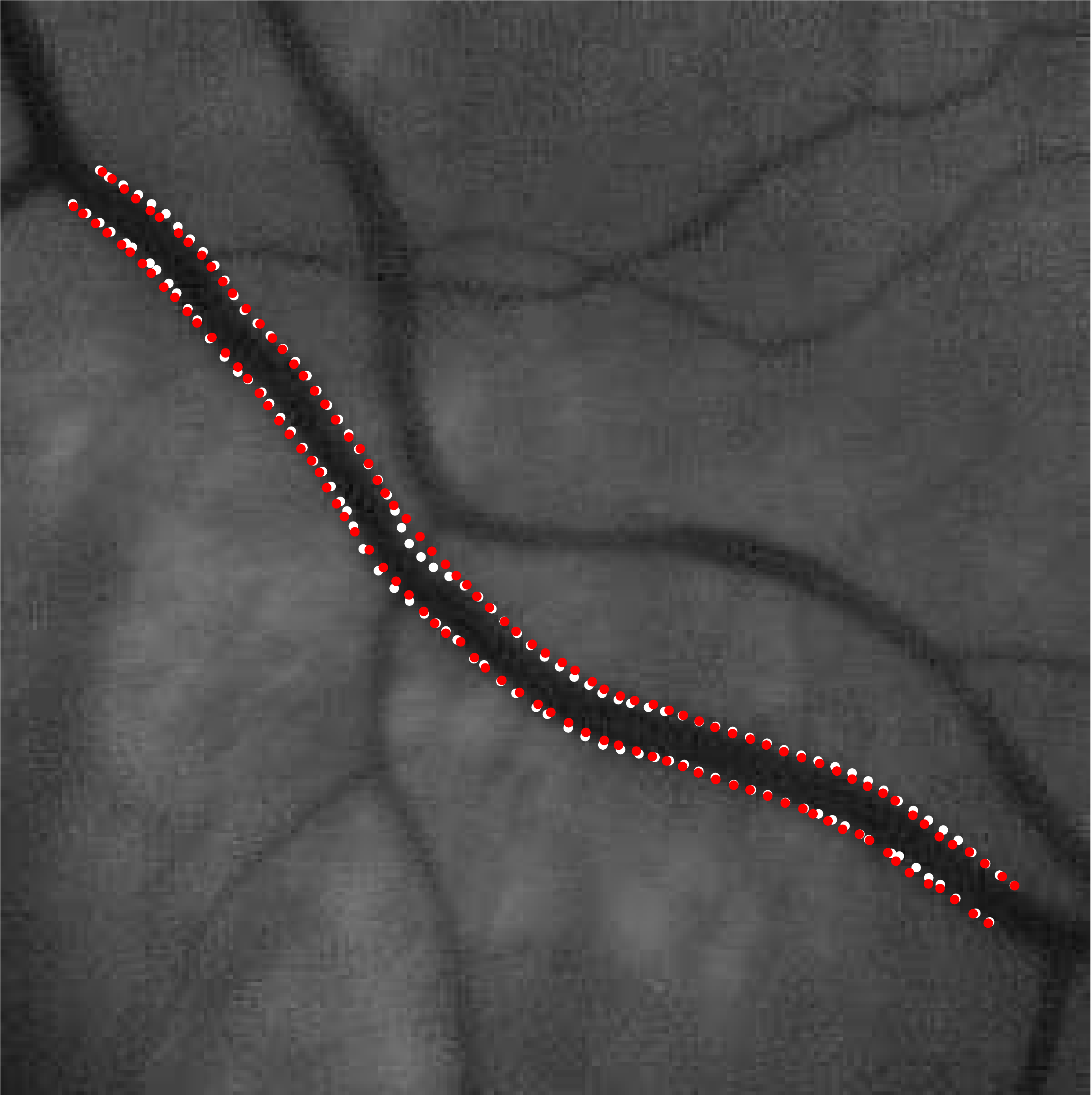}\\
\includegraphics[width=0.33\textwidthtwo]{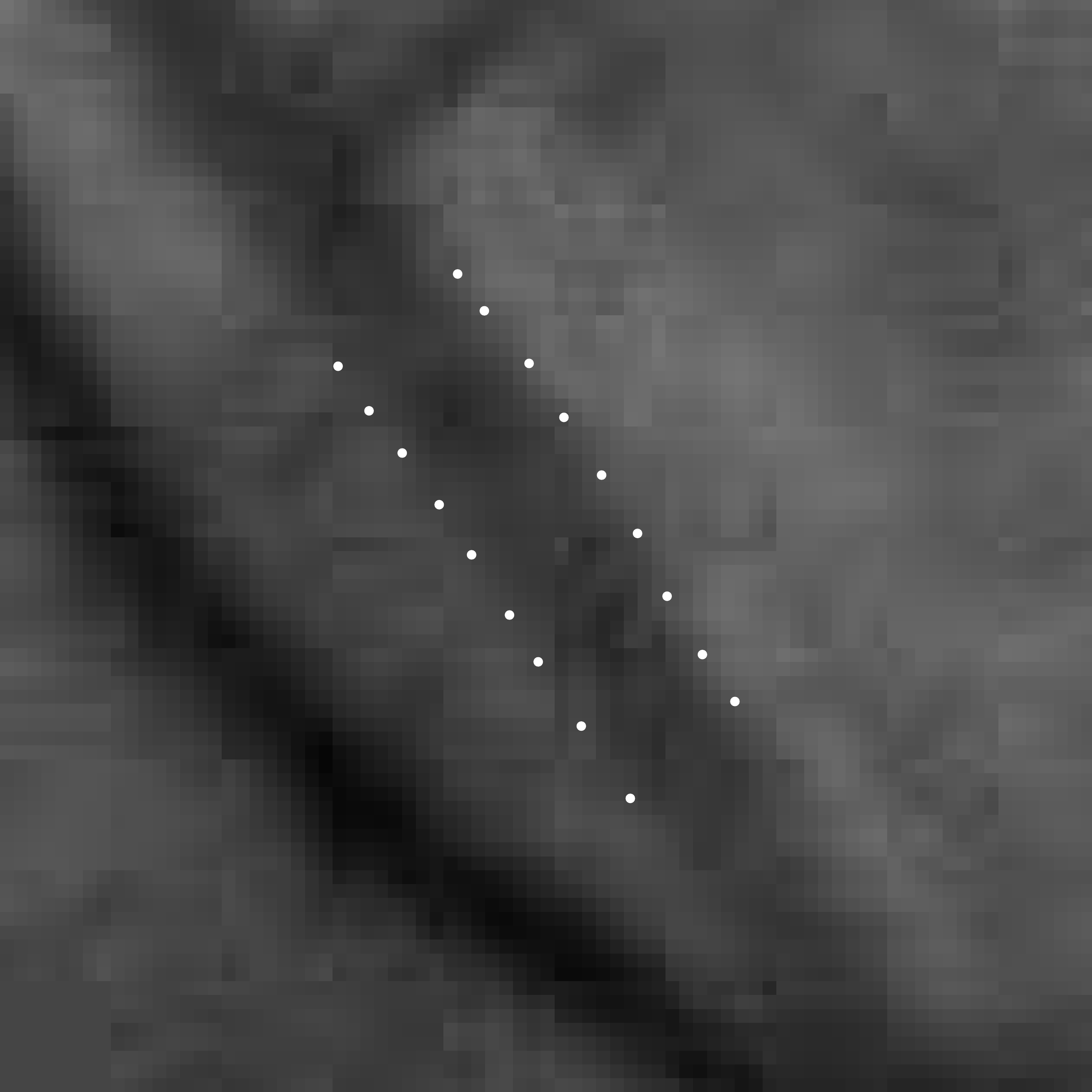}&
\includegraphics[width=0.33\textwidthtwo]{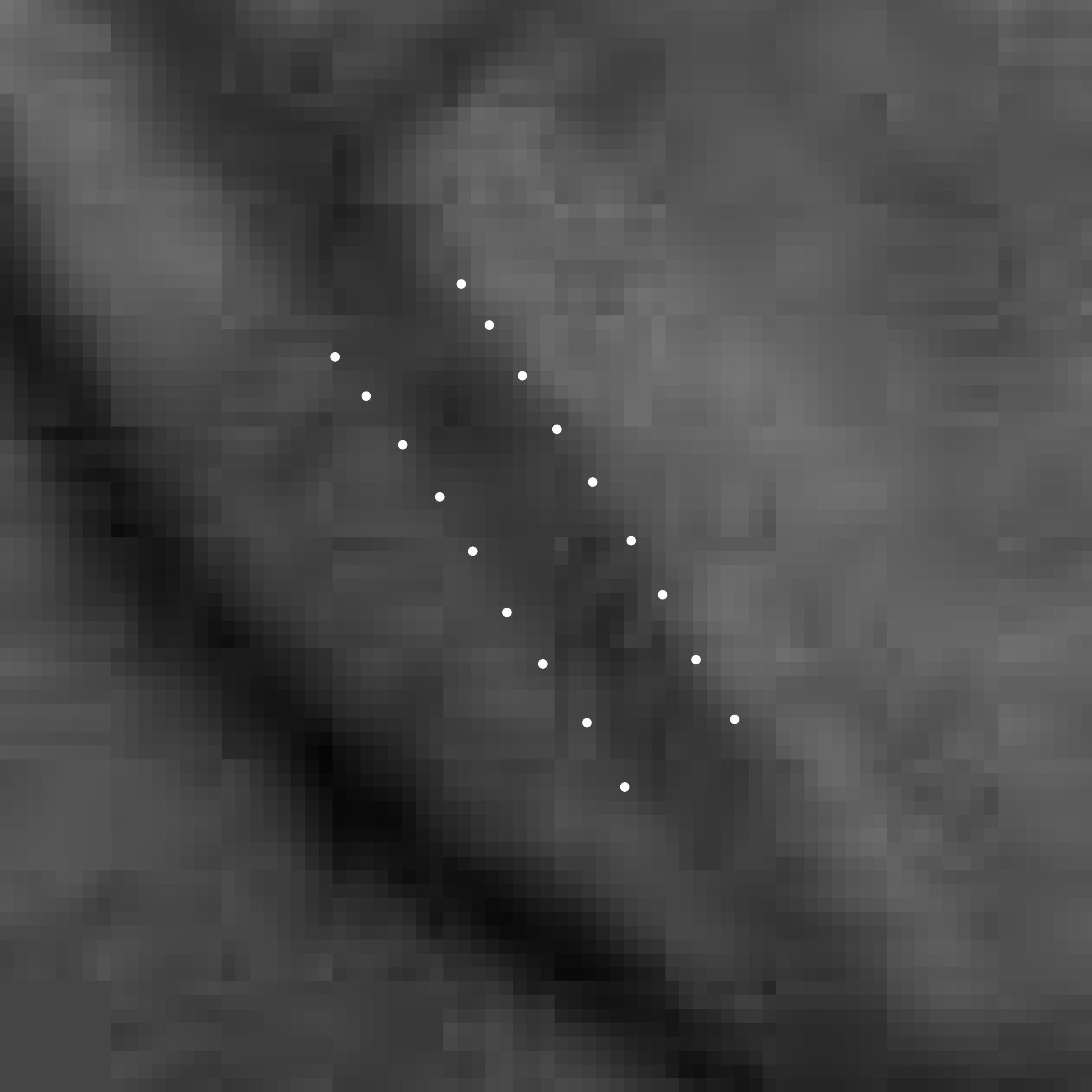}&
\includegraphics[width=0.33\textwidthtwo]{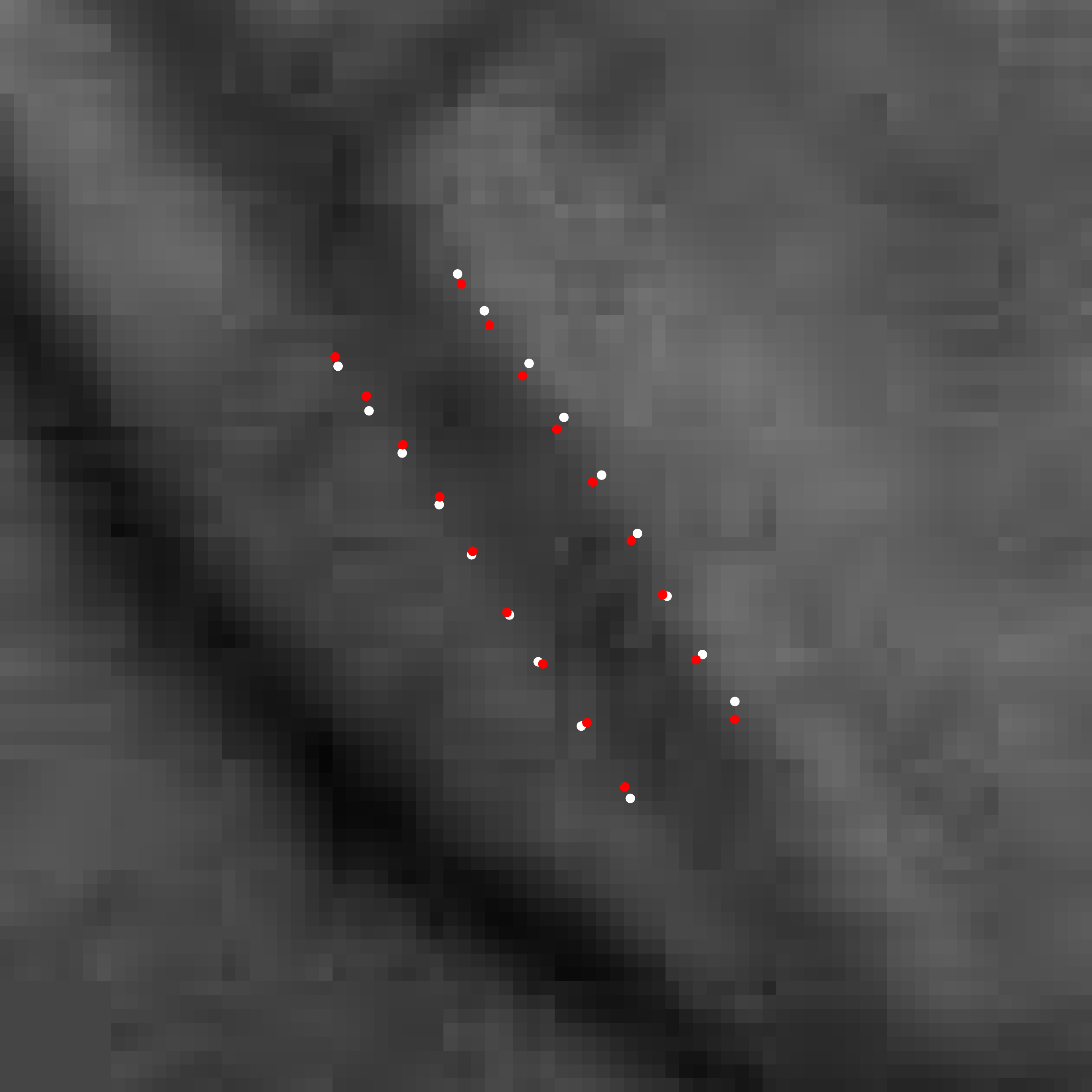}\\
\includegraphics[width=0.33\textwidthtwo]{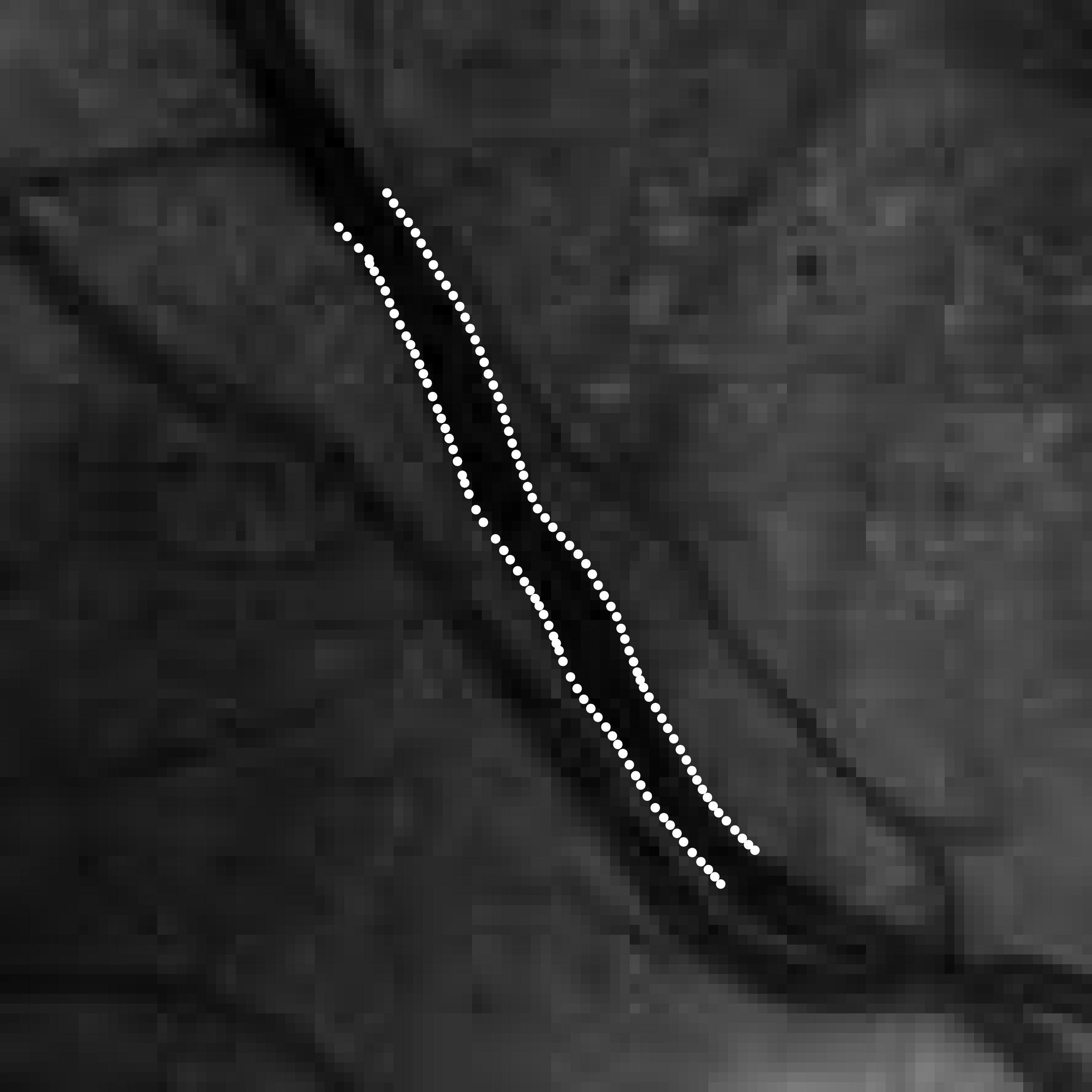}&
\includegraphics[width=0.33\textwidthtwo]{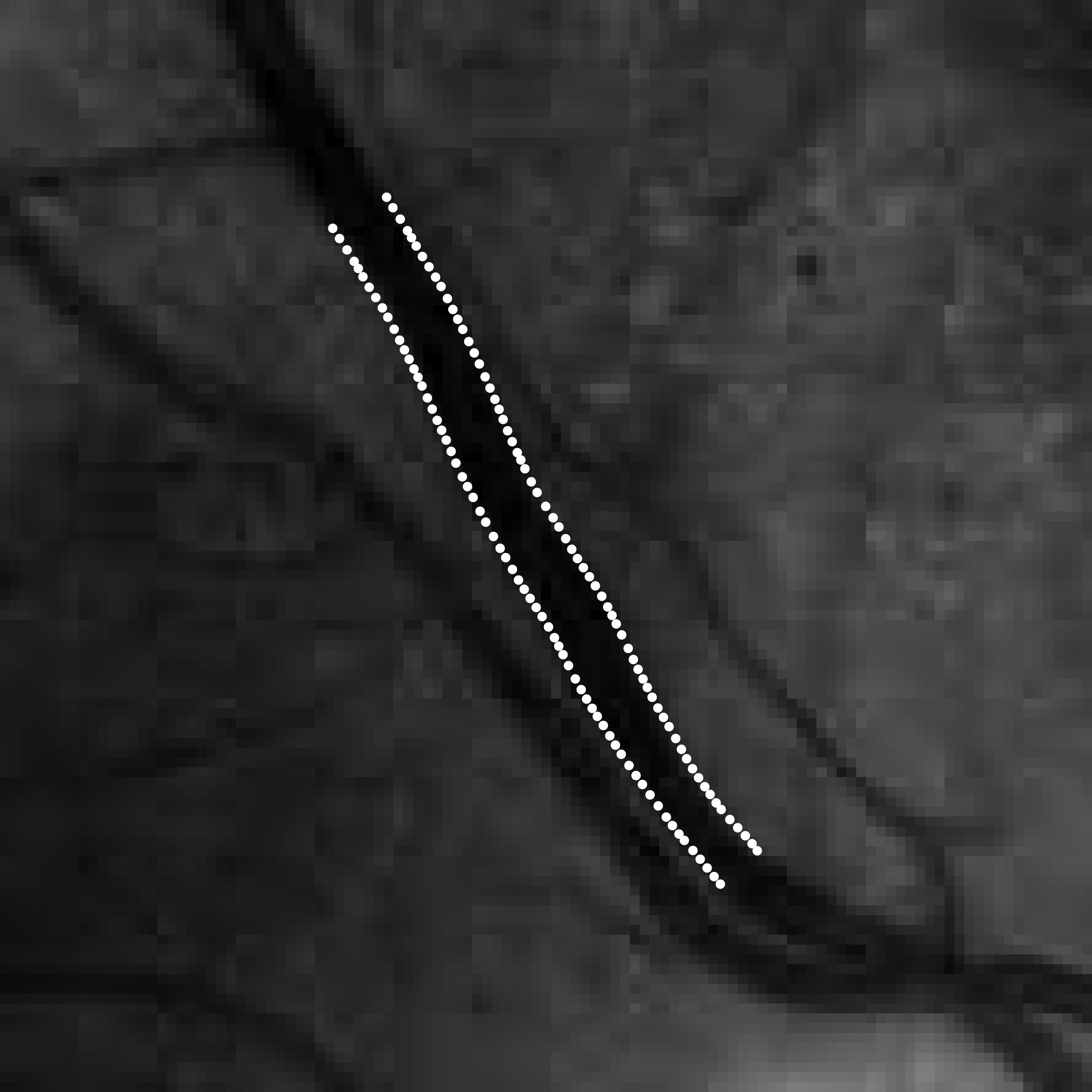}&
\includegraphics[width=0.33\textwidthtwo]{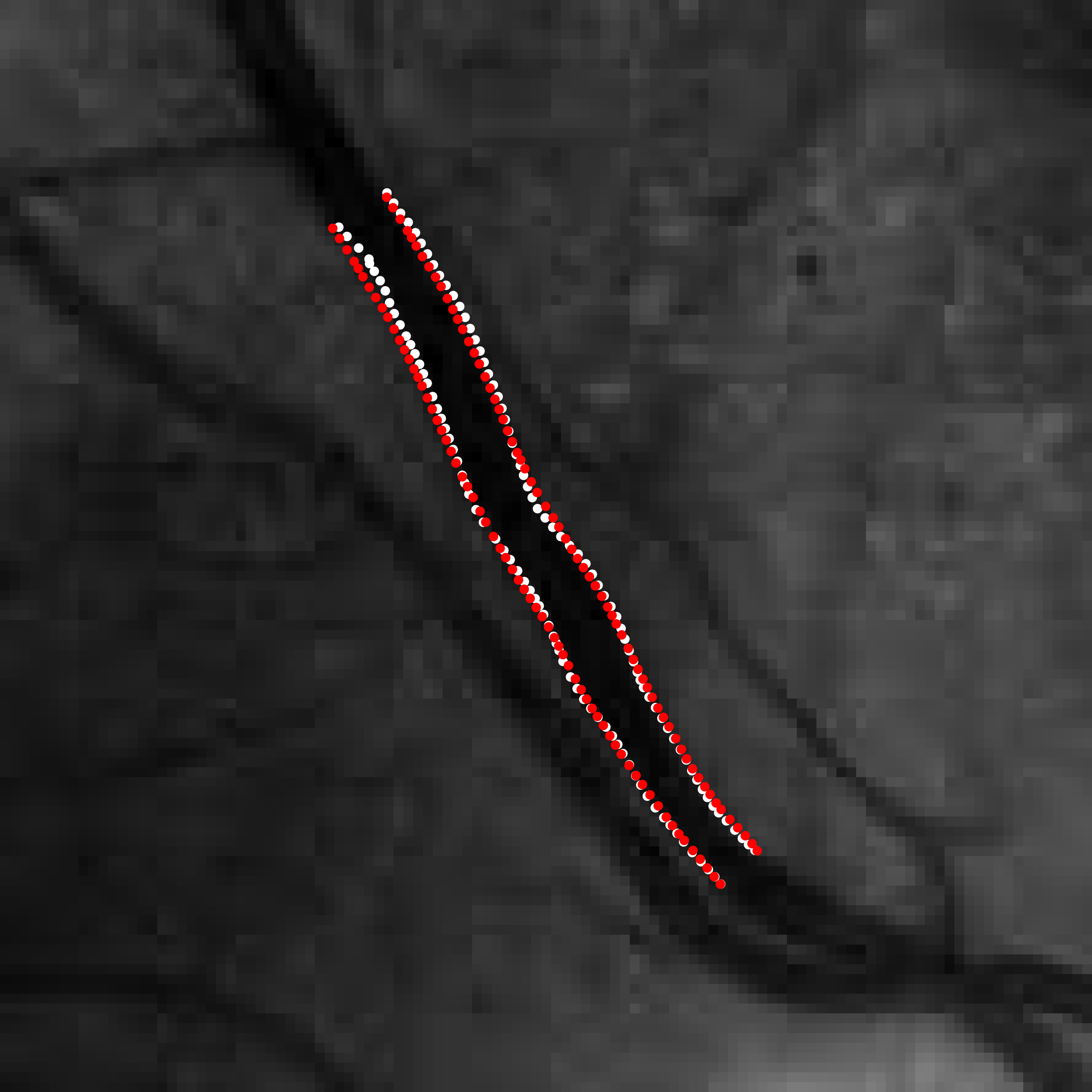}
\end{array}$
\end{center}
\caption{Several tracked vessel segments by ETOS in comparison with manual width measurements. Left column shows the ground truth vessel edge labeling as provided by the REVIEW database. The middle column shows results obtained by the ETOS algorithm and the right column shows both the ground truth (in white) and our results (in red).}
\label{fig:widthResults}
\end{figure}

Table~\ref{tab:widthError} shows the validation results of our ETOS algorithm, in comparison with methods by other authors that published their results using the same database \cite{Al-Diri2009,Bankhead2012,Xu2011}. The first four rows of the table show the results of the manual annotations (observer 1, 2 and 3) and the golden standard. The next four rows show results of four classic approaches to vessel width measurements:
\begin{itemize}
  \item Gregson: a rectangle is fitted to a vessel intensity profile, and the width is set such that the area under the rectangle and profile \cite{Gregson1995} are equal.
  \item Half Height Full Width (HHFW): the standard half-height method, which uses thresholds set half-way between the maximum and minimum intensities at either side of an estimated center point \cite{Brinchmann-Hansen1986}.
  \item 1D Gaussian (1DG): a 1D Gaussian model is fit to the vessel intensity profile \cite{Zhou1994}.
  \item 2D Gaussian: a 2D Gaussian model is fit to the vessel intensity profile \cite{Lowell2004}.
\end{itemize}
The next three rows give results of the most recent, state of the art methods that published their results:
\begin{itemize}
  \item The Extraction of Segment Profiles (ESP) is an active contour algorithm by Al-Diri et al. \cite{Al-Diri2009}.
  \item The Graph method is a graph based edge segmentation technique developed by Xu et al. \cite{Xu2011}.
  \item The Automated Retinal Image Analyzer (ARIA) is an algorithm developed by Bankhead et al. \cite{Bankhead2012}, they used a wavelet approach to vessel segmentation after which the edge locations are refined.
\end{itemize}
The last row shows the results we achieved using our ETOS algorithm.

The column labeled with \% shows the success rate, it indicates how many width measurements could successfully be validated (for more detail see \cite{Al-Diri2009}). The success percentage is smaller then 100\% whenever measurements failed to converge, e.g. when the distance between the ground truth and measured edge pair was too large. The column labeled with \emph{Mean} indicates the mean vessel width of all the measured vessel profiles. The column labeled with $\sigma_\chi$ indicates the standard deviation of the error (Eq.~(\ref{eq:error})), a lower $\sigma_\chi$ is favorable since it indicates that the error is consistent.

From Table~\ref{tab:widthError} it can be observed that ESP, Graph, ARIA and our ETOS algorithm all outperform the classic width measurement techniques. Also compared to the state of the art methods our algorithm scores very well. The ETOS algorithm performs remarkably well on the CLRIS dataset, which contains a large number of vessels with the central light reflex. For these images, the standard deviation of the errors is even lower than those of the observers. For other datasets, our method`s performance is comparable to the state of the art.

Fig. \ref{fig:scatter} shows a scatter plot of the ground truth widths against the widths measured by our ETOS algorithm, together with a linear regression model that was fit through these points. The points are very much centered around the line $y = x$, indicating a strong positive correlation. This is confirmed by the slope of the linear regression model $y = 0.85 + 0.88 x$, which is near to 1. The low number of outliers in the scatter plot confirms the low standard deviation in errors, as demonstrated by Table~\ref{tab:widthError}. The offset of 0.85, together with slope 0.88, indicate that the ETOS algorithm has the tendency to slightly overestimate for vessels of size up to 7 pixels and underestimates for larger vessel sizes.

We conclude that our ETOS algorithm, which is highly capable of tracking blood vessels through all sorts of complex situations, also provides reliable width measurements.

\section{Vasculature tracking}
In this section we describe additions to the ETOS algorithm, so as to be able to construct models of the complete retinal vasculature. Our vasculature tracking algorithm consists of:
 \begin{enumerate}
   \item Optic disk detection.
   \item Seed point detection in the optic disk region.
   \item Correct Initialization of the ETOS algorith by robust initial edge detection.
   \item Automatic termination based on a set of stopping criteria.
   \item Junction detection, classification and numbering.
   \item Junction resolving.
 \end{enumerate}
Each of these items is described in Sections~\ref{sec:opticdiskdetection} to \ref{sec:junctionResolving} and the complete algorithm is validated in Section~ \ref{sec:vasculatureResults}.

\subsection{Methods}
\label{sec:vasculatureTrackingMethods}

\subsubsection{Optic disk detection}
\label{sec:opticdiskdetection}
Since the blood vessels of interest all enter the eye through the optic disk, vasculature tracking is initiated in this region. The detection of the optic disk occurs in two phases. In the first phase a rough estimation of the optic disk position is made based on a method using variance filtering, as proposed by Sinthanayothin \cite{Sinthanayothin1999}. This method is based on the significant variance in pixel intensities that occur in the optic disk region.

Next, the estimated optic disk position is refined by edge focusing \cite{Bergholm1987} on the optic disk boundary, followed by a weighted Hough transform \cite{Hough1962} for circles. Here we assume that the optic disk is approximately circular. To avoid disturbance of blood vessels during the detection of the optic disk boundary, the vessels are first removed by applying a closing operator on the red channel of the image (Fig.~\ref{fig:opticDiskDetection}b). Edge focussing is performed on a star-shaped set of profiles (Fig.~\ref{fig:opticDiskDetection}c) and the detected edge positions are used as input for the weighted Hough transform for circles. The weight of each edge position is determined by the scale up to which the edge can be traced in Gaussian scale space (before it reaches a so-called toppoint \cite{Florack2000,Johansen1994}).

The optic disk radius $R_{OD}$, found by the Hough transform for circles, describes the size of the optic disk. For adult human eyes, the average radius of the optic disk is known to be $R_{OD}^{ref} = 0.92mm$ \cite{Tasman2009}. Vessels in the optic disk region typically have a caliber of $\langle w \rangle_{av}^{ref} = 0.15 mm$. While the resolution of retinal images may vary from camera to camera, the physical dimensions of the human eye are rather constant between individuals. In order to make our algorithm generally applicable to retinal images of different resolution, we will normalize distances used in our routines with
\begin{equation}
\label{eq:widthav}
\langle w \rangle_{av} = \cfrac{\langle w \rangle_{av}^{ref}}{R_{OD}^{ref}} R_{OD}\approx \cfrac{1}{6}R_{OD},
\end{equation}
where, based on real physical values, $\langle w \rangle_{av}$ describes typical retinal vessel calibers in pixels.

\begin{figure}[!ht]
        \centering
        \begin{subfigure}[t]{0.345\textwidthtwo}
                \centering
                \includegraphics[width=\textwidth]{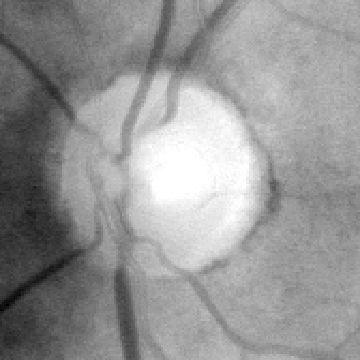}
                \caption{}
                \label{fig:edgesInit:a}
        \end{subfigure}
        \begin{subfigure}[t]{0.345\textwidthtwo}
                \centering
                \includegraphics[width=\textwidth]{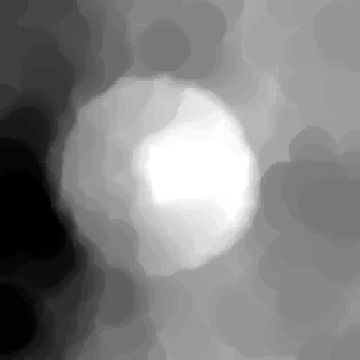}
                \caption{}
                \label{fig:edgesInit:b}
        \end{subfigure}
        \begin{subfigure}[t]{0.345\textwidthtwo}
                \centering
                \includegraphics[width=\textwidth]{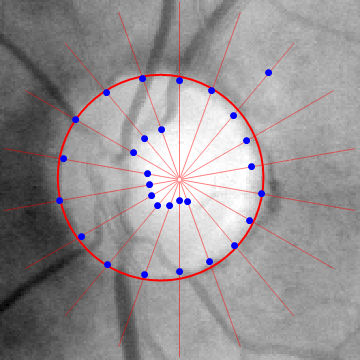}
                \caption{}
                \label{fig:edgesInit:c}
        \end{subfigure}
\caption{Automated optic disk segmentation. In image (a) the red channel of a color fundus image around the estimated optic disk position is shown. The vessels in this sub-image are filtered out by a closing operator (b). To detect dominant edges, edge focussing is performed on a star shaped set of intensity profiles, dominant edge positions are shown as blue dots in image (c). A weighted Hough transform is performed using the these positions, the result is shown as a red circle.}
\label{fig:opticDiskDetection}
\end{figure}

\begin{figure}[!ht]
        \centering
        \begin{subfigure}[t]{.52\textwidthtwo}
                \centering
                \includegraphics[width=.52\textwidthtwo]{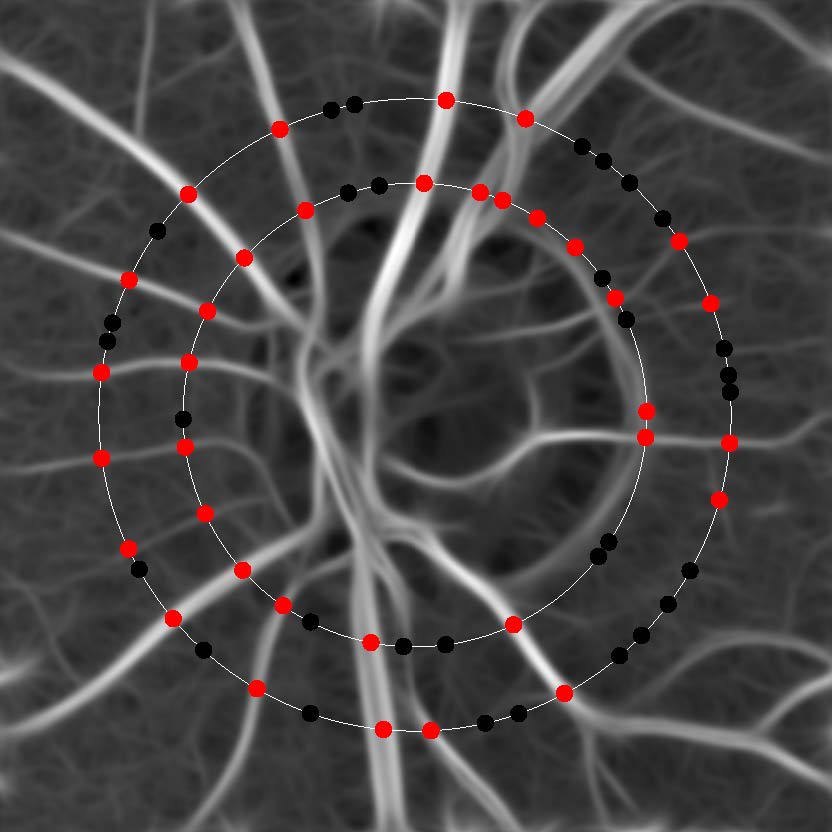}
                \caption{}
                \label{fig:SPsInit:a}
        \end{subfigure}
        \begin{subfigure}[t]{.52\textwidthtwo}
                \centering
                \includegraphics[width=.52\textwidthtwo]{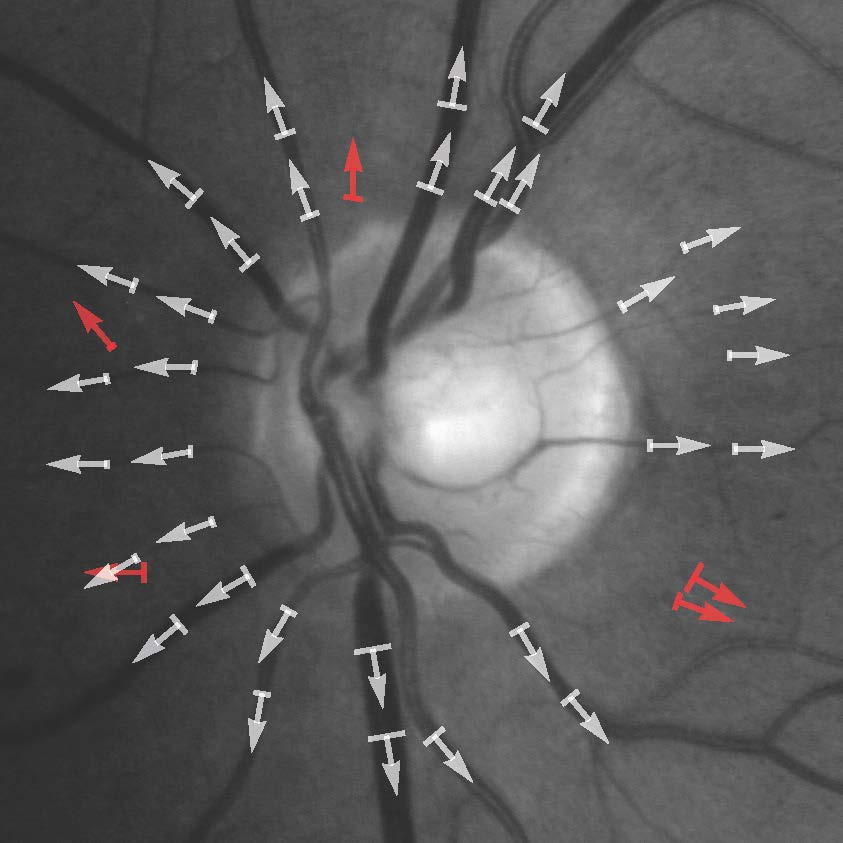}
                \caption{}
                \label{fig:SPsInit:b}
        \end{subfigure}
        \caption{Initial seed point detection from vessel likelihood maps. (a) A vessel likelihood map of the optic disk region, with two circular profiles on which initial seed points are detected. Detected seed points are shown as red dots, discarded as black dots. (b) Edge initialization and true positive seed point selection. White arrows show the detected seed points, red arrows are seed points classified as false positives.}
        \label{fig:SPsInit}
\end{figure}

\begin{figure*}[!ht]
        \centering
        \begin{subfigure}[t]{0.40\textwidth}
                \centering
                \frame{\includegraphics[height=.19\textheight]{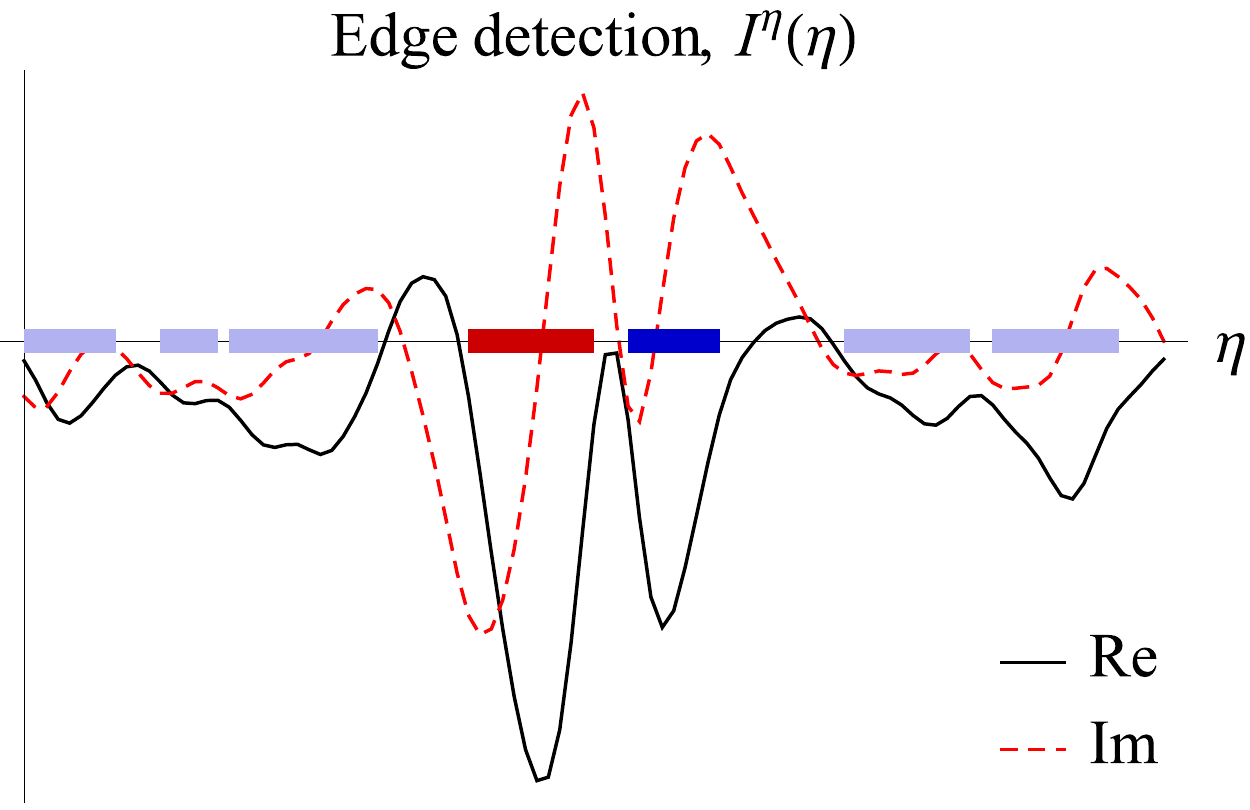}}
                \caption{}
                \label{fig:edgesInit:a}
        \end{subfigure}
        \begin{subfigure}[t]{0.29\textwidth}
                \centering
                \includegraphics[height=.19\textheight]{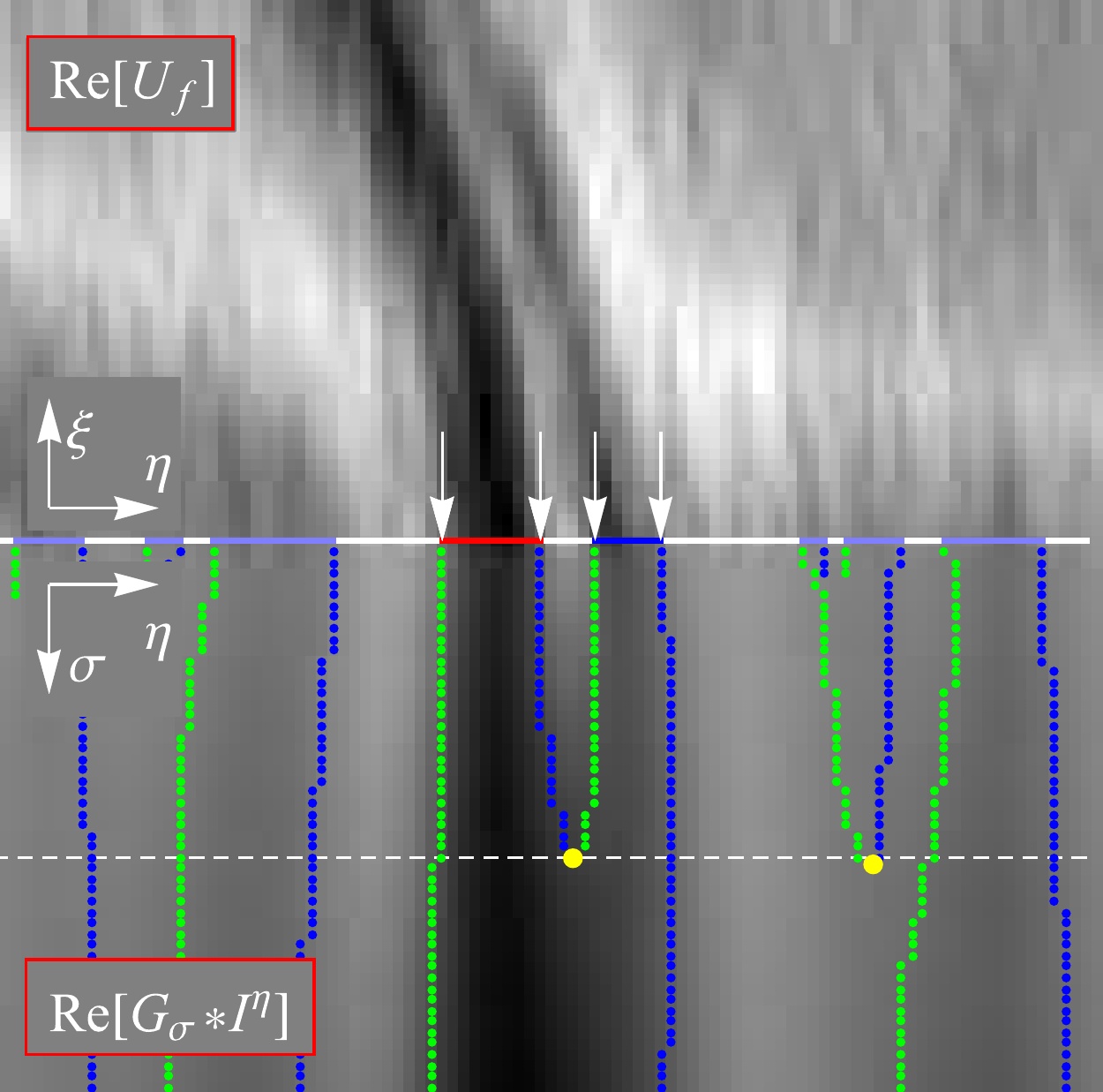}
                \caption{}
                \label{fig:edgesInit:b}
        \end{subfigure}
        \begin{subfigure}[t]{0.29\textwidth}
                \centering
                \includegraphics[height=.19\textheight]{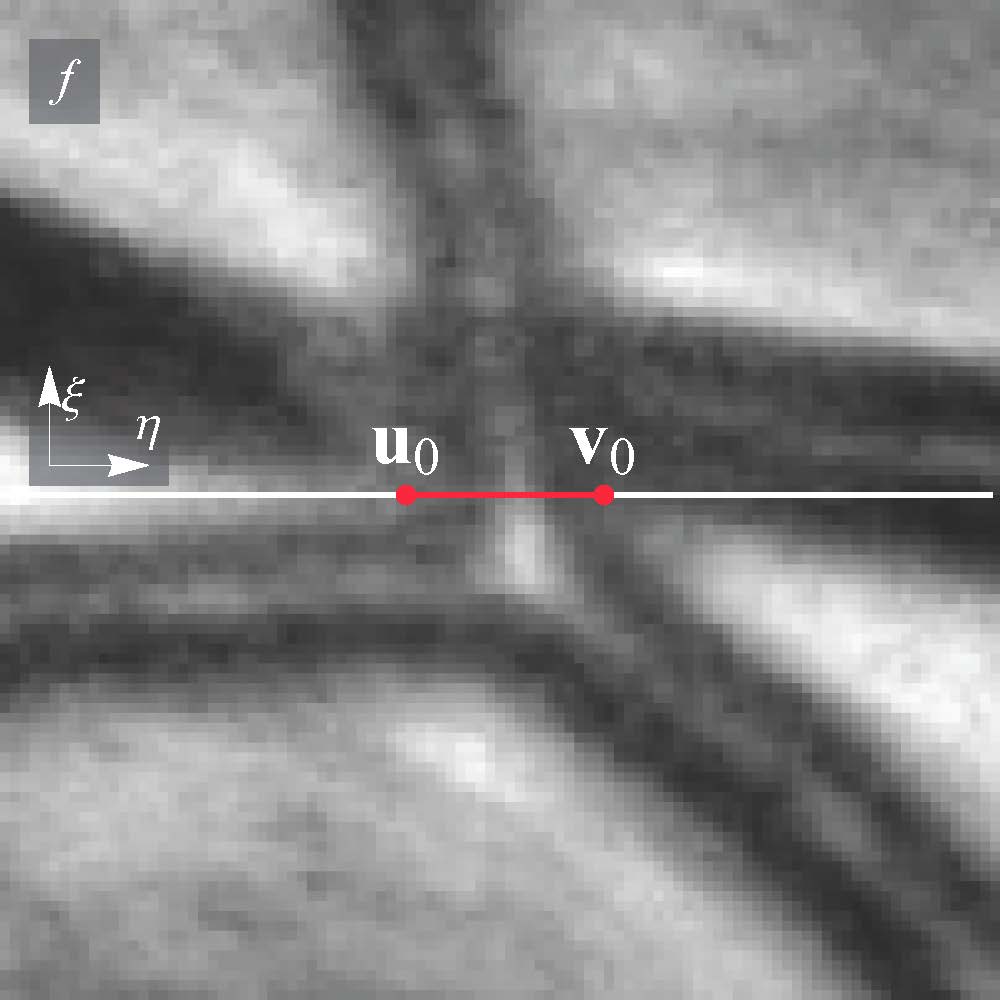}
                \caption{}
                \label{fig:edgesInit:c}
        \end{subfigure}
        \caption{Initial edge detection. (a) Based on local optima in the imaginary part of $I^\eta$, potential vessel patches are formed (shown as horizontal blocks). The main patch of interest is denoted in red and neighbouring patches of interest in non-transparent blue. (b) The edges of the patches of interest are tracked in scale (lower part of this figure) up to the scale $\sigma$ of the corresponding toppoints (yellow points) \cite{Florack2000,Johansen1994}. (c) The strongest edges at this scale are initialized to be the vessel edges $\mathbf{u}_0$ and $\mathbf{v}_0$.}
        \label{fig:edgesInit}
\end{figure*}

\subsubsection{Vessel likelihood map and seed point detection}
\label{sec:vessellikelihood}
For the detection of the initial seed points, a vessel likelihood map $V:\mathbb{R}^2 \mapsto \mathbb{R}$ of the optic disk region is constructed using invertible orientation scores. Before images are subjected to our algorithms, they have their DC-component removed by means of high-pass filtering. As a result vessel pixels have negative values and background pixel values are assumed to be zero. In effect the real part of the orientation score at a certain vessel position and orientation is negative, the orientation score is approximately zero at background area's. This is also demonstrated in Fig.~\ref{fig:ETOS3D:b}, were the vessel in the score can be seen as a dark (negative value) blob on the plane. Based on these observations, we define a vessel likelihood map $V(\mathbf{x})$ as:
\begin{equation}
V(\mathbf{x}) = \underset{\theta\in[0,2\pi]}{\operatorname{max}} \operatorname{Re}( \; -U_f(\mathbf{x},\theta) \;).
\end{equation}
Seed points are detected as local maxima on circular intensity profiles centered around the optic disk (Fig.~\ref{fig:SPsInit:a}) with radii $R = \{ R_{OD} , 1.5R_{OD} \}$, where $R_{OD}$ is the detected optic disk radius. A seed point is discarded whenever its value in the vessel likelihood map $V$ is smaller then the average value of all points on the circle. For each remaining seed point $\mathbf{c}_0$, the initial orientation $\theta_0$ is detected as the orientation that provides the highest modulus of the orientation scores: $\theta_0= \underset{\theta\in[0,2\pi]}{\operatorname{argmax}} |U_f(\mathbf{c}_0,\theta)| $. An additional filtering step, in which the seed points are classified as either true or false positive, is described in Section~\ref{sec:edgesInit}.

\subsubsection{Initial edge detection}
\label{sec:edgesInit}
The ETOS algorithm is initialized with a starting vessel center point $\mathbf{c}_0$, orientation $\theta_0$ and edges $\mathbf{u}_0$ and $\mathbf{v}_0$. Starting with an already detected initial center point $\mathbf{c}_0$ and orientation $\theta_0$, intensity profile $I^\eta$ can be obtained from the orientation scores using Eq.~(\ref{eq:ieta}). Candidate edges are detected as local optima on the imaginary part of $I^\eta$. Beside the main vessel edges that we are interested in, it is very likely that multiple other candidate edges are detected as well (as a result of noise or a central light reflex). Therefore, we use an edge focussing approach to detect the dominant edges. Each combination of neighboring left and right edges will form potential vessel patches (see Fig.~\ref{fig:edgesInit:a}). Note that a blood vessel with a central light reflex consists of two neighboring vessel patches. We start initial edge detection by detecting the main vessel patch of interest by scoring each edge pair ($\mathbf{u}$,$\mathbf{v}$), based on the initial center point estimate $\mathbf{c}_0$ and initial orientation $\theta_0$, as follows:
\begin{equation}
\label{eq:pairScoring}
s_{\mathbf{c}_0,\theta_0}(\mathbf{u},\mathbf{v}) = \nu_{\theta_0}(\mathbf{u},\mathbf{v}) \; e^{-\cfrac{1}{2}\cfrac{\|(\mathbf{u}+\mathbf{v})/2-\mathbf{c}_0\|^2}{(0.5\langle w \rangle_{av})^2}},
\end{equation}
with
\begin{equation}
\label{vesselnessvalue}
\nu_\theta(\mathbf{u},\mathbf{v}) = \cfrac{1}{\norm{\mathbf{u}-\mathbf{v}}}\int_0^1{\vert U_f(\mathbf{u} + t (\mathbf{v} - \mathbf{u}), \theta) \vert dt}.
\end{equation}
The function $\nu_\theta(\mathbf{u},\mathbf{v})$ provides the so called \emph{vessel value} and basically is the average value of the modulus of the orientation score at orientation $\theta$, calculated from the left to the right edge. This value is high for elongated/vessel structures, and low for background structures in the score. The exponential in Eq.~(\ref{eq:pairScoring}) penalizes the distance of the edge pair to the initialized center point. The edge pair with the highest score will be our main pair of interest.

Neighboring pairs are considered only if the distance between the nearest edges of the main and neighboring pair is smaller than the width of the main patch, Fig.~\ref{fig:edgesInit:a} shows typical results. All edges from the main patch and its neighbors are now traced (in scale) in the Gaussian scale space of profile $I^\eta$. The edges are traced up to the scale of the first appearing toppoint. The strongest edges at this scale are chosen as the true vessel edge points $\mathbf{u}_0$ and $\mathbf{v}_0$ (Fig.~\ref{fig:edgesInit:c}).

An initialized seed point whose vessel value is lower than a threshold $T_{\nu}$ is regarded as a false positive (Fig.~\ref{fig:SPsInit:b}). The threshold $T_{\nu}$ is defined as:
\begin{equation}
\begin{array}{l}
\label{eq:betanu}
T_{\nu} = 0.5 \langle \nu \rangle_{av},\\
\langle \nu \rangle_{av} = \cfrac{1}{N_{sp}} \sum_{i=1}^{N_{sp}} \nu_{\theta_0^i}(\mathbf{u}_0^i,\mathbf{v}_0^i),
\end{array}
\end{equation}
where using all $N_{sp}$ initialized seed points, each with initial edge points ($\mathbf{u}_0^i$, $\mathbf{v}_0^i$) and orientation ($\theta_0^i$), an average vessel value $\langle \nu \rangle_{av}$ is calculated. All true positive seed points are submitted to the ETOS algorithm to start expanding a model of the retinal vasculature.

\subsubsection{Stopping criteria}
\label{sec:stopCriteria}
For ETOS to stop tracking a single vessel, three stopping-criteria are defined:
\begin{enumerate}
  \item Tracking stops whenever a vessel being tracked leaves a prescribed region of interest. For this purpose a mask image is generated that covers the (typically circular) field of view in the fundus image.
  \item Tracking stops whenever a blood vessel is already tracked. Each tracked segment is used to construct a pixel map, in which each pixel within the vessel edge contours is set to 1 and outside to 0. Whenever a point within the vessel being tracked lies within the pixel map for $\lceil4 \langle w \rangle_{av}/\lambda\rceil$ iterations in a row, the ETOS algorithm terminates. Recall $\langle w \rangle_{av}$ defined in Eq.~(\ref{eq:widthav}) and $\lambda$ being the step size (Fig.~\ref{fig:schematicTracking}).
  \item Tracking stops whenever the vessel value $\nu_\theta$ (Eq.~(\ref{vesselnessvalue})) drops below threshold value $T_{\nu}$. Here we assumed that $T_{\nu}$, which is based on the average vessel value $\langle \nu \rangle_{av}$ calculated in the optic disk region, is a good indicator of the vessel values of the vessels elsewhere in the retina.
\end{enumerate}

\subsubsection{Junction detection, classification and numbering}
\label{sec:junctiondetection}
In order to model the complete vasculature, starting by expanding a model from a set of initial seed points, the vasculature tracking algorithm should also be able to automatically detect junctions. Junction points are points where blood vessels bifurcate/branch or where two blood vessels cross. Either way, at a junction point multiple orientations may be observed. At a point on a straight vessel the modulus of the directional orientation column ideally contains two local maxima; one in the positive direction, say $\theta_+$, and one (with a $\pi$ phase difference) in the negative direction, say $\theta_-$. A candidate vessel junction point at the left edge is detected whenever there exists another local maximum in between $\theta_+$ and $\theta_-$. Similarly, a candidate junction point is detected at the right vessel edge whenever another local maximum is present between $\theta_-$ and $\theta_+$.

During vessel tracking, the orientation columns at the left and right edge are scanned for the presence of a junction point, and the location and orientation are stored (top image Fig.~\ref{fig:jd}). The detected junction points are clustered on position by grouping all points whose distance to one another is smaller than $\langle w \rangle_{av}$. Within each cluster, the candidate junction points are clustered on orientation to prevent merging two proximate junction points. Clustering on orientation is done according to the number of local maxima in the histogram of orientations. Finally, clusters are merged to a single junction point by averaging the positions and by taking the most common orientation within the cluster. The found center points and orientations are then subjected to the edge initialization method described in Section~\ref{sec:edgesInit}, and are discarded whenever their vessel value is lower than $T_{\nu}$.

To classify between junctions and bifurcations we check for alignment of local orientations in SE(2) via the sub-Riemanian metric \cite{Citti,Boscain2}, see Eq.~(\ref{geod}) in Appendix~\ref{app:optimalpaths}.

Each detected seed point is assigned two ID numbers, a number that is unique for each vessel segment and the ID number of the vessel from which it originates. This way the relation between each vessel segment remains known, and the vessel segments can be organized in a hierarchical fashion. Vasculature tracking terminates whenever all detected bifurcations and crossings are evaluated.

\begin{figure}[!t]
\begin{center}$
\begin{array}{c}
\includegraphics[width=0.52\textwidthtwo]{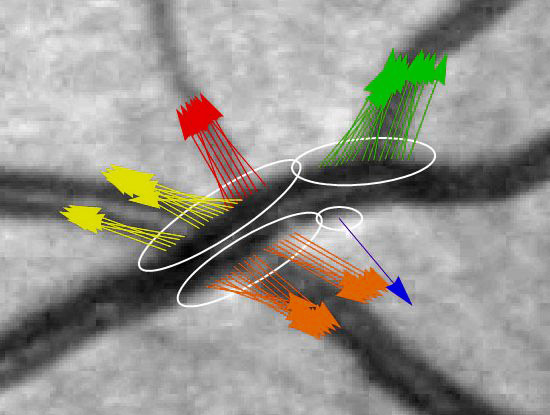}\;\;
\includegraphics[width=0.52\textwidthtwo]{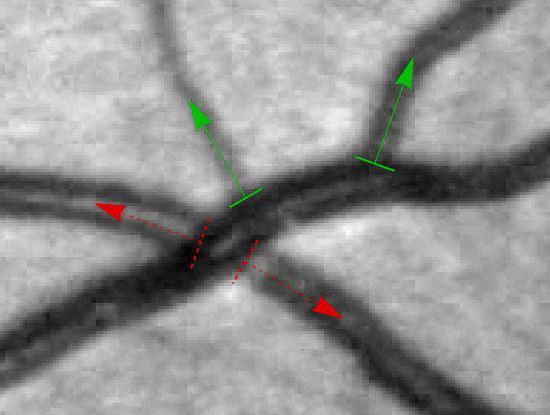}
\end{array}$
\end{center}
\caption{Junction detection. The top image: A tracked vessel segmented in white with candidate junction points indicated by colored arrows. The junction points are clustered on position (indicated by white ellipsoids) after which they are clustered on orientation (indicated by different colors). Bottom image: Each cluster is merged to a single junction point and based on proximity to other junction points, orientation and width, junction points are classified as bifurcation (green arrows) or crossing (red dashed arrows).}
\label{fig:jd}
\end{figure}

\subsubsection{Junction resolving}
\label{sec:junctionResolving}
As mentioned in Section~\ref{sec:stopCriteria} vessel tracking is terminated whenever the algorithm is tracking a vessel that is already tracked. This criterion can be met at several situations, where for each situation appropriate actions need to be taken in order to maintain correct topological models of the vasculature. A detected point that suggests inappropriate modeling of the vasculature will be called an \emph{unresolved junction point}. The detection of an unresolved junction point, together with the corresponding actions that are necessary to solve it will be called \emph{junction resolving}. The appropriate actions necessary for junction resolving are based on the position of the junction point on the already established track, whether or not the two overlapping segments have the same with and in the case the

The junction resolving steps are described using the following labeling: $S_{old}$ is the tracked segment of the established track before the junction point (Source), $T_{old}$ is the tracked segment of the same track after the junction point (Target), $S_{new}$ is the segment of the new track before the junction point, and $T_{new}$ is the segment of the new track after the junction point. The appropriate actions necessary for junction resolving are based on
\begin{enumerate}
  \item the position of the junction point on the already established track ($S_{new}$-$T_{new}$),
  \item whether or not the two overlapping segments ($T_{new}$ and $T_{old}$) have the same width,
  \item and the similarity between the source and target tracks.
\end{enumerate}
The complete junction resolving scheme and all detail can be found in \cite{Bekkers2012}.

\begin{figure*}[!ht]
        \centering
        \begin{subfigure}[b]{0.327\textwidth}
                \centering
                \includegraphics[width=\textwidth]{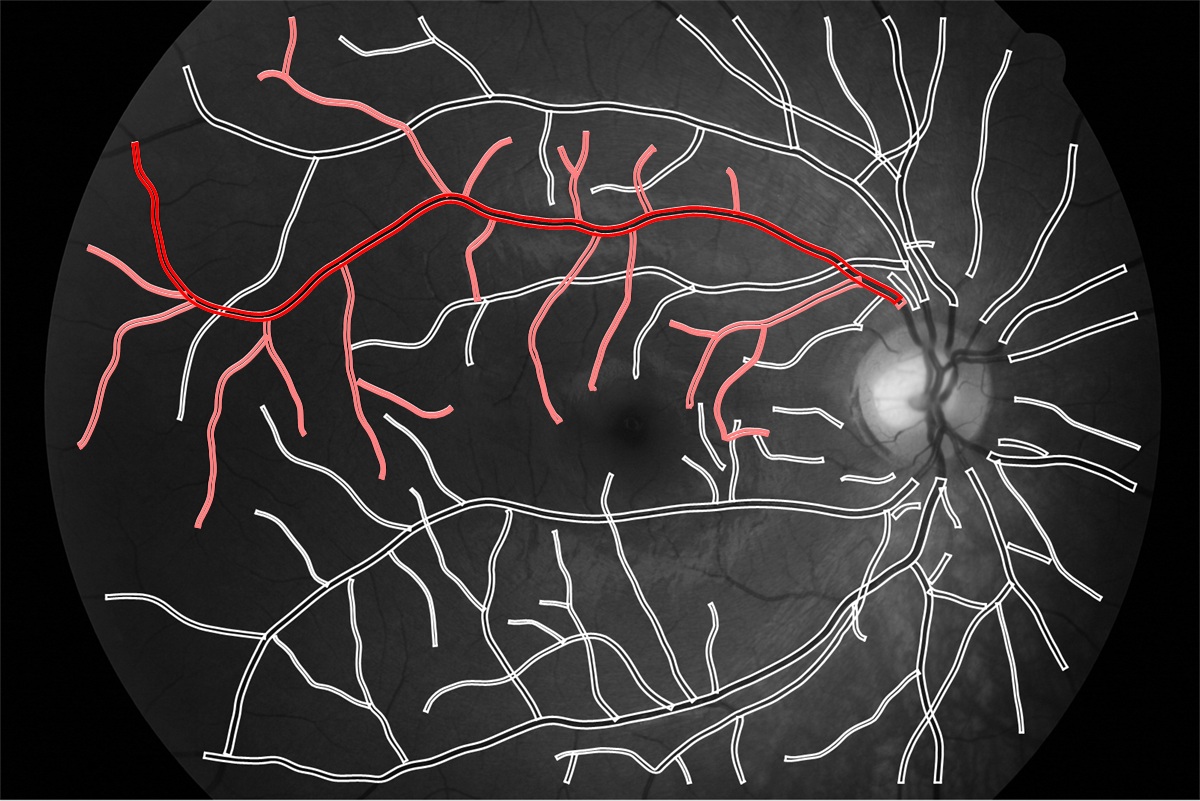}
                \caption{Vessel tree extraction}
                \label{fig:ModelPotential:a}
        \end{subfigure}
        \begin{subfigure}[b]{.327\textwidth}
                \centering
                \includegraphics[width=\textwidth]{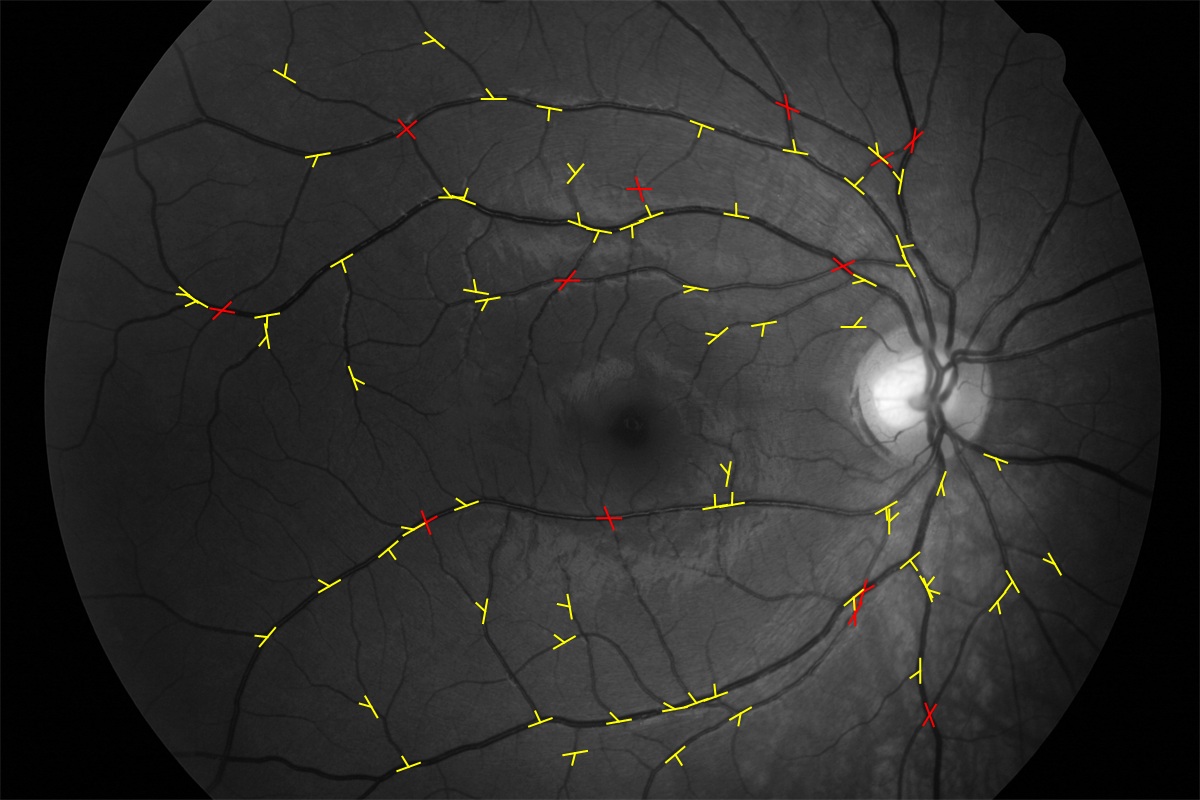}
                \caption{Junction points}
                \label{fig:ModelPotential:b}
        \end{subfigure}
        \begin{subfigure}[b]{.327\textwidth}
                \centering
                \includegraphics[width=\textwidth]{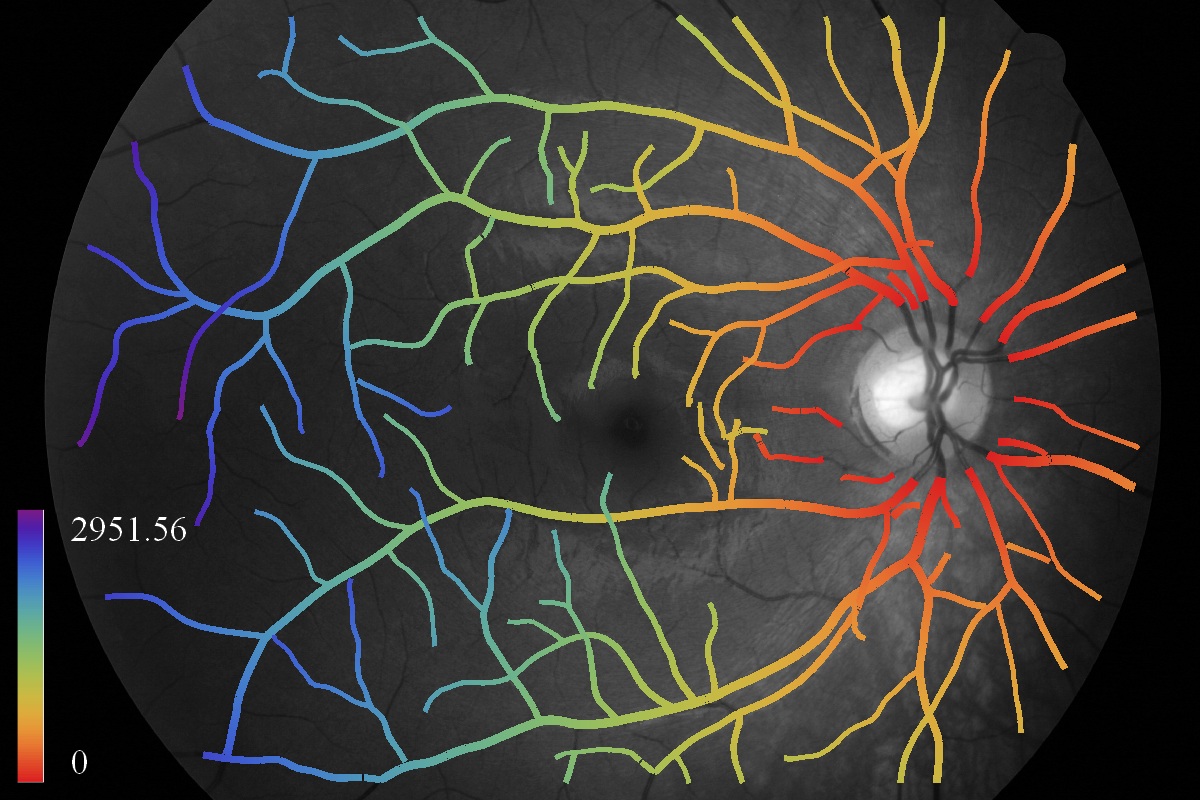}
                \caption{Distance to optic disk}
                \label{fig:ModelPotential:c}
        \end{subfigure}
        \caption{(a):The hierarchical structure of the generated vasculature models allow the segmentation and analysis of complete branches. (b) The automatic extraction of branching (yellow) and crossing points (red). (c) The distance to the optic disk; a feature that can easily be extracted because of the guaranteed connectedness of vessel segments in the generated vasculature models.}
        \label{fig:ModelPotential}
\end{figure*}

\subsection{Validation}
\label{sec:vasculatureResults}
In Section~\ref{sec:ValidationOfWidthMeasurements} we demonstrated the reliability of the width measurements provided by the ETOS algorithm. In the following section we validate the topological correctness of the complete vasculature models that are generated by our algorithm. The correctness of the models is validated by analyzing the junction points. The results discussed in this section are generated with the same parameters for the ETOS algorithm that are described in Section~ \ref{sec:AlgorithmBehaviorAtComplexVesselJunctionPoints}. A typical model generated by the vasculature tracking algorithm with these parameters is shown in Fig.~\ref{fig:ModelPotential}.

\begin{table*}[!ht]
  \centering
  \caption{Validation of detected bifurcations}
    \begin{tabular}{r|rlrlrlr|rlrlr}
    \toprule
    \multicolumn{1}{c|}{} & \multicolumn{7}{c|}{\textbf{Bifurcations}} & \multicolumn{5}{c}{\textbf{Crossings}}\\
    \multicolumn{1}{r|}{Disease} & \multicolumn{2}{c}{Correct} & \multicolumn{2}{c}{Crossing (E1)} & \multicolumn{2}{c}{False vessel (E2)} & \multicolumn{1}{r|}{Total} & \multicolumn{2}{c}{Correct} & \multicolumn{2}{c}{Incorrect} & Total\\
    \midrule
    Healthy         & 90 &(81.82\%) & 16 &(14.55\%) & 4 &(3.64\%)   & 110 & 41 & (100.00\%) & 0 & (0.00\%) & 41\\
    Diabetes        & 111 &(73.03\%)& 25 &(16.45\%) & 16 &(10.53\%) & 152 & 37 & (88.10\%)  & 5 & (11.90\%)& 42\\
    Glaucoma        & 89 &(74.79\%) & 30 &(25.21\%) & 0 &(0.00\%)   & 119 & 31 & (100.00\%) & 0 & (0.00\%) & 31\\
    \textbf{All}    & \textbf{290} & \textbf{(76.12\%)} & \textbf{71} & \textbf{(18.64\%)} & \textbf{20} & \textbf{(5.25\%)} & \textbf{381} & \textbf{109} & \textbf{(95.61\%)} & \textbf{5} & \textbf{(4.39\%)}& \textbf{114}\\
    \bottomrule
    \end{tabular}%
  \label{tab:junctions}%
\end{table*}

\begin{figure}[!h]
        \centering
        \begin{subfigure}[b]{.5\textwidthtwo}
                \centering
                \includegraphics[width=.5\textwidthtwo]{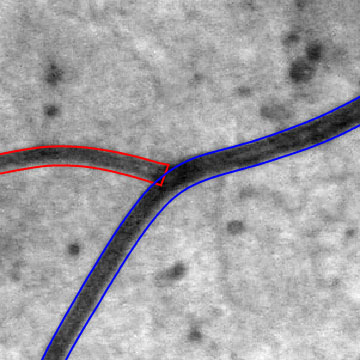}
        \end{subfigure}
        \begin{subfigure}[b]{.5\textwidthtwo}
                \centering
                \includegraphics[width=.5\textwidthtwo]{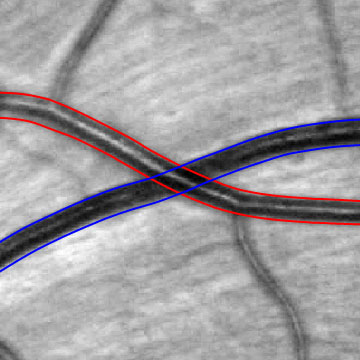}
        \end{subfigure}
        \caption{A typical bifurcation (left) and crossing (right), detected from a model generated by our vasculature tracking algorithm.}
        \label{fig:junctions}
\end{figure}

\subsubsection{Validation of topological correctness}
\label{sec:junctionsAndVasculature}
For each vessel it is known from which parent vessel it originates and bifurcations can thus be directly extracted from the model. Crossings can easily be extracted by detecting overlapping vessel segments. Fig.~\ref{fig:ModelPotential:b} provides an overview of detected junction points for one image of the HRFI-database, Fig.~\ref{fig:junctions} shows detailed views of a bifurcations and a crossing. We evaluated the junction points of the first three images from each of the three HRFI-datasets (healthy, diabetes and glaucoma). The 9 generated vasculature models provided 495 junction points, of which 381 were bifurcation points and 114 were crossing points. The following types of errors for bifurcations were identified: the bifurcation was actually part of a crossing (E1), the vessel originating from the bifurcation did not represent a blood vessel according to the ground truth pixel map provided by the HRFI-database (E2). E2 errors indicate the presence of incorrect single vessel models, non-vessel elongated structures such as the optic disk border or pathological features such as aneurysms. Crossings are extracted by searching the vasculature model for overlapping vessel segments. A false positive crossing can thus only occur if false positive vessel segments exists within the models. The results are summarized in Table~\ref{tab:junctions}.

In total 290 out of 381 bifurcations and 109 out of 114 crossings were correctly detected, corresponding to precision rates of 76.54\% and 96.03\% respectively. Most of the incorrectly identified bifurcations were correct in the sense that they represent a true blood vessel, however they were actually part of a crossing. Only 5.25\% of the bifurcations were incorrect in the sense that they did not represent a blood vessel. While the first kind of false positive bifurcations only affect the topological correctness of the model, the latter also pollutes the model with false positive vessel segments. The low percentage of E2 errors indicates that the generated vasculature models are very clean in the sense that almost all vessel segments actually represent true blood vessels.

Note that E1 errors are introduced by incorrect junction classification and that each miss-classified crossing introduces two E1 errors. Additional post-processing steps that try to solve for correct topology \cite{Al-Diri2010a} may be used to filter out these errors.

\section{Conclusion}
\label{conclusion}
In this paper we demonstrated that by representing image information in an invertible orientation score, one can exploit the disentanglement of crossing structures to track blood vessels through crossing points. We introduced a new algorithm that tracks vessel edges through the orientation score of an image (ETOS). The ETOS algorithm can generally be used with both invertible and non-invertible orientation scores, which were in this paper constructed with cake wavelets and Gabor wavelets respectively. We demonstrated that best results were obtained using invertible orientation scores. We also introduced a fast alternative method based on vessel centerline tracking through a multi-scale set of non-invertible orientation scores (CTOS). While the CTOS algorithm is very fast, the multi-scale approach makes the algorithm less stable at critical vessel points (crossings and parallel vessels) compared to ETOS.

The ETOS algorithm was used as a basis for our vasculature tracking algorithm, which we used to construct detailed hierarchical models of the retinal vasculature. Within this paper we validated the reliability of the width measurements provided by the models using ground truth data, and showed that our method performs excellent in comparison to other state of the art algorithms. Validation of the topology of the models showed that our method constructs clean topological models of the vasculare tree, i.e. they contain very few false
positive vessels.

Tracking within orientation scores relies on a novel and basic geometrical principle ($\mathcal{V}$-plane optimization) within a sub-Riemanian manifold within SE(2). In Appendix~\ref{app:optimalpaths} we relate probabilistic approaches to optimal curves in SE(2) to this geometric principle. Here we include supporting examples of analytic and numeric computations on completion fields on SE(2). Such completion fields are obtained from collision probabilities on SE(2), while the orientation score provide a whole distribution of these particles. The next step in future work is to apply tracking by geometric control \cite{Duits2013} in enhanced orientation scores, obtained by applying stochastic processes for \cite{Duits2010,DuitsR2006AMS} contour completion and enhancement directly on the score.

\appendix

\section{A mathematical underpinning of optimization in the $\mathbf{e}_\eta - \mathbf{e}_\theta$ tangent planes $\mathcal{V}$\label{ch:motivate}}
\label{app:optimalpaths}

Before we provide the key-geometrical principle behind our ETOS algorithm, we make some comments on the moving frame of reference that lives in the domain of an orientation score.

The domain of an orientation requires a curved (non-Euclidean) geometry as can be seen in Fig.~\ref{fig:OrientationScores}. To this end
we identify the coupled space of positions and orientations $\R^{2}\rtimes S^{1}$
with $SE(2)$. This needs a bit of explanation.
To see whether a local line element $(\ul{x}',\theta') \in \R^{2}\times S^{1}$ with position $\ul{x}'\in \R^{2}$ and orientation $\theta'$
is aligned with its neighbors one applies a rigid body motion $g=(\ul{x},\ul{R}_{\theta}) \in SE(2)$, with counter-clockwise rotation and translation via Eq.~(\ref{action}).
This puts a 1-to-1 correspondence between rigid body motions $(x,y,\ul{R}_{\theta})$ and elements from $\R^{2}\times S^{1}$
as every local line element $(x,y,\theta)$ can be obtained by a local line element $(0,0,0)$ positioned at the origin $\ul{0}=(0,0)$ pointing in say $x$-direction:
\[
(\ul{x},\theta) = (\ul{x},R_{\theta}) (\ul{0},0) \desda (\ul{x},\theta) \equiv (\ul{x},R_{\theta}).
\]
Via this identification we see that the curved domain of an orientation score is due to the non-commutative group structure of the rigid body motion (also known as roto-translation group or Euclidean motion group) $SE(2)$ in the plane. Indeed, a concatenation of two rigid body motions is again a rigid body motion and this induces the group-product given by Eq.~(\ref{product})

It is a \emph{semi-direct} product of rotations $SO(2)$ and translations $\R^{2}$, as the rotation part
affects the position part. As a result this group product is non-commutative, i.e. $g g'$ can be very different from  $g'g$.
Therefore it is conceptually wrong to consider the domain of an orientation scores as the flat Cartesian space $\R^{2} \times S^{1}$ with Euclidean metric and it is wrong to just take derivatives of orientation scores only in $\partial_{x}$ direction or only in $\partial_{y}$ direction. For formal underpinning of this statement see \cite[Thm.21]{Duits2005}, for a short illustration see \cite[Fig.2.6]{Franken2008}. Instead one must differentiate along the moving frame of reference $\{\partial_{\theta},\partial_{\xi}, \partial_{\eta}\}$ with
\[
\partial_{\theta}, \partial_{\xi}=\cos \theta \partial_{x} + \sin \theta \partial_{y}, \partial_{\eta}=
-\sin \theta \partial_{x} +\cos \theta\partial_{y}.
\]
This a priori moving frame of reference coincides with the so-called Lie algebra of left-invariant vector fields on the group $SE(2)$. The curved geometry in the orientation score requires us to apply parallel transport along this moving frame of reference. For instance in Fig.~1(d) we see examples of auto-parallel curves whose tangent vector is always constant to the moving frame of reference. The non-commutative structure of $SE(2)$ (due to the fact that rotations and translations do not commute) is reflected by the commutators in the Lie-algebra
\[
\begin{array}{l}
\, [\partial_{\theta},\partial_{\xi}]=\partial_{\eta}, \
\, [\partial_{\theta},\partial_{\eta}]= -\partial_{\xi}, \
\, [\partial_{\xi},\partial_{\eta}]= 0, \\
\end{array}
\]
where $[A,B]=AB-BA$. As a result first moving in $\xi$ direction and then moving in $\theta$ direction brings us to a different point in $SE(2)\equiv \R^{2} \times S^{1}$ than first moving in $\theta$ direction and then moving in $\xi$ direction (as can be deduced in Fig.~1(b)).

\subsection{The geometrical principle behind the ETOS-algorithm}

The ETOS-algorithm presented in Section~\ref{sec:EdgeTrackingUsingInvertibleOrientationScores} heavily relies on local optimization in each transversal 2D-tangent plane
$
\mathcal{V}=\textrm{span}\{\partial_{\theta}, \partial_{\eta}\}$
spanned by $\partial_{\theta}$ and $\partial_{\eta}=-\sin \theta \partial_{x} + \cos \theta \partial_{y}$ in tangent-bundle
\[
T(SE(2))=\bigcup \limits_{g \in SE(2)} T_{g}(SE(2)),
\]
where
$T_{g}(SE(2))$ denotes the tangent space at $g=(x,y,\theta) \in SE(2)$.

To this end we recall Fig.~\ref{fig:ETOS3D} where one can observe that
each tangent vector to the lifted curve $s \mapsto (x(s),y(s),\theta(s))$ with $\theta(s)=\arg(\dot{x}(s)+i \dot{y}(s))$ (e.g. the curve following the right edge of a blood vessel) in $SE(2)$ is pointing orthogonal to a plane $\mathcal{V}$ (plotted in yellow). Within $\mathcal{V}$ we see that the maximum values of the absolute value of the imaginary part of the score is located at the origin of the yellow plane in the tangent space.
\begin{definition}
A smooth curve $s \mapsto \gamma(s)=(x(s),y(s),\theta(s))$ in $SE(2)$ is called the lifted curve of a smooth planar curve iff for all $s \in [0,\ell]$ we have $\theta(s)=\arg (\dot{x}(s) +i \, \dot{y}(s))$.
\end{definition}
\begin{definition}
If a curve $\gamma$ is equal to the lift of its spatial projection (i.e. if its satisfies Eq.~(\ref{horizontal}))
it is called horizontal.
\end{definition}
Tangent vectors to horizontal curves always lay
in the tangent plane
$H=\textrm{span}\{\partial_{\theta},
\partial_{\xi}:= \cos \theta \partial_{x}+\sin \theta \partial_{y}\}$,
spanned by $\partial_{\xi}$ and $\partial_{\theta}$, since one has
\begin{equation} \label{horizontal}
\begin{array}{l}
\theta(s)=\arg (\dot{x}(s) +i \dot{y}(s)) \desda \\
\dot{\gamma}(s) \in \textrm{span}\{\left.\partial_{\xi}\right|_{\gamma(s)},\partial_{\theta}\},
\end{array}
\end{equation}
In this appendix we will underpin and discuss our fundamental venture point: The most salient curves
in the smooth imaginary part/real part/absolute value $C:SE(2) \to \R$ of the smooth orientation score $U:SE(2)\to \mathbb{C}$ are given by
\[
\gamma \textrm{ is horizontal such that } \partial_{\theta}C(\gamma)=0 \textrm{ and } (\left.\partial_{\eta}\right|_{\gamma}C)(\gamma)=0
\]
The intuitive idea behind this is as follows.
Let $C:SE(2)\mapsto \mathbb{R}^+$ denote an a priori given cost. Now,
if a horizontal curve $s \mapsto \gamma(s):=(\ul{x}(s), \theta(s))$
satisfies
\begin{equation} \label{theeqs}
\begin{split}
\partial_{\eta} C(\ul{x}(s),\theta(s))&= (-\sin \theta(s) \partial_{x}C + \cos \theta \partial_{y} C)( \ul{x}(s),\theta(s))\\
& = \partial_{\theta}C(\ul{x}(s),\theta(s))=0
\end{split}
\end{equation}
with $\partial_{\eta}^{2} C(\ul{x}(s),\theta(s)) <0 $ and $\partial_{\theta}^{2} C(\ul{x}(s),\theta(s)) <0 $, then there is no gain in moving\footnote{In such a way that the perturbed curve is again horizontal.} the curve $s \mapsto \gamma(s)=(\ul{x}(s),\theta(s))$ in directions orthogonal to the spatial propagation vector
$\partial_{\xi} |_{\gamma(s)}=\cos \theta(s) \partial_{x} + \sin \theta(s) \partial_{y}$.

In fact, we expect the curve $\gamma(s)=(\ul{x}(s),\theta(s))$ satisfying (\ref{theeqs})
to be optimal in some sense within the sub-Riemannian manifold
\[
M=(SE(2), \textrm{span}\{\partial_{\xi},\partial_{\theta}\}, G_{\beta})
\]
with metric tensor $G_{\beta}:SE(2) \times T(SE(2)) \times T(SE(2)) \to \R$ is given by
\begin{equation} \label{Gxi}
\begin{array}{ll}
\left.G_{\beta}\right|_{(x,y,\theta)} &= {\rm d}\theta \otimes {\rm d}\theta  + \\
 &\beta^{2}(\cos \theta {\rm d
}x \!+\!\sin \theta {\rm d}y) \otimes (\cos \theta {\rm d
}x \!+\!\sin \theta {\rm d}y),
\end{array}
\end{equation}
where $\otimes$ denotes the tensor product. Note that
\[
\begin{array}{l}
(\cos \theta {\rm d
}x +\sin \theta {\rm d}y) (\dot{\gamma}(s))=
\dot{x}(s) \cos \theta(s) +\dot{y}(s)\sin \theta(s)  \Rightarrow \\
G_{\beta}(\dot{\gamma}(s),\dot{\gamma}(s))=
\beta^{2}|\cos \theta(s)\dot{x}(s) +\sin \theta(s) \dot{y}(s)|^2 +|\dot{\theta}(s)|^{2}\ .
\end{array}
\]
If $s$ equals arclength (i.e. we have $\|\dot{\ul{x}}(s)\|^{2}=1$) of the spatial part $s \mapsto \ul{x}(s)=(x(s),y(s))$ of the horizontal curve $s \mapsto \gamma(s)=(\ul{x}(s),\theta(s))$ we have $\kappa^{2}(s)=|\dot{\theta}(s)|^2$, so that
\begin{equation} \label{needed}
\|\dot{\gamma}(s)\|=
\sqrt{G_{\beta}(\dot{\gamma}(s),\dot{\gamma}(s))}= \sqrt{\kappa^{2}(s) + \beta^{2}}.
\end{equation}

Intuitively, such a sub-Riemannian manifold $M$ equals the group $SE(2)$ where one restricts oneself to horizontal curves
with a constant relative penalty for bending and stretching determined by $\beta>0$ which has physical dimension one over length.

For special cases of $C$ we can show that our geometrical principle indeed produces optimal curves in $SE(2)$, as we will show in the subsequent section.

\subsection{Application of the geometrical principle to completion fields}

In general the real part, imaginary part, or absolute value of an orientation score is a complicated function on $SE(2)$. Therefore
we will consider the case where $C: SE(2) \to \R^{+}$ is a so-called ``completion field'' \cite{Thornber2,Zweck,August,Thornber,August2003}.
This corresponds to collision probability densities of a source particle $g_{1} \in SE(2)$ and a sink particle $g_{2} \in SE(2)$.

There exist remarkable relations \cite{Duits2010a,DuitsR2006AMS} between optimal curves (i.e. curves minimizing an optimal control problem on $SE(2)$) and solutions of Eq.~(\ref{theeqs}) for special cases where $C$ denotes a so-called completion distribution (or ``completion field'').
Given two sources at the origin $g_1=(0,0,0)$ and at $g_{2}=(x_{1},y_{1},\theta_{1})$,
such completion fields are defined as products of resolvent Green's functions of stochastic processes for contour completion \cite{Mumford} and contour enhancement \cite{Duits2010,Citti} in $SE(2)$:
\begin{equation} \label{cost}
C(g):= \lambda^{2} \cdot R_{\lambda}(g_{1}^{-1}g) R_{\lambda,*}(g_{2}^{-1}g),
\end{equation}
where $R_{\lambda}(g)$ denotes the probability density of finding a random walker $g$ in the underlying stochastic process \cite{Duits2010,Mumford} given that it started at $g=(x,y,\theta)=(0,0,0)$ regardless its memoryless traveling time $T$ which is negatively exponentially distributed with expectation $E(T)=\lambda^{-1}$.
\begin{figure}
\includegraphics[width=\hsize]{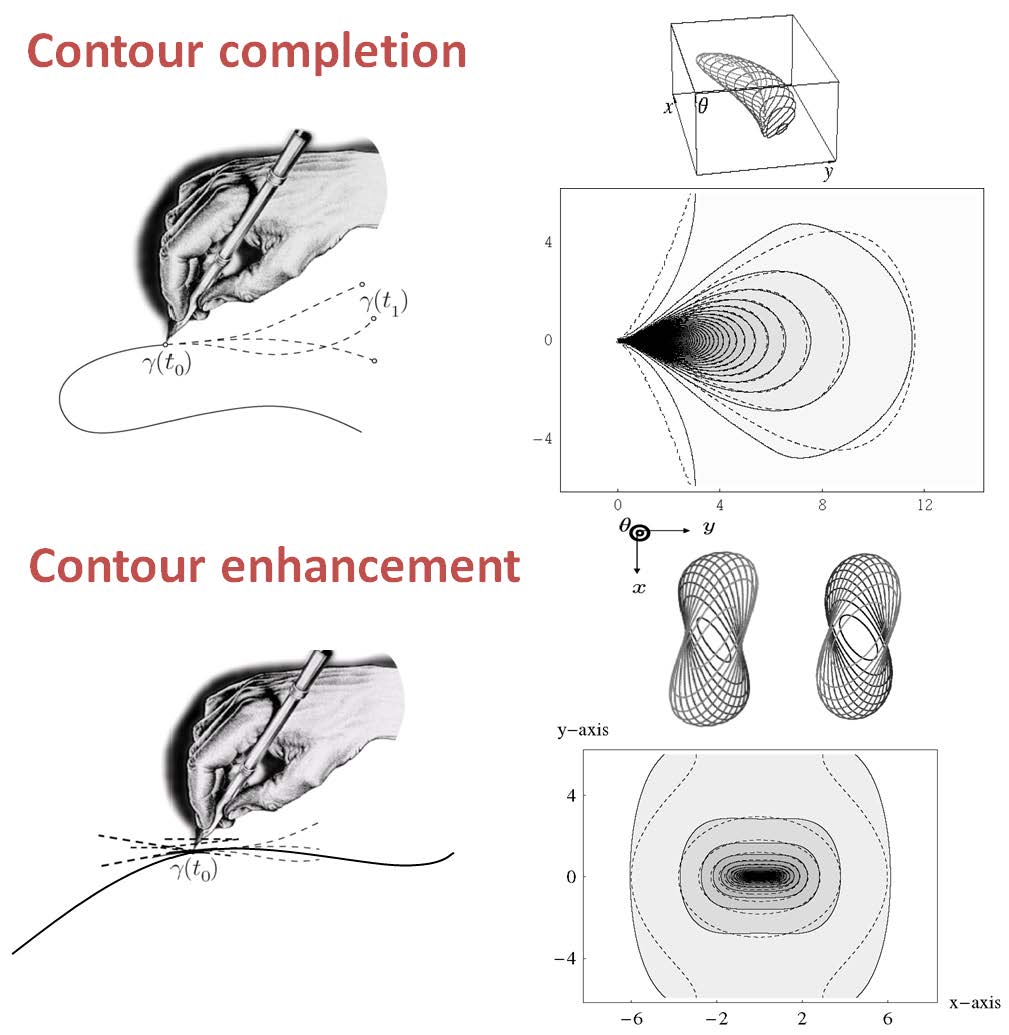}
\caption{Top: A contour completion process and the corresponding Green's function (in $SE(2)$) and its planar $\theta$-integrated version.
Bottom: A contour enhancement process the corresponding Green's function (in $SE(2)$) and its planar $\theta$-integrated version. In the completion process one has randomness in $\theta$ with variance $2D_{11}>0$ and non-random advection in $\xi$. In the enhancement process one has randomness both in $\theta$-direction (with variance $2D_{11}>0$) and in $\xi$-direction (with variance $2D_{22}>0$).
}\label{fig:contourEnhCom}
\end{figure}
Furthermore, $g \mapsto R_{\lambda,*}(g)$ denotes the adjoint resolvent kernel (i.e. the resolvent
that arises by taking the adjoint of the generator \cite[ch:4.4]{DuitsR2006AMS}.

For a concise overview on (Green's functions of ) contour completion and enhancement see \cite{DuitsFranken}.
Exact formulas for the resolvent Green's function $R_{\lambda}$ for contour enhancement can be found in \cite{Duits2010}, whereas exact formulae for resolvent kernels $R_{\lambda}$ for contour completion (direction process) can be found in \cite{DuitsR2006AMS}. In both cases there are representations
involving 4 Mathieu functions. For a visual impression of exact Green's functions see Fig.~\ref{fig:contourEnhCom} (where in dashed lines we have depicted level sets of the corresponding Heisenberg approximations that we will discuss and employ in the next subsection).

These completion fields relate to the well-known Brownian bridges (in probability theory) where the traveling time is integrated out. This relation is relevant, since it is known that such Brownian bridges concentrate on geodesics \cite{wittich}.

If $G_{t}:SE(2) \to \R^{+}$ denotes the time dependent Green's function of the Fokker-Planck equation of the underlying time dependent stochastic process \cite{DuitsFranken,Duits2010,Mumford} at time $t>0$ we have
\[
C(g)= \lambda^2 \int \limits_{0}^{\infty} \int \limits_{0}^{t} G_{t-s}(g_{1}^{-1}g) e^{-\lambda(t-s)} G_{s}(g_{2}^{-1}g)
e^{-\lambda s} \, {\rm d}s {\rm d}t\ .
\]
This identity follows from two facts. Firstly, the resolvent Green's function
follow from the time dependent Green's function
via Laplace transform with respect to time
\[
R_{\lambda}(g)= \lambda \mathcal{L}(t \mapsto G_{t}(g))(\lambda) = \lambda \int_0^\infty G_t(g) e^{-\lambda t}dt.
\]
Secondly, a temporal convolution relates to a product in the Laplace domain.
As Brownian bridge measures concentrate on geodesics when $\lambda \to 0$, cf.~\cite{wittich},\cite[App.B]{DuitsR2007ArXiV},
the completion field for contour enhancement concentrates on sub-Riemannian geodesics (shortest distance curves) within $M$.

Such a sub-Riemannian geodesic $\gamma=(\ul{x},\theta)$ connecting $g_{1}$ and $g_{2}$ (with $g_1$ and $g_{2}$ sufficiently close\footnote{This means $g_2^{-1}g_{1}$ is within the range of the exponential map of the sub-Riemannian control problem, i.e. $g_1$ and $g_{2}$ can be connected via a stationary curve without cusps \cite{CDC,Boscain2,Duits2013}.}) is the lifted curve of the minimizer to the following optimal control problem
\begin{equation} \label{geod}
\begin{array}{l}
\min \limits_{
\begin{array}{c}
\ul{x} \in C^{1}([0,\ell],\R^{2}), \\
\ell>0, \\
(\ul{x}(0),\dot{\ul{x}}(0))=g_1, \\
(\ul{x}(\ell),\dot{\ul{x}}(\ell))=g_2
\end{array}
} \int \limits_{0}^{\ell} \sqrt{\kappa^{2}(s)+\beta^{2}}\, {\rm d}s \\
=
\min \limits_{
\begin{array}{c}
\gamma=(\ul{x},\theta) \in C^{1}([0,\ell], SE(2)),\\
 \ell>0, \\
(\ul{x}(0),\dot{\ul{x}}(0))=g_1, \\
(\ul{x}(\ell),\dot{\ul{x}}(\ell))=g_2, \\
\theta(s)= \arg(\dot{x}(s)+i \dot{y}(s)) \\
\end{array}
} \int \limits_{0}^{\ell} \sqrt{G_{\beta}(\dot{\gamma}(s),\dot{\gamma}(s))}\, {\rm d}s
\end{array}
\end{equation}
with spatial arc-length parameter $s>0$ and
with free total length $\ell>0$ and with curvature $\kappa(s)=\|\ddot{x}(s)\|$ and where metric tensor $G_{\xi}$ is defined by Eq.~(\ref{Gxi}).

On the other hand, in his paper \cite{Mumford} Mumford showed that the modes of the direction process (also known as the contour completion process) coincide with elastica curves which are the solutions to the following optimal control problem
\begin{equation} \label{square}
\inf \limits_{
\begin{array}{c}
\ul{x} \in C^{1}[0,\ell], \ell>0 \\
(\ul{x}(0),\dot{\ul{x}}(0))=g_0, \\
(\ul{x}(l),\dot{\ul{x}}(\ell))=g_1
\end{array}
} \int \limits_{0}^{\ell} \kappa^{2}(s)+\beta^{2}\, {\rm d}s
\end{equation}
Similar to the Onsager-Machlup approach to optimal paths \cite{Takahashi} he obtains these modes by looking at the most probable/likely realization of discretized versions
of the direction process. 

As this cannot be (efficiently) realized in practice, one needs a more tangible description of the mode.
To this end, we will call solution curves of (\ref{theeqs}) the ``modes''. In case of the direction process and in case one uses the completion distribution (Eq.~(\ref{cost})) for the function $C(g)$, the solution curves of (\ref{theeqs}) indeed seem to numerically coincide with the elastica. This experimentally underpins our conjecture in \cite{DuitsR2006AMS}, where we have shown such a result does hold exactly for the corresponding Heisenberg approximations as we will explain next.

\subsection{In the Heisenberg approximation of completion fields our approach produces B-splines}

The Heisenberg group approximation (obtained by contraction \cite{Duits2010}) of the Green's functions and induced completion field arises by replacing the moving frame of left-invariant vector fields
\[
\{\cos \theta \partial x + \sin \theta \partial_y,-\sin \theta \partial_{x} +\cos \theta \partial_{y},\partial_{\theta}\} \]
on $SE(2)$ by the moving frame of reference of left-invariant vector fields on the Heisenberg group \[
\{\partial_{x}+ \theta \partial_{y},\partial_{y},\partial_{\theta}\}\]
and by replacing spatial arc-lengt parametrization  via $s$ by spatial coordinate $x$. Intuitively, such replacement boils down to replacing the space of positions and orientations by the space of positions and velocities.

When contracting (for details on this contraction see \cite[Ch.5.4]{Duits2010}) our fundamental equation (\ref{theeqs}) with cost $C$ given by Eq.~(\ref{cost}) towards the Heisenberg group $H_{3}$ we obtain
\begin{equation} \label{approx}
\frac{d}{dy} \tilde{C}(x,y,\theta) = \frac{d}{d\theta} \tilde{C}(x,y,\theta)  =0,
\end{equation}
where again in the Heisenberg group each tangent plane $\{\partial_{\theta},\partial_{y}\}$ is orthogonal to the propagation direction $\partial_{x} + \theta \partial_{y}$ and where the completion field is the product of two resolvent Green's functions \[
\tilde{C}(x,y,\theta)= R_{\lambda}(x,y,\theta) R_{\lambda}(-x+x_{1},y-y_{1}-\theta_{1}(x-x_{1}),-\theta)
\]
which can be derived in exact form as
\[
R_{\lambda}(x,y,\theta)= \frac{\lambda\sqrt{3}}{2 D_{11} \pi x^2} e^{-\lambda x}e^{-\frac{3(x\theta-2y)^2 +x^2\theta^2 }{4x^3 D_{11}}} {\rm u}(x),
\]
where $D_{11}>0$ stands for the amount of diffusion in $\theta$-direction and ${\rm u}$ for the unit step function,
for details\footnote{Set $\kappa_0=\kappa_{1}=0$ in \cite{DuitsR2006AMS}.} see (cf.~\cite[Thm~4.6]{DuitsR2006AMS}).

Interestingly, the solution of (\ref{approx}) is a third order polynomial $y(x)$ naturally lifted to
(the corresponding sub-Riemannian manifold within) $H(3)$ by setting $\theta(x)=y'(x)$. The solutions of Eq.~(\ref{approx}) are therefore
cubic $B$-splines which are the solutions of the Euler-Lagrange equation \[
y^{(4)}(x)=\theta^{(3)}(x)=0\]
of a curve optimization problem which arises by contracting (\ref{square}) towards the Heisenberg group. This gives the following ``Heisenberg group equivalent'' of control problem (\ref{square}):
\[
\begin{array}{l}
\min \limits_{{\small
\begin{array}{c}
\gamma(\cdot)=(\cdot, y(\cdot),\theta(\cdot)) \in C^{1}(H_{3}),   \\
\gamma(0)=(0,0,0), \\
\gamma(x_{2})=(x_2, y_{2},\theta_{2}), \\
\theta(x)=y'(x),
\end{array}
}} \int \limits_{0}^{x_{2}} \beta^2+|\theta'(x)|^{2} \, {\rm d}x \\[8pt]
= \beta^{2} x_{2}+ \frac{4(3y^2_{2}+3x_{2}y_{2}\theta_{2}+ x_{2}^{2}\theta_{2}^{2})}{x^{3}_2},
\end{array}
\]
which takes the minimum along a lifted cubic-B spline \\
$(x,y(x),y'(x))$ (with $y(x)$ a third order
polynomial matching the boundary conditions). For details
see
\cite[Eq.~(4.\!6.\!2)]{DuitsR2006AMS} and \cite[ch:9.1.1]{DuitsR2007ArXiV}. See Fig.~\ref{fig:approx}.
\begin{figure}
\centerline{
\includegraphics[width=0.6\hsize]{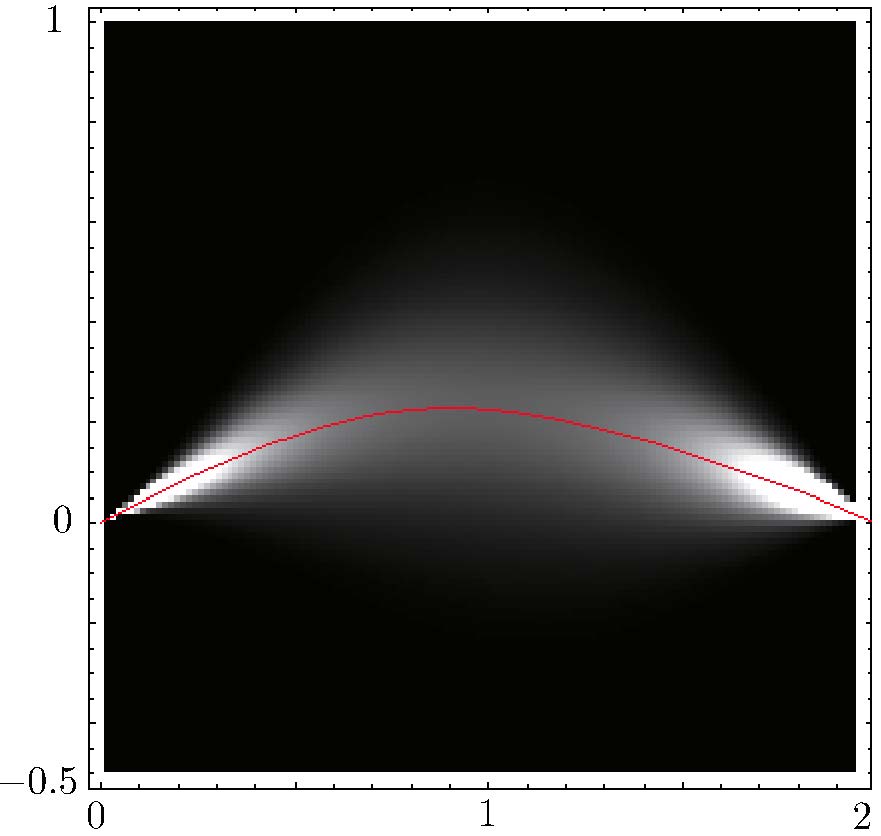}
}
\caption{The intersection of the planes $\{(x,y,\theta) \in \R^{3} \; |\; \partial_{\theta}\tilde{C}(x,y,\theta)=0\}$  and $\{(x,y,\theta) \in \R^{3} \; |\; \partial_{y}\tilde{C}(x,y,\theta)=0\}$ of the Heisenberg approximation
$\tilde{C}(g)$ of the completion field $C(g)$ given by Eq.~(\ref{cost}), produces a cubic B-spline lifted in $H(3)$, i.e.
$(x,y(x),\theta(x)=y'(x))$ with $y^{(4)}(x)=0$. Boundary conditions have been set to
$x_{1}=y_{1}=0$, $\theta_1=0.4$, $x_{2}=2$, $y_{2}=0$, $\theta_{2}=-0.4$, $D_{11}=1/8$.
}\label{fig:approx}
\end{figure}

\begin{figure}
\centerline{
\includegraphics[width=1.1\hsize]{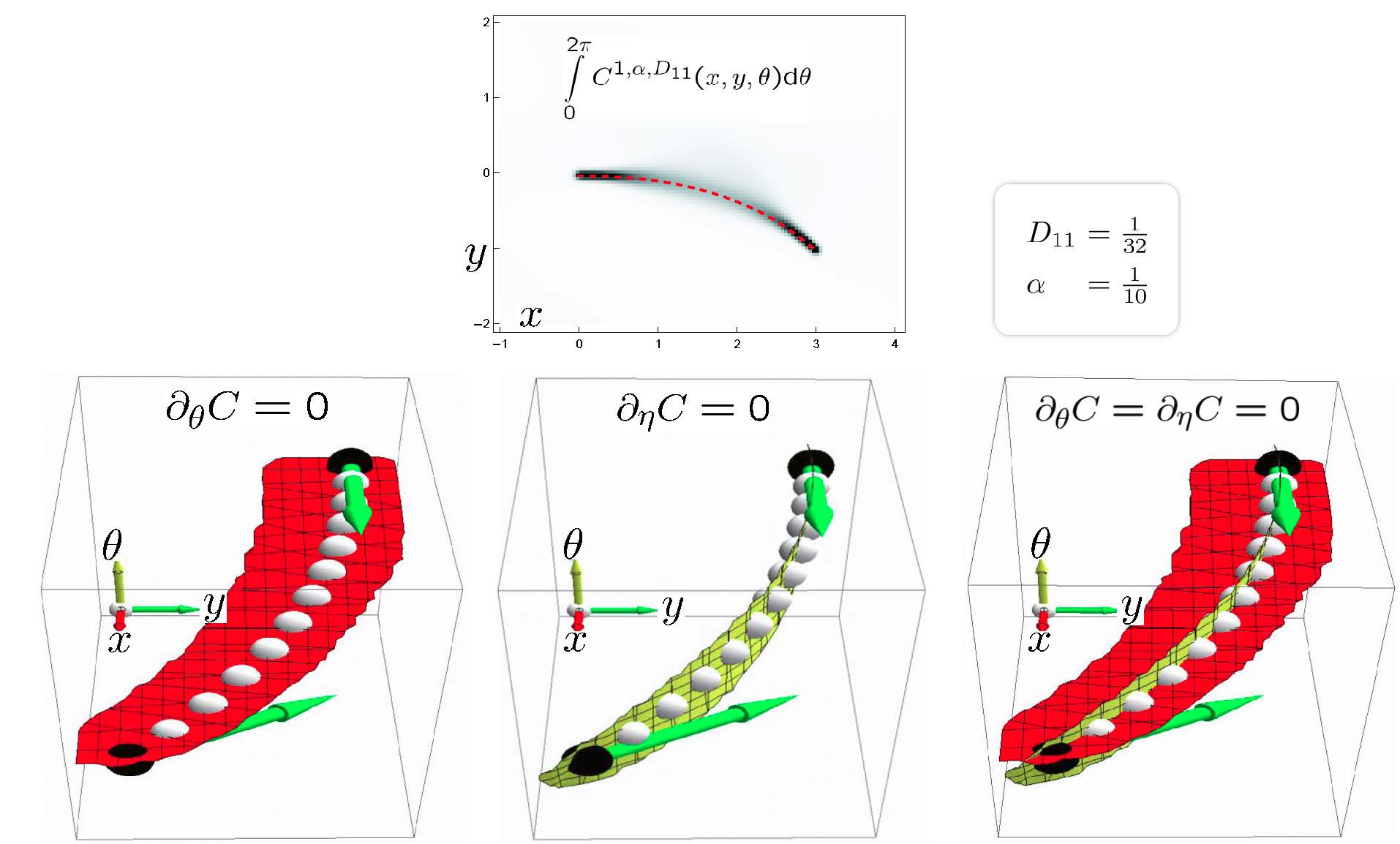}
}
\caption{The intersection of the planes $\{g \in SE(2) \; |\; \partial_{\theta}C(g)=0\}$ (depicted in red) and $\{g \in SE(2)\; |\; \partial_{\eta}C(g)=0 \}$ (depicted in yellow) of the exact completion field $g=(x,y,\theta) \mapsto C(g)$ given by Eq.~(\ref{cost}), with
$x_{1}=y_{1}=\theta_{1}=0$ and
$x_{2}=3$, $y_{2}=-1$, $\theta_{2}=7/4 \pi$ $\lambda=0.1$, $D_{11}=1/32$.
The intersection of these planes seems to coincides with an elastica (with $\beta^2=4\lambda D_{11}$), which
we plotted in dashed red in the top figure and by white balls in the bottom figures }\label{fig:conj}
\end{figure}

\subsection{Concluding remarks}

By the results of the previous section the conjecture rises whether elastica curves coincide with
Eq.~(\ref{cost}) as such relation holds for their counterparts in the Heisenberg group.
Numeric computations seem to provide a confirmation of this conjecture, see Fig.~\ref{fig:conj}.

On the one hand, we expect with respect to the contour enhancement process that our approach produces the sub-Riemannian geodesics (based on the results in \cite{wittich,DuitsR2007ArXiV,Duits2010a}), but this is a point for future investigation.
On the other hand, the conjecture together with the result in \cite{Mumford} (and the result that B-splines solve Eq.~(\ref{approx}))
would underpin our alternative \emph{applicable definition} of modes as solution curves of Eq.~(\ref{theeqs}).

In fact this means that optimization in each $(\eta,\theta)$-plane $\mathcal{V}$, as is done in our ETOS algorithm, produces the most probable curves (in the sense of  \cite{Mumford,Takahashi}) in direction processes.

\section{The closure of the distributional orientation score transform is an $\mathbb{L}_2$-isometry on the whole $\mathbb{L}_2(\mathbb{R}^2)$ space.}
\label{app:distributionalostrafo}

As mentioned in Section \ref{sec:ConstructionOfOrientationScores}, the condition given by Eq.~(\ref{eq:conditionmpsi}) violates the condition $\psi \in \mathbb{L}_1(\mathbb{R}^2) \bigcap \mathbb{L}_2(\mathbb{R}^2)$, as for such $\psi$, the function $M_\psi:\mathbb{R}^2 \rightarrow \mathbb{R}^+$ is a continuous function vanishing at infinity. In fact this follows by \cite[Thm.~7.5]{Rudin1973}, Fubini's theorem and compactness of SO(2).

To avoid technicalities one can restrict oneself a priori to bandlimited/disklimited images $f$, but this imposes a depending on sampling rate and induced Nyquist frequency $\varrho$. However, as we will show in this section such a restriction to band limited images is not crucial.

Akin to the unitary Fourier transform $\mathcal{F}:\mathbb{L}_2(\mathbb{R}^2) \rightarrow \mathbb{L}_2(\mathbb{R}^2)$, whose kernel $k(\www,\mathbf{x}) = e^{-i \www \cdot \mathbf{x}}$ is not square integrable, we can allow non-square integrable kernels and rely on Gelfand-Triples \cite{Wloka1987l} as we will explain next.

To this end we first drop the constraint $\psi \in \mathbb{L}_1(\mathbb{R}^d) \bigcap \mathbb{L}_2(\mathbb{R}^d)$ by imposing $\psi$ to be in a dual Sobolev-space:
\begin{equation}
\psi \in \mathbb{H}_{-I}(\mathbb{R}^2) = \mathbb{H}_I^*(\mathbb{R}^2) = (\mathcal{D}(D_\beta))^*,
\end{equation}
where
\begin{equation}
\mathbb{H}_I(\mathbb{R}^2)
= \mathcal{D}(D_\beta)
= \left\{ f \in \mathbb{L}_2(\mathbb{R}^2) \left| \parbox{10em}{$f$ admits generalized derivatives s.t. $D_\beta f~\in~\mathbb{L}_2(\mathbb{R}^2)$} \right. \right\}
\end{equation}
equipped with inner product
$$
(\cdot,\cdot)_{\mathbb{H}_I(\mathbb{R}^2)}=(D_\beta \cdot, D_\beta \cdot)_{\mathbb{L}_2(\mathbb{R}^2)}
$$
and where $\beta:\mathbb{R}^2 \mapsto \mathbb{R}^+$ is continuous, bounded from below, isotropic, and with differential operator
$$
D_\beta = \mathcal{F}^{-1} \beta \mathcal{F}.
$$
Now $D_\beta$ is an unbounded self adjoint, positive operator with bounded inverse. This means that $D_\beta$ defines a so-called Gelfand-Triple
$$
\mathbb{H}_I(\mathbb{R}^2) \hookrightarrow \mathbb{L}_2(\mathbb{R}^2) \hookrightarrow \mathbb{H}^*_I(\mathbb{R}^2) = \mathbb{H}_{-I}(\mathbb{R}^2),
$$
where $\mathbb{H}_{-I}(\mathbb{R}^2)$ is equipped with inner product
$$
(\cdot,\cdot)_{\mathbb{H}_{-I}(\mathbb{R}^2)} = (D_\beta^{-1} \cdot ,D_\beta^{-1} \cdot )_{\mathbb{L}_2(\mathbb{R}^2)},
$$
and where all embeddings are dense.

Subsequently, we define the \emph{distributional orientation score transform}\footnote{The unitary representation $\mathcal{U}_g$ naturally extends to $\mathbb{H}_I(\mathbb{R}^2)$ as the multiplier $\beta$ in the Fourier domain is isotropic.}
$$
(\gothic{W}_\psi f)(g) = \left< \psi, \mathcal{U}_{g^{-1}} f \right>,
$$
for all $f \in \mathbb{H}_I(\mathbb{R}^2)$ and all $g \in SE(2)$, where we applied the notation $\left<b,a\right>=b(a)$ for functional $b$ acting on vector $a$. Note that our non-distributional orientation score transform can be rewritten as
$$
(\mathcal{W}_\psi f)(g) = (\mathcal{U}_g \psi , f) = (\psi,\mathcal{U}_g^* f) = (\psi, \mathcal{U}_{g^{-1}} f).
$$
Under a certain condition on $\psi$, we show that operator $\gothic{W}$ is an isometry from $\mathbb{H}_I(\mathbb{R}^2)$ (with $\mathbb{L}_2$-norm) into $\mathbb{L}_2(SE(2))$. Therefore this operator is closable and its closure is an isometry. This bring us to the main result.

\begin{theorem}
Let $\psi \in \mathbb{H}_{-I}(\mathbb{R}^2)$. If $M_{D_\beta^{-1}\psi} = \beta^{-2}$ then $\gothic{W}_\psi$ maps $\mathcal{D}(D_\beta)$ isometrically (w.r.t. $\mathbb{L}_2$-norm) onto a closed subspace of $\mathbb{L}_2(SE(2))$. Moreover, this operator is closable and its isometric closure is given by $D_\beta \mathcal{W}_{D_\beta^{-1}\psi}:\mathbb{L}_2(\mathbb{R}^2) \rightarrow \mathbb{L}_2(SE(2))$, where $\mathcal{W}_{D_\beta^{-1}\psi}$ is the normed non-distributional orientation score transform, w.r.t. kernel $D_\beta^{-1} \psi \in \mathbb{L}_2(\mathbb{R}^2)$.
\end{theorem}

\begin{proof}
First we provide some preliminaries. Operator $D_\beta$ is an unbounded, self-adjoint (thereby closed) operator that is bounded from below, with bounded inverse. Therefore, $\mathbb{H}_I(\mathbb{R}^2)$ is again a Hilbert space:
\begin{itemize}
   \item[] Let $(f_n)_{n\in\mathbb{N}}$ be Cauchy in $\mathbb{H}_I(\mathbb{R}^2)$. Then $(D_\beta f_n)_{n\in\mathbb{N}}$ is Cauchy in $\mathbb{L}_2(\mathbb{R}^2)$. Because $\mathbb{L}_2(\mathbb{R}^2)$ is complete we have $D_\beta f_n \rightarrow g$ in $\mathbb{L}_2(\mathbb{R}^2)$. But then, since $D_\beta^{-1}$ is bounded, $f_n$ is also Cauchy in $\mathbb{L}_2(\mathbb{R}^2)$, so $f_n \rightarrow f$ in $\mathbb{L}_2(\mathbb{R}^2)$ to some $f \in \mathbb{L}_2(\mathbb{R}^2)$. Now $D_\beta$ is self adjoint and therefore closed so $f \in \mathcal{D}(D_\beta)$ and $D_\beta f = g$. So we have $D_\beta f_n \rightarrow D_\beta f$ in $\mathbb{L}_2(\mathbb{R}^2)$, so $f_n \rightarrow f$ in $\mathbb{H}_I(\mathbb{R}^2)$, and $f \in \mathbb{H}_I(\mathbb{R}^2)$.
\end{itemize}
The space $\mathbb{H}_{-I}(\mathbb{R}^2)$ is defined as the completion of $\mathbb{H}_I(\mathbb{R}^2)$ and is equipped with inverse product $(f,g)_{-I} \break= (D_\beta^{-1} f, D_\beta^{-1} g)_{\mathbb{L}_2(\mathbb{R}^2)}$. This space is isomorphic to the dual space of $\mathbb{H}_I(\mathbb{R}^2)$ under the pairing
\begin{equation}
\label{eq:EqB}
\left<F,f\right> = (R^{-1} F, R f)_{\mathbb{L}_2(\mathbb{R}^2)}
\end{equation}
for all $F \in \mathbb{H}_{-I}(\mathbb{R}^2)$ and $f \in \mathbb{H}_{I}(\mathbb{R}^2)$. In fact, all embeddings in $\mathbb{H}_I(\mathbb{R}^2) \hookrightarrow \mathbb{L}_2(\mathbb{R}^2) \hookrightarrow \mathbb{H}_{-I}(\mathbb{R}^2)$ are dense. Now every $\mathbb{L}_2(\mathbb{R}^2)$ element is the limit of $\mathbb{H}_I(\mathbb{R}^2)$ elements, i.e., $\mathbb{H}_I(\mathbb{R}^2)$ is dense in $\mathbb{L}_2(\mathbb{R}^2)$. Furthermore, since $D_\beta^{-1}$ is bounded we have $\mathbb{H}_I(\mathbb{R}^2) = D_\beta^{-1}(\mathbb{L}_2(\mathbb{R}^2))$.

Now after these preliminaries, let us continue with the proof. Consider the associated normal orientation score transform
$$
\mathcal{W}_{\tilde{\psi}}: \mathbb{L}_2(\mathbb{R}^2) \mapsto \mathbb{C}_K^{SE(2)}
$$
with $\tilde{\psi} = D_\beta^{-1} \psi \in \mathbb{L}_2(\mathbb{R}^2)$ associated to $\psi \in \mathbb{H}_{-I}(\mathbb{R}^2)$. Then by the results in \cite[Thm.~18 and 19]{Duits2005} this transform is unitary.  This transform maps $\mathbb{L}_2(\mathbb{R}^2)$ onto the unique reproducing kernel subspace $\mathbb{C}_K^{SE(2)}$, with reproducing kernel $K(g,h) = (\mathcal{U}_g\tilde{\psi},\mathcal{U}_h\tilde{\psi})$. In fact we have
$$
\begin{array}{rl}
\Vert \mathcal{W}_{\tilde{\psi}} f \Vert_{\mathbb{C}_k^{SE(2)}}^2 &= \int \limits_{\mathbb{R}^2} \int \limits_{S^1} |\mathcal{F} \mathcal{W}_{\tilde{\psi}} f (\www,\theta) |^2 d\theta M_{\tilde{\psi}}^{-1}(\www) d\www \\
&= \int \limits_{\mathbb{R}^2} |f(\mathbf{x})|^2 d\mathbf{x},
\end{array}
$$
for all $f\in\mathbb{L}_2(\mathbf{R}^2)$.
Therefore $D_\beta \mathcal{W}_{\tilde{\psi}} = D_\beta \mathcal{W}_{D_\beta^{-1}\psi}$ is an isometry from $\mathbb{L}_2(\mathbb{R}^2)$ into $\mathbb{L}_2(SE(2))$ if $M_{D_\beta^{-1}\psi}=\beta^{-2}$ (since $\beta^2 M_{D_\beta^{-1}\psi} = 1$), and moreover if $f \in \mathcal{D}(D_\beta)=\mathbb{H}_I(\mathbb{R}^2)$ we have (using Eq.~(\ref{eq:EqB}) and $D_\beta \mathcal{U}_g = \mathcal{U}_g D_\beta$) that
\begin{multline}
(D_\beta \mathcal{W}_{D_\beta^{-1}\psi}f)(g)
=(\mathcal{W}_{D_\beta^{-1}\psi}D_\beta f)(g)
\\=(\mathcal{U}_g D_\beta^{-1} \psi , D_\beta f)_{\mathbb{L}_2(\mathbb{R}^2)}
= (D_\beta^{-1} \psi , D_\beta \mathcal{U}_{g^{-1}}f)_{\mathbb{L}_2(\mathbb{R}^2)}
\\=\left<\psi,\mathcal{U}_{g^{-1}}f\right>
=(\gothic{W}_\psi f)(g),
\end{multline}
for all $g\in SE(2)$ and for all $f \in \mathbb{H}_I(\mathbb{R}^2)$.

\begin{sloppypar}Now Hilbert space $\mathbb{H}_I(\mathbb{R}^2)$ is dense in $\mathbb{L}_2(\mathbb{R}^2)$ and $\left. D_\beta \mathcal{W}_{D_\beta^{-1}\psi}\right|_{\mathbb{H}_I(\mathbb{R}^2)} = \gothic{W}_\psi$ maps $\mathbb{H}_I(\mathbb{R}^2)$ (with $\mathbb{L}_2$-norm) isometrically into $\mathbb{L}_2(SE(2))$. So $\gothic{W}_\psi$ is closable as it admits the closed extension $D_\beta \mathcal{W}_{D_\beta^{-1}\psi}$ as an extension.
\end{sloppypar}
\end{proof}

\paragraph{Concluding Remark}
By the result of the previous theorem, $\psi \in \mathbb{H}_{-I}(\mathbb{R}^2)$ with $M_{D_\beta^{-1}\psi}=\beta^{-2}$ can be called proper \emph{distributional} wavelets. When insisting on an $\mathbb{L}_2$-isometric mapping between image and score one has to fall back on these kind of wavelets. In case of cake wavelets (proper wavelets of class I \cite[Ch.~4.6.1]{Duits2005}), when $\varrho \rightarrow \infty$ such wavelets typically become oriented $\delta$-distributions.

In case of proper wavelets of class II \cite[Ch.~4.6.2]{Duits2005} (including the kernel proposed by Kalitzin \cite{Kalitzin1999}) such wavelets concentrate around and explode along the x-axis when $N\rightarrow\infty$, \cite[Fig.~4.11 and Ch.~7.3]{Duits2005}. In both cases the limits do not exists in $\mathbb{L}_2$-sense, but they do exist both pointwise and in $\mathbb{H}_{-I}$-sense.

We conclude from the results in this section that the orientation score framework does not insist on images to be bandlimited, and remains valid regardless the sampling size/rate.


\bibliographystyle{spmpsci}      
\bibliography{ErikBekkers}   

%
%

\end{document}